\documentclass{article}

\usepackage{arxiv}

\usepackage[utf8]{inputenc} 
\usepackage[T1]{fontenc}    
\usepackage[hidelinks]{hyperref}       
\usepackage{url}            
\usepackage{booktabs}       
\usepackage{amsfonts}       
\usepackage{nicefrac}       
\usepackage{microtype}      
\usepackage{lipsum}		
\usepackage{graphicx}
\usepackage{natbib}
\usepackage{doi}

\usepackage{amsmath,amssymb}
\usepackage{array}
\usepackage{cite}
\usepackage{amsfonts}
\usepackage{graphicx} 
\usepackage{epsfig} 
\usepackage{epstopdf}
\usepackage{epic}
\usepackage[percent]{overpic}
\usepackage{float}
\usepackage{algpseudocode}
\usepackage{algorithm}
\usepackage{balance}
\usepackage{mathrsfs}
\usepackage{mathtools}
\usepackage{times} 
\usepackage{makeidx}
\usepackage{float}
\usepackage{url}
\usepackage{array}
\usepackage{subfigure}
\usepackage{color}
\usepackage[pdftex,dvipsnames]{xcolor}
\usepackage{lipsum} 
\usepackage{ulem}
\usepackage{xargs}
\usepackage{setspace}
\usepackage{cleveref}
\usepackage{multirow}
\usepackage{eurosym,rotating,booktabs}
\usepackage[flushleft]{threeparttable}
\usepackage{setspace}
\usepackage{textcomp}

\newcommand{\simiid}[1]{\mathbin{\overset{#1}{{\tiny \textrm{\kern\z}} \sim}}}%

\algnewcommand\algorithmicforeach{\textbf{for each}}
\algdef{S}[FOR]{ForEach}[1]{\algorithmicforeach\ #1\ \algorithmicdo}
\algdef{SE}[FOR]{NoDoFor}{EndFor}[1]{\algorithmicfor\ #1}{\algorithmicend\ \algorithmicfor}
\algdef{SE}[IF]{NoThenIf}{EndIf}[1]{\algorithmicif\ #1}{\algorithmicend\ \algorithmicif}%
\algdef{SE}[WHILE]{NoDoWhile}{EndWhile}[1]{\algorithmicwhile\ #1}{\algorithmicend\ \algorithmicwhile}%

\newtheorem{theorem}{Theorem}
\newtheorem{lemma}{Lemma}

\newtheorem{remark}{Remark}
\newtheorem{definition}{Definition}
\newtheorem{assumption}{Assumption}
\newtheorem{proposition}{Proposition}
\newtheorem{problem}{Problem}
\newenvironment{proof}{\paragraph{\emph{Proof:}}}{
}

\DeclareSymbolFont{symbolsC}{U}{txsyc}{m}{n}
\DeclareMathSymbol{\notniFromTxfonts}{\mathrel}{symbolsC}{61}

\def\BibTeX{{\rm B\kern-.05em{\sc i\kern-.025em b}\kern-.08em
    T\kern-.1667em\lower.7ex\hbox{E}\kern-.125emX}}

\usepackage{booktabs}
\usepackage{multirow}
\usepackage{array}
\usepackage[T1]{fontenc}
\usepackage[utf8]{inputenc} 
\usepackage{tabularx,ragged2e,booktabs}
\newcolumntype{C}[1]{>{\Centering}m{#1}}

\usepackage{makecell}
\usepackage{amsmath}
\usepackage[toc,page]{appendix}

\newcommand{\indep}{\perp \!\!\! \perp}

\title{Fully Decentralized, Scalable Gaussian Processes \\for Multi-Agent Federated Learning}

\author{ 
George P.~Kontoudis
\\
	Bradley Dept. of Electrical and Computer Eng.\\
	Virginia Tech\\
	\texttt{gpkont@vt.edu} \\
	\And
	Daniel J.~Stilwell
	\\
	Bradley Dept. of Electrical and Computer Eng.\\
	Virginia Tech\\
	\texttt{stilwell@vt.edu} \\
}

\renewcommand{\shorttitle}{Decentralized, Scalable Gaussian Processes}

\hypersetup{
pdftitle={Decentralized Gaussian processes},
pdfauthor={George P.~Kontoudis, Daniel J.~Stilwell},
}

\begin{document}
\normalem
\maketitle

\begin{abstract}
	In this paper, we propose decentralized and scalable algorithms for Gaussian process (GP) training and prediction in multi-agent systems. To decentralize the implementation of GP training optimization algorithms, we employ the alternating direction method of multipliers (ADMM). A closed-form solution of the decentralized proximal ADMM is provided for the case of GP hyper-parameter training with maximum likelihood estimation. Multiple aggregation techniques for GP prediction are decentralized with the use of iterative and consensus methods. In addition, we propose a covariance-based nearest neighbor selection strategy that enables a subset of agents to perform predictions. The efficacy of the proposed methods is illustrated with numerical experiments on synthetic and real data.
\end{abstract}


\section{Introduction}
Teams of agents have received considerable attention in recent years, as they can address tasks that cannot be efficiently accomplished by a single entity. Multi-agent systems are attractive for their inherent property of collecting simultaneously data from multiple locations---a group of agents can collect more data than a single agent during the same time period. Central to machine learning (ML) methodologies is the collection of large datasets in order to ensure reliable training. To this end, networks of agents favor learning techniques, due to their data collection capabilities. However, they face major challenges including limited computational resources and communication restrictions. A typical approach to address these challenges relies on centralizing the collected data in a single node (e.g., cloud or data center), which requires high computational and storage resources. Yet, gathering data to a central server may lead to network traffic congestion and security or privacy issues. To ensure data privacy, a promising solution is federated learning (FL) \citep{konevcny2016federated}. FL aims to implement ML techniques in centralized or decentralized networks, but with no communication of real data in order to comply to the EU/UK general data protection regulation (GDPR) \citep{horvitz2015data}. For certain applications, such as in GPS-denied environments, it is unfeasible to implement ML algorithms in a centralized network, as distant nodes may not be able to establish communication directly with the central node due to communication range limitations or bandwidth. Such cases include autonomous vehicles and multi-robot systems. Finally, even if we manage to collect all the data in a central node, the time and space computational complexity for rapid updates of the ML models require resources that are not available to agents operating in the field. In this work, we propose methodologies for fully decentralizing Gaussian processes (GPs) \citep{rasmussen2006gaussian,gramacy2020surrogates,cressie1992statistics} from training to prediction, so that they can be implemented efficiently on teams of agents. GPs are used in various multi-agent applications \citep{singh2009efficient,xu2011mobile,gu2012spatial,chen2015gaussian,choi2015distributed,allamraju2017communication,pillonetto2018distributed,zhi2019continuous,hoang2019collective,tavassolipour2019learning,yuan2020communication,jang2020multi,lin2020distributed,kepler2020approach,kontoudis2021prediction,Kontoudis2021RAS,suryan2020learning,kontoudis2021communication}. The major disadvantage of GPs is the poor scalability with the number of observations. Moreover, GPs are not easily decentralized for implementation across multiple agents due to high communication requirements.  
Our aim in this work is to develop fully decentralized approximate methodologies that relax the communication and computation requirements of GPs, exchanging as little information as possible and by performing only local computations.  We propose three distributed optimization techniques to implement GP hyperparameter training with maximum likelihood estimation (MLE), based on the alternating direction method of multipliers (ADMM) \citep{boyd2011admm}. Next, we synthesize 13 decentralized approximate methods to perform GP prediction with aggregation of GP experts \citep{liu2020gaussian}, using iterative and consensus protocols \citep{bertsekas2003parallel,olfati2007consensus,wang2016improvement}. Two of the latter approximate methods for GP prediction (DEC-NPAE and DEC-NN-NPAE) have been discussed in our previous work~\citep{kontoudis2021decentralized}.

\textbf{\textit{Related work}}: Despite their effectiveness in function approximation and uncertainty quantification, GPs scale poorly with the number of observations. Particularly, provided $N$ observations, the training entails $\mathcal{O}(N^3)$ computations and the prediction requires $\mathcal{O}(N^2)$ computations. Another limitation for the implementation of GPs in multi-agent systems is the communication. For centralized GPs, every agent has to communicate all observations to a central node. However, excessive communication is challenging in decentralized networks. Moreover, agents in networks can pass messages only within a communication~range~\citep{bullo2009distributed} which may vary in space and time \citep{kontoudis2019comparison}. 

To overcome the computational burden of hyper-parameter GP training with maximum likelihood estimation (MLE), a factorized GP training  method is discussed in \citep{deisenroth2015distributed}. That is a centralized method which is based on a server-client structure and distributes the computations to multiple entities. The main idea is to assume independence between sub-models, which results in the approximation of the inverse covariance matrix by the inverse of a block diagonal matrix. To this end, a significant reduction in computation of the inverse of multiple covariance matrices is achieved at the cost of excessive communication overhead. More specifically, every local entity transmits multiple inverted blocks of the covariance matrix per MLE iteration. Recently, Xu \textit{et al.} \citep{xu2019wireless} reformulated the factorized GP training method using the exact consensus alternating direction method of multipliers (ADMM) \citep{boyd2011admm}, which is appealing in centralized multi-agent settings \citep{halsted2021survey}. Consensus ADMM reduces the communication overhead of GP training, but requires high computational resources to solve a nested optimization problem at every ADMM-iteration. Subsequently, the authors in \citep{xie2019distributed} employed the inexact proximal ADMM \citep{hong2016convergence} to alleviate the computation demand. However, both ADMM-based factorized GP training methods require a centralized network topology. 

Two major research directions for GP prediction are based on global and local approximations \citep{liu2020gaussian}. Global approximation methods promote sparsity by using either a subset of $N_{\textrm{sub}}$ observations or by introducing a set of $N_{\textrm{sub}}$ pseudo-inputs, where $N_{\textrm{sub}} \ll N $ \citep{quinonero2005unifying,snelson2006sparse,hensman2013gaussian}. Sparse GPs have been used in mobile sensor networks to model spatial fields \citep{gu2012spatial}. In \citep{xu2011mobile}, a GP with truncated observations in a mobile sensor network is proposed, and in \citep{chen2015gaussian} a subset of observations is used for traffic modeling and prediction. These methods require global knowledge of the observations, which increases inter-agent communications. Additionally, the methods that use pseudo-inputs do not retain the interpolation property.  

Alternatively, the second research direction uses local approximation methods to reduce the computational burden of GP prediction. These are centralized algorithms with a server-client structure. The main idea is to aggregate local sub-models produced by local subsets of the observations \citep{tresp2000bayesian,hinton2002training,deisenroth2015distributed,cao2014generalized}. In other words, every sub-model makes a local prediction, and then the central node aggregates to a single prediction. In comparison to global approximations, local methods do not require inducing inputs, they distribute the computational load to multiple agents, and they work with all observations. However, it is proved in \citep[Proposition~1]{bachoc2017some} that the local methods  \citep{tresp2000bayesian,hinton2002training,deisenroth2015distributed} are \textit{inconsistent}, i.e. as the observation size grows to infinity, the aggregated predictions do not converge to the true values. Subsequently, the authors in \citep{rulliere2018nested} proposed the nested point-wise aggregation of experts (NPAE) that takes into account the covariance between sub-models and produces consistent predictions. The price to achieve consistency in NPAE comes with much higher computational complexity in the central node. Liu \textit{et al.} \citep{liu2018generalized} introduced a computationally efficient and consistent methodology, termed as generalized robust Bayesian committee machine (grBCM). The latter entails additional communication between agents to enrich local datasets with a global random dataset. In addition, both \textsc{NPAE} and grBCM are centralized techniques---not well-suited for multi-agent systems \citep{bullo2009distributed}.

A decentralized method for the computation of spatio-temporal GPs is proposed in \citep{cortes2009distributed}. In \citep{choi2015distributed}, a decentralized technique for spatial GPs with localization uncertainty is presented. Both \citep{cortes2009distributed} and \citep{choi2015distributed} employ the Jacobi over-relaxation (JOR), which requires a strongly complete graph topology, i.e. every node must communicate to every other node. That is a conservative topology and is not common in mobile sensor networks \citep{bullo2009distributed}. Essentially, for not strongly complete topologies, JOR entails flooding before every iteration. In flooding each agent broadcasts all input packets to its neighbors \citep{topkis1985concurrent}. Thus, the communication requirements of JOR are high. Yuan and Zhu \citep{yuan2020communication,yuan2021resource}, proposed a methodology that combines nearest neighbor GPs \citep{datta2016hierarchical} and local approximation \citep{cao2014generalized}. Although \citep{cao2014generalized} is consistent in terms of prediction mean, it produces overconfident prediction variances \citep[Proposition~1]{liu2018generalized}. In addition, arbitrary selection of nearest neighbor sets may lead to poor approximations \citep{datta2016hierarchical} and suffers from prediction discontinuities \citep{rulliere2018nested}. Pillonetto \textit{et al.} \citep{pillonetto2018distributed} proposed sub-optimal methods to distributively estimate a latent function with a GP by employing orthonormal eigenfunctions, computed by the Karhunen-Lo\`eve expansion of a kernel. An extension of this work to multi-robot systems with online information gathering is discussed in \citep{jang2020multi}. This is a promising line of research for GPs in decentralized networks, but our focus is on decentralized and scalable GP training with MLE and GP prediction with aggregation methods. Nevertheless, computing orthonormal eigenfunctions in closed-form is not feasible for all kernels and may yield significant storage requirements.

\textbf{\textit{Contributions}}: 
The contributions are as follows: 
\begin{enumerate}
    \item We extend a centralized GP training methodology \citep{xie2019distributed} by devising augmented local datatsets to equip local entities, so that the hyper-parameter estimation accuracy of large-scale multi-agent systems is improved. 
    \item We introduce three decentralized GP training methods for strongly connected graph topologies and we derive a closed-form solution on the decentralized inexact ADMM \citep{chang2014multi} that reduces the computational requirements of local agents. 
    \item We decentralize the implementation of multiple aggregation of GP experts methods (PoE~\citep{hinton2002training}, gPoE~\citep{cao2014generalized}, BCM \citep{tresp2000bayesian}, rBCM~\citep{deisenroth2015distributed}, and grBCM~\citep{liu2018generalized}) for strongly connected graph topologies, by using the discrete-time average consensus (DAC) \citep{olfati2007consensus}.
    \item We decentralize the implementation of NPAE \citep{rulliere2018nested} for strongly complete graph topologies, by combining Jacobi over-relaxation (JOR) \citep[Chapter 2.4]{bertsekas2003parallel} and DAC. Moreover, we introduce a technique to recover the optimal relaxation factor of JOR \citep{udwadia1992some} for strongly complete graph topologies by using the power method (PM) \citep[Chapter 8]{golub2013matrix}. The later ensures faster convergence. 
    \item We introduce a covariance-based nearest neighbor (CBNN) technique that selects statistically correlated agents for GP prediction on locations of interest, and provide a consistency proof. The CBNN is applicable to the decentralized versions of PoE, gPoE, BCM, rBCM, and grBCM introduced in 3). In addition, CBNN allows the use of a distributed algorithm to solve systems of linear equations (DALE) \citep{wang2016improvement,liu2017asynchronous} which replaces JOR in the decentralized NPAE of 4) and relaxes the graph topology from strongly complete to strongly connected. 
\end{enumerate}
\textbf{\textit{Structure}}: 
In \Cref{sec:probForm} we formulate the decentralized GP training and prediction problem, \Cref{sec:centTrain} discusses existing centralized GP training methods and extends one of them. In \Cref{sec:decentTrain}, we propose methods for decentralized GP training, \Cref{sec:decentPred} introduces multiple techniques for decentralized GP prediction, \Cref{sec:numericalExperiments} provides numerical examples for both decentralized GP training and prediction, and \Cref{sec:conclusion} concludes the paper. 

\section{Preliminaries and Problem Statement}\label{sec:probForm}
In this section, we discuss the foundations of algebraic graph theory, overview GPs, describe existing distributed approximate methods for scalable GPs, and define the problem of decentralized, scalable GP training and prediction.
\subsection{Foundations}\label{ssec:found}
The notation here is standard. The set of all positive real numbers $\mathbb{R}_{>0}$ and the set of all non-negative real numbers $\mathbb{R}_{\geq 0}$. We denote by $I_n$ the identity matrix of $n\times n$ dimension. The vector of $n$ zeros is represented as $\mathbf{0}_n$ and the matrix of $n\times m$ zeros as $\mathbf{0}_{n\times m}$. The superscript in parenthesis ${y}^{(s)}$ denotes the $s$-th iteration of an estimation process. The cardinality of the set $K$ is denoted $\mathrm{card}(K)$, the absolute values is denoted $\vert \cdot \vert$, the $L_2$ norm is denoted $\| \cdot \|_2$, and $\| \cdot \|_{\infty}$ denotes the infinity norm. The notation $\overline{\lambda}(\boldsymbol{F})$ and $\underline{\lambda}(\boldsymbol{F})$ denote the maximum and minimum eigenvalue of matrix $\boldsymbol{F}$ respectively. The $i$-th row of matrix $\boldsymbol{F}$ is denoted $\textrm{row}_i\{\boldsymbol{F}\}$, the $j$-th entry of the $i$-th row is denoted $[\textrm{row}_i\{\boldsymbol{F}\}]_j$, and the $i$-th element of a vector $\boldsymbol{x}$ is denoted $[\boldsymbol{x}]_i$ or $x_i$. A collection of elements that comprise a vector $\boldsymbol{x}\in \mathbb{R}^N$ is denoted $\{ x_i \}_{i=1}^N$. 

The communication complexity is denoted $\mathcal{O}(\cdot)$ and describes the total number of bits required to be transmitted over the course of the algorithm up to convergence
\citep[Chapter~3]{bullo2009distributed}. Time and space complexity are both denoted $\mathcal{O}(\cdot)$ and provide the maximum computations to be performed and space to be occupied at any instant of an algorithm respectively. All complexities are calculated with respect to the the total number of observations $N$, the input dimension $D$, and the size of the fleet $M$. In addition, the communication complexity employs the iterations required for an algorithm to convergence $s^{\textrm{end}}$.

Suppose a network consists of $M$ agents that can perform local computations. The network is described by an undirected time-varying graph $\mathcal{G}(t) = (\mathcal{V}, \mathcal{E}(t))$, where $\mathcal{V} = 1, \ldots, M$ is the set of nodes and $\mathcal{E}(t) \subseteq \mathcal{V} \times \mathcal{V} $ the set of edges at time $t$. An undirected graph implies that for all $t$ the communication is bidirectional. Nodes represent agents and edges their communication. The neighbors of the $i$-th node are denoted $\mathcal{N}_i(t) = \{j\in \mathcal{V} : (i,j) \in \mathcal{E}(t) \}$. The adjacency matrix of $\mathcal{G}(t)$ is denoted $\boldsymbol{A}(t) = [a_{ij}] \in \mathbb{R}^{M \times M}$, where $a_{ij} =1$ if $(i,j) \in \mathcal{E}(t)$ and $a_{ij} =0$ otherwise. Similarly, the degree matrix of $\mathcal{G}(t)$ is denoted $\boldsymbol{D}(t) = [d_{ij}] \in \mathbb{R}^{M \times M}$ and is diagonal with $d_i = \sum_{j=1}^M a_{ij}$. The graph Laplacian is defined as $\mathcal{L}(t) \coloneqq \boldsymbol{D}(t) - \boldsymbol{A}(t)$. The maximum degree is denoted $\Delta = \max_i\{\sum_{j\neq i } a_{ij}\} $ and represents the maximum number of neighbors in the graph. The Perron matrix is defined as $\mathcal{P}(t) \coloneqq I_M - \epsilon \mathcal{L}(t)$, where $\epsilon$ is a parameter with range $\epsilon \in (0,1]$. The maximum shortest distance between any pair of nodes in $\mathcal{G}$ is denoted $\mathrm{diam}(\mathcal{G})$. If the adjacency matrix $A$ is irreducible, then the graph $\mathcal{G}$ is strongly connected \citep{olfati2007consensus}. In addition, a graph $\mathcal{G}$ is strongly complete if every agent can communicate to every other agent in the graph. We consider three decentralized network topologies as presented in Figure~\ref{fig:graph_mixed}.

\begin{assumption}\label{ass:repeat_connectivity}
\citep{liu2017asynchronous} There exists a positive integer $\gamma\in \mathbb{Z}_{\geq 0}$ such that for all time $t$ the graph $\mathcal{H} = (\mathcal{V},\mathcal{E}(t)\cup \mathcal{E}(t\gamma+1)\cup \ldots \cup \mathcal{E}((t+1)\gamma-1 )$ is strongly connected.
\end{assumption}

\begin{figure}[!t]
	\includegraphics[width=.45\columnwidth]{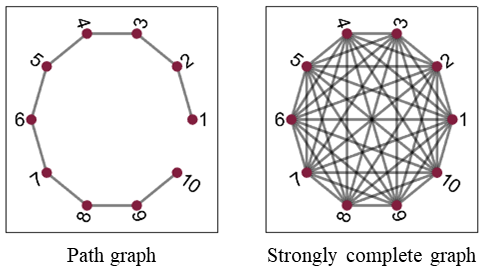}
	\centering
	\caption{Graph topologies of multi-agent systems. }
	\label{fig:graph_mixed}
\end{figure}
\subsection{Gaussian Processes}\label{ssec:GP}
Let the observations be modeled by,
\begin{equation}\label{eq:model}
    y(\boldsymbol{x} ) = f(\boldsymbol{x}) + \epsilon,
\end{equation}
where $\boldsymbol{x}\in \mathbb{R}^D$ is the input location with $D$ the input space dimension, $f(\boldsymbol{x}) \sim \mathcal{GP}(0,k(\boldsymbol{x},\boldsymbol{x}^{\prime}))$ is a zero-mean GP with covariance function $k: \mathbb{R}^D\times  \mathbb{R}^D \rightarrow  \mathbb{R}$, and $\epsilon \sim \mathcal{N}(0, \sigma_{\epsilon}^2)$ is the i.i.d. measurement noise with variance $\sigma_{\epsilon}^2>0$. We employ the 
{separable} squared exponential covariance function, 
\begin{equation}\label{eq:seKernel}
    k(\boldsymbol{x} ,\boldsymbol{x}^{\prime}) = \sigma_f^2 \exp \left\{ -    \sum_{d=1}^D\frac{\left({x}_d - {x}_d^{\prime}\right)^2 }{l_d^2}\right\},
\end{equation}
where $\sigma_f^2>0$ is the signal variance {and $l_d>0$ the length-scale hyperparameter at the $d$-th direction of the input space.} The goal of GPs is to infer the underlying latent function $f$ given the data $\mathcal{D} = \{\boldsymbol{X},\boldsymbol{y} \}$, where $\boldsymbol{X} = \{\boldsymbol{x}_n\}_{n=1}^N$ the inputs, $\boldsymbol{y} = \{y_n\}_{n=1}^N$ the outputs, and $N$ the number of observations.
\subsubsection{Training} 
A GP is trained to find the hyperparameters $\boldsymbol{\theta} = ( l_1,\hdots,l_D,\sigma_f,\sigma_{\epsilon} )^{\intercal} \in \Theta \subset \mathbb{R}^{D+2}$ that maximize the marginal log-likelihood,
\begin{align*} 
    \mathcal{L} = \log p(\boldsymbol{y}\mid  \boldsymbol{X})  =   -\frac{1}{2} \left( \boldsymbol{y}^{\intercal}\boldsymbol{C}_{\theta}^{-1}\boldsymbol{y} 
        + \log |\boldsymbol{C}_{\theta}| +N\log 2\pi \right) ,
\end{align*}
where $\boldsymbol{C}_{\theta} = \boldsymbol{K}+\sigma_{\epsilon}^2I_N$ is the positive definite (PD) covariance matrix with $\boldsymbol{K} = k(\boldsymbol{X},\boldsymbol{X}) \succeq 0 \in  \mathbb{R}^{N\times N}$ the positive semi-definite (PSD) correlation matrix. The minimization problem employs the negative marginal log-likelihood (NLL) function,
\begin{align}\nonumber
\textrm{(P1)} \quad \hat{\boldsymbol{\theta}} = \arg\min_{\boldsymbol{\theta}} \quad & 
\boldsymbol{y}^{\intercal}\boldsymbol{C}_{\theta}^{-1}\boldsymbol{y} 
        +  \log |\boldsymbol{C}_{\theta}| 
        \\ \label{eq:P1_constr}
\textrm{s.to} \quad & \theta_j > 0, \quad \forall j =1,\hdots, D+2 .   
\end{align}
The bound constraints \eqref{eq:P1_constr} on the length-scales $l_d$ ensure that the correlation matrix is PSD. Additionally, in practice even small noise variance $\sigma_{\epsilon}^2$ is useful to avoid an ill-conditioned covariance matrix. First-order iterative methods (e.g., conjugate gradient descent) or second-order iterative methods with approximated Hessian (e.g., L-BFGS-B \citep{byrd1995limited}) are widely used to tackle (P1). Note that the NLL in (P1) is non-convex with respect to the hyperparameters $\boldsymbol{\theta}$ and usually multiple starting locations are randomly selected to ensure global optimality \citep{chen2018priors}. Both optimization approaches require the computation of the gradient of (P1) which is given by,
\begin{align}\label{eq:grad_mle}
\frac{\partial \mathcal{L}(\boldsymbol{\theta})}{\partial \theta_j} = \frac{1}{2}\mathrm{tr} \left\{ \left( \boldsymbol{C}_{\theta}^{-1} - \boldsymbol{C}_{\theta}^{-1} \boldsymbol{y}\boldsymbol{y}^{\intercal} \boldsymbol{C}_{\theta}^{-1} \right) \frac{\partial \boldsymbol{C}_{\theta}} {\partial \theta_j} \right\},
\end{align}
The partial derivative of the covariance matrix ${\partial \boldsymbol{C}_{\theta}}/ {\partial \theta_j}$ depends on the selected covariance function $k(\cdot, \cdot)$. For our selection \eqref{eq:seKernel}, the partial derivative is provided in Appendix~\ref{app:derivative_SE_covariance}.
\subsubsection{Prediction} 
After obtaining the hyperparameters $\hat{\boldsymbol{\theta}}$, the predictive distribution of the location of interest $\boldsymbol{x}_*\in \mathbb{R}^D$ conditioned on the data $\mathcal{D}$ yields $p({y}_* \mid \mathcal{D}, \boldsymbol{x}_* )\sim \mathcal{N}(\mu(\boldsymbol{x}_*), \sigma^2(\boldsymbol{x}_*))$ with prediction mean and variance,
\begin{align}\label{eq:mean_gp}
    \mu_{\textrm{full}}(\boldsymbol{x}_*) &= \boldsymbol{k}_*^{\intercal}\boldsymbol{C}_{\theta}^{-1}\boldsymbol{y},\\ \label{eq:variance_gp}
    \sigma_{\textrm{full}}^2(\boldsymbol{x}_*) &= 
    k_{**} -  \boldsymbol{k}_*^{\intercal}\boldsymbol{C}_{\theta}^{-1}\boldsymbol{k}_*
    ,
\end{align}
where $\boldsymbol{k}_*=k(\boldsymbol{X},\boldsymbol{x}_*) \in \mathbb{R}^N$ and $k_{**}=k(\boldsymbol{x}_*,\boldsymbol{x}_*) \in \mathbb{R}$. 
\subsubsection{Complexity}\label{sssec:GPcomplexity}
The time complexity of the training is $\mathcal{O}(N^3)$ for computing the inverse of the covariance matrix in (P1) and \eqref{eq:grad_mle}. Note that only the inverse of the covariance matrix $\boldsymbol{C}_{\theta}^{-1}$ is required to be computed for the training (P1) and not the logarithm of its determinant $\log |\boldsymbol{C}_{\theta}|$. That is because the covariance matrix is symmetric and PD, and thus the Cholesky decomposition can be employed to compute the logarithm of the covariance matrix determinant \citep[Appendix A.4]{rasmussen2006gaussian}. The inverse computation of the covariance matrix is performed repeatedly in the optimization (P1) to find the hyperparameters $\hat{\boldsymbol{\theta}}$. After solving (P1) and obtaining the hyperparameter vector $\hat{\boldsymbol{\theta}}$, we store the inverse of the covariance matrix $\boldsymbol{C}^{-1}$ and $N$ observations, which results in $\mathcal{O}(N^2+DN)$ space complexity. For agents with limited RAM memory capacity, the space complexity
may be more restrictive than the time complexity. The prediction mean \eqref{eq:mean_gp} and variance \eqref{eq:variance_gp} yield $\mathcal{O}(N)$ and $\mathcal{O}(N^2)$ computations respectively for matrix multiplications. 
\subsection{Centralized Scalable Gaussian Processes}\label{ssec:centrGP}
Let us consider a network of $M$ agents that can perform local computations. Each agent $i$ collects local observations to form the local dataset $\{\mathcal{D}_i = \{\boldsymbol{X}_i,\boldsymbol{y}_i \}\}_{i=1}^M$ corresponding to $N_i$ observations for $M$ agents with $\sum_{i=1}^{M}N_i=N$. Thus, the global dataset is composed as $\mathcal{D} = \cup_{i=1}^{M}\mathcal{D}_{i}$. All local datasets have equal number of observations, i.e. $N_i = N_j = N/M$ for all $i,j \in \mathcal{V}$ with $i\neq j$. For privacy reasons, for example, we presume that the local datasets $\mathcal{D}_i$ should not be communicated to other agents. In practice, even if all agents have access to the global dataset $\mathcal{D}$, the GP computational complexity (\Cref{sssec:GPcomplexity}) is prohibitive if $\mathcal{D}$ is large. Furthermore, in case we assume a centralized topology, where every entity $i$ communicates its dataset $\mathcal{D}_i$ to a central node with significant computational and storage resources, then we face several problems. These problems comprise of: i) \textit{security} and \textit{robustness}, as the central node is vulnerable to malicious attacks or even failure; ii) \textit{traffic network congestion}, when all agents communicate their local datasets with the central entity; and iii) \textit{privacy}, because a single entity has access to the global dataset. In addition, for certain cases (e.g., autonomous vehicles and multi-robot systems), distant agents may be subject to communication range limitations. To this end, we make the following assumptions.

\begin{table*}[!t]
\caption{Time, Space, and Communication Complexity of GP Training}
\centering
\begin{threeparttable}
{\begin{tabular}{ c c c c c c }
\toprule

 & & & FULL-GP  & FACT-GP &  g-FACT-GP \\
 
 \midrule 
  
 \multirow{2}{*}{Local}  & &  Time & - & $\mathcal{O}(N^3/M^3)$ & $\mathcal{O}(8(N^3/M^3))$ \\
 
    & &   Space & - & $\mathcal{O}(\xi)$ & $\mathcal{O}(2\xi + 2(N^2/M^2))$ \\
    
    \cmidrule(lr){1-6}
 
 \multirow{5}{*}{Global} & \multirow{2}{*}{GD} &  Space & - & $\mathcal{O}(DM+2M)$ & $\mathcal{O}(DM+2M)$ \\

  & &  Comm & - & $\mathcal{O}(s^{\textrm{end}}(DM+2M))$ & $\mathcal{O}(s^{\textrm{end}}(DM+2M))$ \\
    
    \cmidrule(lr){2-6}
       & \multirow{3}{*}{Final} &  Time & $\mathcal{O}(N^3)$ & - & - \\
     
   & &  Space & $\mathcal{O}(N^2+DN)$ & $\mathcal{O}(N^2/M)$ & $  \mathcal{O}(4(N^2/M))$\\

   & &  Comm & - & $  \mathcal{O}(N^2/M)$ & $  \mathcal{O}(4(N^2/M))$ \\

\bottomrule
\end{tabular}}
    \begin{tablenotes}[para,flushleft]
      \item {$\xi = N^2/M^2+D(N/M)$}. 
    \end{tablenotes}\label{tab:complexityTraining}
    \end{threeparttable}
\end{table*}

\begin{assumption}\label{ass:noData}
Every agent $i$ can communicate only with agents in its neighborhood $\mathcal{N}_i$ and the communication shall not include any data exchange.  
\end{assumption}

\begin{assumption}\label{ass:parData}
Every agent $i$ can communicate only with agents in its neighborhood $\mathcal{N}_i$ and the communication shall include partial exchange of the local dataset $\mathcal{D}_i$.
\end{assumption}

Assumption~\ref{ass:noData} prohibits the communication of any observation, while Assumption~\ref{ass:parData} allows the communication of a subset of the local dataset $\mathcal{D}_i$. This distinction has been made to propose different methodologies in case that partial communication of the local dataset is permitted.

\subsubsection{Factorized Training}\label{sssec:factor_training}
The factorized GP training \citep{ng2014hierarchical,deisenroth2015distributed}, termed as FACT-GP, relies on the following assumption.
\begin{assumption}\label{ass:independence}
All local sub-models $\mathcal{M}_i$ are independent.
\end{assumption}

The independence in Assumption~\ref{ass:independence} is invoked to result in the approximation of the global marginal likelihood as, 
\begin{equation}\label{eq:training_approx}
    p(\boldsymbol{y}\mid  \boldsymbol{X}) \approx \prod_{i=1}^M p_i (\boldsymbol{y}_i\mid  \boldsymbol{X}_i),
\end{equation}
where $p_i (\boldsymbol{y}_i\mid  \boldsymbol{X}_i)\sim \mathcal{N}(0, \boldsymbol{C}_{\theta,i})$ is the local marginal likelihood of the $i$-th node with local covariance matrix $\boldsymbol{C}_{\theta,i} = \boldsymbol{K}_i+\sigma_{\epsilon}^2I_{N_i}$ and $\boldsymbol{K}_i = k(\boldsymbol{X}_i,\boldsymbol{X}_i)\in \mathbb{R}^{N_i \times N_i}$. The factorized approximation \eqref{eq:training_approx} implies that the covariance matrix is approximated by a block diagonal matrix that results in 
$\boldsymbol{K}^{-1}\approx \mathrm{diag}(\boldsymbol{K}_1^{-1},\boldsymbol{K}_2^{-1}, \ldots,\boldsymbol{K}_M^{-1}  )$. Subsequently, {the global marginal log-likelihood is approximated by $\mathcal{L}\approx \sum_{i=1}^M\mathcal{L}_i$} which yields,
\begin{equation*}
  \log  p(\boldsymbol{y}\mid  \boldsymbol{X}) \approx \sum_{i=1}^M \log p_i (\boldsymbol{y}_i\mid  \boldsymbol{X}_i),
\end{equation*}
with local marginal log-likelihood,
\begin{align} \label{eq:local_ll} 
    \mathcal{L}_i &= \log p_i (\boldsymbol{y}_i\mid  \boldsymbol{X}_i)
    =  -\frac{1}{2}\bigg(  \boldsymbol{y}_i^{\intercal}\boldsymbol{C}_{\theta,i}^{-1}\boldsymbol{y}_i
    + \log |\boldsymbol{C}_{\theta,i}| +N_i \log 2\pi \bigg).
\end{align}
The gradient of the global marginal log-likelihood in factorized training is computed by $\nabla_{\boldsymbol{\theta}} \mathcal{L} = \sum_{i=1}^M\nabla_{\boldsymbol{\theta}} \mathcal{L}_i$ \citep{xu2019wireless,xie2019distributed}. The minimization problem uses the local negative marginal log-likelihood (LNLL) function and takes the form of,
\begin{align}\nonumber
\textrm{(P2)} \quad \hat{\boldsymbol{\theta}} = \arg\min_{\boldsymbol{\theta}} \quad & 
\sum_{i=1}^M      \boldsymbol{y}_i^{\intercal}\boldsymbol{C}_{\theta,i}^{-1}\boldsymbol{y}_i + \log |\boldsymbol{C}_{\theta,i}| 
\\ \label{eq:P2_constrPos}
\textrm{s.to}\quad & \boldsymbol{\theta_i} > \boldsymbol{0}_{D+2}, \quad \forall i \in \mathcal{V},  
\end{align}
where $\boldsymbol{\theta}_i = \{\theta_{1,i}, \hdots, \theta_{D+2,i}    \}$. Similarly to (P1), constraint \eqref{eq:P2_constrPos} imposes positivity on the agreed hyperparameters. 

\begin{remark}\label{rem:positivity_relax}
 A common approach to relax the positivity constraint \eqref{eq:P1_constr} in (P1) and \eqref{eq:P2_constrPos} in (P2) is to employ the logarithmic transformation on the hyperparameter vector that has strictly positive domain, i.e. $\log(\boldsymbol{\theta}): \mathbb{R}_{>0} \rightarrow \mathbb{R}$. After convergence the inverse transformation, by using the exponential, yields the hyperparameter vector. 
\end{remark}

The computation of \eqref{eq:local_ll} for the FACT-GP training (P2) yields $\mathcal{O}(N_i^3)= \mathcal{O}(N^3/M^3)$ time complexity for each local entity to invert the local covariance matrix $\boldsymbol{C}_{\theta,i}^{-1}$. Additionally, for the storage of the local inverted covariance matrix and the local observations $\mathcal{O}(N_i^2+DN_i) = \mathcal{O}(N^2/M^2+D(N/M))$ space is needed. The factorized training requires also communication from every node $i$ to the central node. The communication complexity depends on the selection of the optimization algorithm. Provided that the central node implements gradient descent \citep{xie2019distributed}, every node communicates the local gradient of LNLL $\nabla_{\boldsymbol{\theta}} \mathcal{L}_i$ at every iteration $s$. That is $\mathcal{O}(s^{\textrm{end}}(D+2)M) = \mathcal{O}(s^{\textrm{end}}(DM+2M))$ total communications from all agents to the central node, where $s^{\textrm{end}}$ is the total number of iterations to reach convergence. Additionally, the central node needs to store at each iteration: i) the hyperparameter vector on the previous iteration $\{\boldsymbol{\theta}^{(s)}_i\}_{i=1}^{M}$ from all $M$ nodes; and ii) their gradient of LNLL $\{\nabla_{\boldsymbol{\theta}} \mathcal{L}_i\}_{i=1}^{M} $, which results in $\mathcal{O}((D+2)M + (D+2))  =\mathcal{O}(DM+2M) $ space complexity. Finally, after the optimization algorithm converges, each node communicates the local inverted covariance matrix $\boldsymbol{C}_{\theta,i}^{-1}$ that yields $\mathcal{O}(MN_i^2) = \mathcal{O}(N^2/M)$ communications to the central node. All local inverted covariance matrices need to be stored in the central node to form the block diagonal approximation for GP prediction, i.e. $\mathcal{O}(N^2/M)$ space complexity. A computational complexity comparison between FULL-GP and FACT-GP is provided in Table~\ref{tab:complexityTraining}. Since $N_i = N/M < N$, the time and space complexity of FACT-GP (P2) is significantly less than the time and space complexity of FULL-GP (P1). 
\subsubsection{Aggregated Prediction}\label{sssec:aggrPred}
Provided the hyperparameter vector $\hat{\boldsymbol{\theta}}$, we shall employ multiple aggregation of GP experts methods to perform joint prediction with local data. These are the PoE~\citep{hinton2002training}, gPoE~\citep{cao2014generalized}, BCM \citep{tresp2000bayesian}, rBCM~\citep{deisenroth2015distributed}, grBCM~\citep{liu2018generalized},  and NPAE~\citep{rulliere2018nested,bachoc2017some}. The main idea is that each local entity $i$ develops a local GP sub-model $\mathcal{M}_i$ using its local dataset $\mathcal{D}_i$. Then, the agents communicate local models to make joint predictions. The local GP sub-model $\mathcal{M}_i$ conditioned on the local dataset is $p_i ({y}_*\mid  \mathcal{D}_i, \boldsymbol{x}_i)\sim \mathcal{GP}(\mu_i(\boldsymbol{x}_i), \sigma^2_i(\boldsymbol{x}_i))$ with local mean and local variance,
\begin{align}\label{eq:local_mean}
    \mu_i(\boldsymbol{x}_*) &= \boldsymbol{k}_{*,i}^{\intercal}\boldsymbol{C}_{\theta,i}^{-1}\boldsymbol{y}_i,\\ \label{eq:local_variance}
    \sigma_i^2(\boldsymbol{x}_*) &= 
    k_{**} -  \boldsymbol{k}_{*,i}^{\intercal}\boldsymbol{C}_{\theta,i}^{-1}\boldsymbol{k}_{*,i}
    ,
\end{align}
where $\boldsymbol{k}_{*,i} = k(\boldsymbol{x}_{*},\boldsymbol{X}_i) \in \mathbb{R}^{N_i}$.

A useful definition to study the properties of joint mean predictions for local aggregation methods is provided below.
\begin{definition}\label{def:consistency}
\citep{rulliere2018nested} Provided $N$ observations, an aggregate GP method with joint prediction mean $\mu_{\mathcal{A}}$, variance $\sigma^2_{\mathcal{A}}$, full GP prediction mean $\mu_{\mathrm{full}}$, and variance $\sigma^2_{\mathrm{full}}$ is consistent if,
\begin{align*} 
    \lim_{N\rightarrow \infty} \mu_{\mathrm{full}}( \boldsymbol{x}_*) - \mu_{\mathcal{A}}( \boldsymbol{x}_*) \rightarrow 0, \quad \forall \boldsymbol{x}_*,\\
    \lim_{N\rightarrow \infty} \sigma^2_{\mathrm{full}}( \boldsymbol{x}_*) - \sigma^2_{\mathcal{A}}( \boldsymbol{x}_*) \rightarrow 0, \quad \forall \boldsymbol{x}_*,
\end{align*} 
where subscript $\mathcal{A}$ denotes any aggregation method for GP prediction.
\end{definition}

\Cref{def:consistency} implies that as the number of observations tends to infinity, the prediction mean and variance of full GP \eqref{eq:mean_gp} and the aggregated prediction mean and variance are identical.

\textbf{\textit{PoE Family}}: After computing the local mean \eqref{eq:local_mean} and the local variance \eqref{eq:local_variance}, the joint mean and precision of both PoE and gPoE $\mathcal{M}_{\mathcal{A}} = \mathcal{M}_{\textrm{(g)PoE}}$ is provided by,
\begin{align}\label{eq:mean_poe}
    \mu_{\textrm{(g)PoE}}( \boldsymbol{x}_*) &= \sigma_{\textrm{(g)PoE}}^2(\boldsymbol{x}_*)\sum_{i=1}^M \beta_i \sigma_i^{-2}(\boldsymbol{x}_*)\mu_i(\boldsymbol{x}_*),\\\label{eq:variance_poe}
    \sigma_{\textrm{(g)PoE}}^{-2}( \boldsymbol{x}_*) &= \sum_{i=1}^M \beta_i \sigma_i^{-2}(\boldsymbol{x}_*),
\end{align}
where $\beta_i=1$ for PoE, and $\beta_i = 1/M$ for gPoE. The original gPoE \citep{cao2014generalized} considers weight $\beta_i$ to be the difference in differential entropy, but this approach limits the computational graph to have a single layer. To allow multiple layer GP aggregation, an average weight $\beta_i$ is proposed in \citep{deisenroth2015distributed}. Since for certain decentralized networks multiple layer GP aggregation is required, we find the average weight more suitable.

\begin{proposition}
\citep[Proposition 1]{liu2018generalized} For a disjoint partition of local datasets $\mathcal{D}_i$, the constant weight of PoE results in overconfident joint variance as the number of observations $N$ tends to infinity, i.e. $\lim_{N \rightarrow \infty}\sigma^2_{\textrm{PoE}}(\boldsymbol{x}_*) \rightarrow 0$. The average weight of gPoE produces a conservative, yet finite joint variance as the number of observations $N$ tends to infinity, i.e. $\sigma_{\textrm{full}}^2 < \lim_{N \rightarrow \infty}\sigma^2_{\textrm{gPoE}}(\boldsymbol{x}_*) < \sigma^2_{**}$, where  $\sigma^2_{\textrm{full}}$ is the target variance of a full GP and $\sigma^2_{**}$ the prior variance.
\end{proposition}

\begin{remark}
It is empirically observed that for a disjoint partition of local datasets $\mathcal{D}_i$, both PoE and gPoE produce inconsistent mean predictions \citep{liu2018generalized}. Specifically, as the number of observations tends to infinity both methods recover the prior mean $\mu_{**}$, i.e. $\lim_{N\rightarrow \infty} \mu_{\textrm{(g)PoE}}(\boldsymbol{x}_{*}) \rightarrow \mu_{**}$ for all $\boldsymbol{x}_{*}$. 
\end{remark}

\begin{proposition}\label{prop:poe_gpoe_mean}
The PoE and gPoE make identical mean predictions \eqref{eq:mean_poe}.
\end{proposition}
\begin{proof}
The proof is provided in Appendix \ref{app:proof:poe_gpoe_mean}.
\end{proof}

The local time computational complexity of both PoE and gPoE is governed by the local variance \eqref{eq:local_variance}, that is $\mathcal{O}(N_i^2) = \mathcal{O}(N^2/M^2)$ for the multiplication of the quadratic term. Provided that the number of local observations is less than the observations from all local entities $N_i < N$, the PoE and gPoE alleviates the computations compared to the full GP $\mathcal{O}(N^2)$. The space complexity requires $\mathcal{O}(N_i^2+DN_i) = \mathcal{O}(N^2/M^2+D(N/M))$ capacity to store the inverse of the local inverted covariance matrix $\boldsymbol{C}_{\theta,i}^{-1}$ and the vector of local observations $\boldsymbol{y}_{i}$. Thus, the space requirement is relaxed compared to full GP that occupies $\mathcal{O}(N^2+DN)$ memory. 
The total communication complexity from all entities to the central node is $\mathcal{O}(2M)$ to transmit all local mean $\mu_i$ and local variance $\sigma^2_i$ values. A comparison of PoE and gPoE with other aggregation methods is presented in Table~\ref{tab:complexityPrediction}.

\begin{table*}[!t]
\caption{Time, Space, and Communication Complexity for Centralized GP Aggregated Prediction}
\centering
\begin{threeparttable}
{\begin{tabular}{ c c c c c c c }
\toprule

 & & FULL-GP & {(g)PoE} \& {BCM} &  rBCM 
 &  grBCM 
 &  \textsc{NPAE} 
 \\
 
 \midrule 
  
 \multirow{2}{*}{Local}   &  Time & - & $\mathcal{O}(\zeta)$ & $\mathcal{O}(\zeta)$ & $\mathcal{O}((5+N/M)\zeta)$ &  $\mathcal{O}(N\zeta)$ \\
 
 &  Space & - & $\mathcal{O}(\xi)$ & $  \mathcal{O}(\xi)$  & $\mathcal{O}(2\xi+2(N^2/M^2))$ & $\mathcal{O}(\xi+DN)$ \\
    
    \cmidrule(lr){1-7}
 
 \multirow{3}{*}{Global}  &   Time & $\mathcal{O}(N^2)$ & $\mathcal{O}(M)$ & $\mathcal{O}(M)$  & $\mathcal{O}(M)$ & $\mathcal{O}(M^3)$ \\
 
 &  Space & $\mathcal{O}(N^2+DN)$ & $\mathcal{O}(2M)$ & $  \mathcal{O}(3M)$  & $\mathcal{O}(5M)$ & $\mathcal{O}(M^2)$ \\
  
   &   Comm & - & $  \mathcal{O}(2M)$ & $  \mathcal{O}(3M)$  & $\mathcal{O}(5M)$ & $\mathcal{O}(M^2)$ \\

\bottomrule
\end{tabular}}
\begin{tablenotes}[para,flushleft]
      \item {$\zeta = N^2/M^2$, $\xi = N^2/M^2+D(N/M)$}. 
    \end{tablenotes}
    \end{threeparttable}
    \label{tab:complexityPrediction}
\end{table*}

\textbf{\textit{BCM Family}}: The BCM, rBCM, and grBCM make an additional assumption other than Assumption~\ref{ass:independence}.

\begin{assumption}\label{ass:cond_independence}
The dataset of every agent $i$ is conditionally independent from any other dataset of agent $j\neq i$ given the posterior distribution $f_*$, i.e. $\mathcal{D}_i \indep \mathcal{D}_j \mid f_*$. 
\end{assumption}

After computing the local mean \eqref{eq:local_mean} and the local variance \eqref{eq:local_variance}, the joint mean and precision of the BCM and rBCM $\mathcal{M}_{\mathcal{A}} = \mathcal{M}_{\textrm{(r)BCM}}$ is provided by,
\begin{align}\label{eq:mean_bcm}
    \mu_{\textrm{(r)BCM}}( \boldsymbol{x}_*) &= \sigma_{\textrm{(r)BCM}}^2(\boldsymbol{x}_*)\sum_{i=1}^M \beta_i \sigma_i^{-2}(\boldsymbol{x}_*)\mu_i(\boldsymbol{x}_*),\\\label{eq:variance_bcm}
    \sigma_{\textrm{(r)BCM}}^2( \boldsymbol{x}_*) &= \sum_{i=1}^M \beta_i \sigma_i^{-2}(\boldsymbol{x}_*) +\left(1-\sum_{i=1}^M\beta_i\right)\sigma_{**}^{-2},
\end{align}
where $\beta_i=1$ for BCM, and $\beta_i = 0.5[\log \sigma_{**}^2 - \log \sigma_i^2(\boldsymbol{x}_*)]$ for rBCM. For rBCM $\beta_i$ describes the difference in the differential entropy between the prior and the posterior distribution. 

\begin{proposition}
\citep[Proposition 1]{liu2018generalized} For a disjoint partition of local datasets $\mathcal{D}_i$, the BCM and rBCM result in overconfident joint variance as the number of observations $N$ tends to infinity, i.e. $\lim_{N \rightarrow \infty}\sigma^2_{\textrm{(r)BCM}}(\boldsymbol{x}_*) \rightarrow 0$.
\end{proposition}

\begin{remark}
It is empirically observed that for a disjoint partition of local datasets $\mathcal{D}_i$, both BCM and rBCM produce inconsistent mean predictions \citep{liu2018generalized}. However, the joint prediction mean of BCM and rBCM converges to the prior mean slower than the PoE and gPoE.
\end{remark}

The time and space complexity of BCM is identical to PoE and gPoE. In addition, BCM requires similar communications with the PoE family. However, the rBCM entails $\mathcal{O}(3M)$ communication complexity to exchange the local mean $\mu_i$, the local variance $\sigma^2_i$, and the difference in the differential entropy between the prior and posterior distribution $\beta_i$. Note that $\beta_i$ in rBCM can be computed by the central node and recovers the communication complexity of PoE and gPoE. Yet, we prefer to express $\beta_i$ as part of the communication exchange, because in the ensuing discussion the central node is removed. A comparison of BCM and rBCM with other aggregation methods is illustrated in Table~\ref{tab:complexityPrediction}.

\textbf{\textit{grBCM}}: The main idea of grBCM is to equip every agent with a new dataset that has global information on the underlying latent function to ensure consistency (Definition~\ref{def:consistency}). Every agent $i$ selects randomly without replacement ${N}_i/M$ data from its local dataset $\mathcal{D}_{i}$ to form the \textit{local sample dataset} $\mathcal{D}_{-i}\in \mathbb{R}^{N_i/M}\subset \mathcal{D}_{i} $. Then, the local sample datasets are communicated to every other agent (Assumption~\ref{ass:parData}) to compose the \textit{communication dataset} $\mathcal{D}_{\textrm{c}} = \{\mathcal{D}_{\textrm{-i}}\}_{i=1}^M = \{ \boldsymbol{X}_{\textrm{c}}, \boldsymbol{y}_{\textrm{c}} \}$. Next, every agent $i$ fuses the communication dataset $\mathcal{D}_{\textrm{c}} \in \mathbb{R}^{N_i}$ with its local dataset $\mathcal{D}_{i}$ to form the \textit{local augmented dataset} $\mathcal{D}_{+i} = \mathcal{D}_{i} \cup \mathcal{D}_{\textrm{c}} \in \mathbb{R}^{2N_i}$. The local augmented dataset $\mathcal{D}_{+i}$ is a new dataset for every agent $i$ that includes the local dataset $\mathcal{D}_{i}$ and the communication dataset $\mathcal{D}_{\textrm{c}}$, providing a global perspective. Note that in \citep{liu2018generalized}, the authors of grBCM suggest to select randomly the communication dataset $\mathcal{D}_{\textrm{c}}$ from the full dataset $\mathcal{D}$. In this paper, we consider a slight variation that does not violate any result towards a network implementation. Thus, the communication dataset $\mathcal{D}_{\textrm{c}}$ is selected by the local datasets $\mathcal{D}_i$ and then fused through information exchange. 

The agents shall use the local augmented dataset $\mathcal{D}_{+i}$ to compute the augmented local mean $\mu_{+i} $ \eqref{eq:local_mean} and the augmented local variance $\sigma^2_{+i}$ \eqref{eq:local_variance}. In addition, grBCM requires the computation of the communication local mean $\mu_{\textrm{c}} $ \eqref{eq:local_mean} and the communication local variance $\sigma^2_{\textrm{c}}$ \eqref{eq:local_variance} using exclusively the communication dataset $\mathcal{D}_{\textrm{c}}$. The joint mean and precision of grBCM $\mathcal{M}_{\mathcal{A}} = \mathcal{M}_{\textrm{grBCM}}$ yield,
\begin{align}\label{eq:mean_grbcm}
\mu_{\textrm{grBCM}}(\boldsymbol{x}_*) &= \sigma_{\textrm{grBCM}}^2(\boldsymbol{x}_*) \left( \sum_{i=1}^M \beta_i\sigma_{+i}^{-2}(\boldsymbol{x}_*) \mu_{+i}(\boldsymbol{x}_*) - \left(\sum_{i=1}^M\beta_i-1\right)\sigma_{\textrm{c}}^{-2}(\boldsymbol{x}_*) \mu_{\textrm{c}}(\boldsymbol{x}_*) \right),\\ \label{eq:variance_grbcm}
\sigma_{\textrm{grBCM}}^{-2}(\boldsymbol{x}_*) &=\sum_{i=1}^M \beta_i\sigma_{+i}^{-2}(\boldsymbol{x}_*) + \left(1-\sum_{i=1}^M\beta_i\right)\sigma_{\textrm{c}}^{-2}(\boldsymbol{x}_*),
\end{align}
where $\beta_1 = 1$ and $\beta_i = 1/2[\log\sigma_{\textrm{c}}^2(\boldsymbol{x}_*)-\log\sigma_{+i}^2(\boldsymbol{x}_*)]$ for $i>2$.

\begin{proposition}\label{prop:grbcmCons}
\citep[Proposition 3]{liu2018generalized} 
For any collection of aggregated prediction mean values $\mu_1(\boldsymbol{x}_*), \ldots, \mu_M(\boldsymbol{x}_*)$ the grBCM is consistent.
\end{proposition}

The local time complexity for grBCM includes the computation of the augmented local variance $\sigma_{+i}^2$, the local communication variance $\sigma^2_{\textrm{c}}$, and the inversion of the communication dataset covariance matrix $\boldsymbol{K}_{\textrm{c}} = k(\boldsymbol{X}_{\textrm{c}},\boldsymbol{X}_{\textrm{c}})$. That is $\mathcal{O}((2N_i)^2+N_i^2+N_i^3) = \mathcal{O}(N^2/M^2(5 +N/M))$. Then, the inverse of the local augmented covariance matrix $\boldsymbol{C}_{\theta,+i}^{-1}$ and the local augmented dataset $\mathcal{D}_{+i}$ occupy $\mathcal{O}((2N_i)^2+D(2N_i)) = \mathcal{O}(2(N^2/M^2+DN/M)+2N^2/M^2)$ space. The total communications from all entities to the central node is $\mathcal{O}(5M)$ to transmit all local augmented means $\mu_{+i}$, local augmented variances $\sigma^2_{+i}$, local communication means $\mu_{\textrm{c}}$, local communication variances $\sigma^2_{\textrm{c}}$, and all differences in the differential entropy $\beta_i$ to the central node. A comparison of grBCM with other aggregation methods is shown in Table~\ref{tab:complexityPrediction}.

The local augmented dataset $\mathcal{D}_{+i}$ is used in factorized GP training (Section~\ref{sssec:factor_training}) to relax the block diagonal matrix approximation induced by Assumption~\ref{ass:independence}, and produce better estimates of the hyperparameter vector \citep{liu2018generalized}. The time, space, and communication complexity of the generalized factorized GP (g-FACT-GP) training is more demanding than the FACT-GP training, yet remains more affordable than the FULL-GP training. A comparison of all centralized GP training methods is presented in Table~\ref{tab:complexityTraining}.

\textbf{\textit{NPAE}}: The main idea of NPAE is to use covariance between sub-models $\mathcal{M}_i$ to ensure consistency. The local computations of \textsc{NPAE} for agent $i$ 
include: i) local prediction mean $\mu_i(\boldsymbol{x}_*) \in \mathbb{R}$ \eqref{eq:local_mean}; ii) $i$-th entry of the cross-covariance vector $[\boldsymbol{k}_A(\boldsymbol{x}_*)]_i \in \mathbb{R}$; and iii) $i$-th row of the covariance $\mathrm{row}_i\{\boldsymbol{C}_{\theta,A}(\boldsymbol{x}_*)\} \in \mathbb{R}^M$. Thus, NPAE requires the local computation of two additional quantities other than \eqref{eq:local_mean}. These are the cross-covariance and the covariance for each agent $i$,
\begin{align}\label{eq:local_covariance}
    [\boldsymbol{k}_A(\boldsymbol{X}_{i},\boldsymbol{x}_*)]_i &=    \boldsymbol{k}_{i,*}^{\intercal}\boldsymbol{C}_{\theta,i}^{-1}\boldsymbol{k}_{i,*},\\
\label{eq:cross_covariance}
    [\mathrm{row}_i\{\boldsymbol{C}_{\theta,A}(\boldsymbol{X}_{i},\boldsymbol{X}_{j},\boldsymbol{x}_*)\}]_j & = \boldsymbol{k}_{i,*}^{\intercal}\boldsymbol{C}_{\theta,i}^{-1}\boldsymbol{C}_{\theta,ij}\boldsymbol{C}_{\theta,j}^{-1}\boldsymbol{k}_{j,*},
\end{align}
where $\boldsymbol{C}_{\theta,ij} = k(\boldsymbol{X}_{i},\boldsymbol{X}_{j})+ \sigma_{\epsilon}^2I_{N_i} \in \mathbb{R}^{N_i \times N_i}$,  $\boldsymbol{k}_{i,*} = (\boldsymbol{X}_i,\boldsymbol{x}_{*})\in \mathbb{R}^{N_i}$, and $\boldsymbol{k}_{j,*} = (\boldsymbol{X}_j,\boldsymbol{x}_{*})\in \mathbb{R}^{N_i}$ for all $j \neq i$.
The next step is to aggregate the local sub-models and obtain the joint prediction mean and variance, 
\begin{align}\label{eq:mean_npae}
    \mu_{\textrm{NPAE}}( \boldsymbol{x}_*) &= \boldsymbol{k}_{A}^{\intercal}\boldsymbol{C}_{\theta,A}^{-1}\boldsymbol{\mu},\\\label{eq:variance_npae}
    \sigma_{\textrm{NPAE}}^2( \boldsymbol{x}_*) &= 
    k_{**} -  \boldsymbol{k}_{A}^{\intercal}\boldsymbol{C}_{\theta,A}^{-1}\boldsymbol{k}_{A}
    ,
\end{align}
where $\boldsymbol{C}_{\theta,A} = \{\mathrm{row}_i\{\boldsymbol{C}_{\theta,A}\}\}_{i=1}^M \in \mathbb{R}^{M \times M}$, $\boldsymbol{k}_{A} = \{\boldsymbol{k}_A(\boldsymbol{X}_{i},\boldsymbol{x}_*)\}_{i=1}^M \in \mathbb{R}^M$, and $\boldsymbol{\mu} = \{\mu_i \}_{i=1}^M\in \mathbb{R}^M$. 

\begin{proposition}\label{prop:consistency}
\citep[Proposition 2]{bachoc2017some} 
For any collection of aggregated prediction mean values $\mu_1(\boldsymbol{x}_*), \ldots, \mu_M(\boldsymbol{x}_*)$ the NPAE is consistent.
\end{proposition}

The local time complexity for NPAE is governed by the computation of all local inverted covariance matrices for every other agent $\boldsymbol{C}_{\theta,j}^{-1}$, $j\neq i$  \eqref{eq:local_covariance} which yields $\mathcal{O}((M-1)N_i^3)=\mathcal{O}(N^3/M^2)$ computations. Additionally, the aggregated covariance $\boldsymbol{C}_{\theta, A}$ needs to be inverted on the central node that entails $\mathcal{O}(M^3)$ computations. The local memory footprint is $\mathcal{O}(N_i^2+DN_i + MDN_i) = \mathcal{O}( N^2/M^2+D(N/M)+DN)$, including the local inverted covariance matrix $\boldsymbol{C}_{\theta,i}^{-1}$, the local dataset $\mathcal{D}_i$, and the inputs of all other datasets $\boldsymbol{X}_j$ for all $j\neq i$. The total communication complexity yields $\mathcal{O}((M+1)M)=\mathcal{O}(M^2)$ governed by the transmission of the row aggregated covariance $\mathrm{row}_i\{\boldsymbol{C}_{\theta,A}\}$, for all $i\in\mathcal{V}$. A major disadvantage of NPAE is the high global time complexity, yet in the ensuing discussion we remove this expensive computation by using decentralized iterative techniques. A comparison of NPAE with other aggregation methods is provided in Table~\ref{tab:complexityPrediction}.

\subsection{Problem Definition}\label{ssec:probDef}
In this section we define the problems we seek to address.

\begin{problem}\label{pr:training1}
Under Assumption~\ref{ass:repeat_connectivity},~\ref{ass:noData}, and~\ref{ass:independence}, solve the optimization problem (P2) to estimate the hyperparmeters $\hat{\boldsymbol{\theta}}$ of the GP for a decentralized network topology.
\end{problem}

\begin{problem}\label{pr:training2}
Under Assumption~\ref{ass:repeat_connectivity},~\ref{ass:parData}, and~\ref{ass:independence}, solve the optimization problem (P2) to estimate the hyperparmeters $\hat{\boldsymbol{\theta}}$ of the GP for a decentralized network topology.
\end{problem}

The difference between Problem~\ref{pr:training1} and \ref{pr:training2} is that the latter allows partial data exchange while the former prohibits any data exchange between agents. Both problems assume a strongly connected network of agents (Assumption~\ref{ass:repeat_connectivity}) and make the independence approximation assumption between local datasets (Assumption~\ref{ass:independence}). 
Recent advancements in distributed GP hyperparameter training \citep{xu2019wireless,xie2019distributed} have addressed the centralized version of Problem~\ref{pr:training1}. The focus of this paper is on decentralized networks without requiring a central coordinator with massive computational and storage capabilities. The decentralized setup  is imposed by Assumption~\ref{ass:noData} and~\ref{ass:parData}.

\begin{problem}\label{pr:prediction1}
Let Assumption~\ref{ass:repeat_connectivity},~\ref{ass:noData},~\ref{ass:independence}, and~\ref{ass:cond_independence} hold, decentralize the computation of the PoE, gPoE, BCM, rBCM, and NPAE using expertise from all agents. In addition, replace Assumption~\ref{ass:noData} with~\ref{ass:parData} and decentralize the computation of grBCM. 
\end{problem}

\begin{problem}\label{pr:prediction2}
Let Assumption~\ref{ass:repeat_connectivity},~\ref{ass:noData},~\ref{ass:independence}, and~\ref{ass:cond_independence} hold, decentralize the computation of the PoE, gPoE, BCM, rBCM, and NPAE using expertise from statistically correlated agents. In addition, replace Assumption~\ref{ass:noData} with \ref{ass:parData} and decentralize the 
grBCM.
\end{problem}

The difference between Problem~\ref{pr:prediction1} and \ref{pr:prediction2} is that the latter involves the joint prediction of only statistically correlated agents. That is because we seek to reduce the communication from distant entities with insignificant statistic correlation to the aggregation.
\section{Centralized GP Training}\label{sec:centTrain}
In this section, we discuss existing centralized methods and propose a centralized technique to address the factorized GP training problem (P2) based on the alternating direction method of multipliers (ADMM) \citep{boyd2011admm}. 

The following Assumption is required for first-order approximation methods.
\begin{assumption}\label{ass:lipschiz}
A function $f:\mathbb{R^{N}}\rightarrow \mathbb{R}$ is Lipschitz continuous with positive parameter $L>0$ if it satisfies,
\begin{equation}
\|\nabla f(\boldsymbol{x}) - \nabla f(\boldsymbol{y}) \|_2 \leq L \| \boldsymbol{x} - \boldsymbol{y}\|_2, \quad \forall \boldsymbol{x}, \boldsymbol{y}.
\end{equation} 
\end{assumption}

\begin{table*}[!t]
\caption{Time, Space, and Communication Complexity of Centralized Factorized GP Training with ADMM-based Optimization Methods}
\centering
\begin{threeparttable}
{\begin{tabular}{ c c c c c c }
\toprule

 & & & c-GP 
 & apx-GP 
 &  gapx-GP (proposed) \\
 
 \midrule 
  
 \multirow{2}{*}{Local}  & &  Time & $\mathcal{O}(s_{\textrm{nest}}^{\textrm{end}}(N^3/M^3))$ & $\mathcal{O}(N^3/M^3)$ & $\mathcal{O}(8(N^3/M^3))$ \\
 
    & &   Space & $\mathcal{O}(\xi)$ & $\mathcal{O}(\xi)$ & $\mathcal{O}(2\xi+2(N^2/M^2))$ \\
    
    \cmidrule(lr){1-6}
 
 \multirow{2}{*}{Global} & ADMM &  Comm  & $\mathcal{O}(s^{\textrm{end}}_{\textrm{c-GP}}M(D+2))$ & $\mathcal{O}(s^{\textrm{end}}_{\textrm{apx-GP}}M(D+2))$ & $\mathcal{O}(s^{\textrm{end}}_{\textrm{gapx-GP}}M(D+2))$ \\
    
    \cmidrule(lr){2-6}
       & Final &  Comm & $\mathcal{O}(N^2/M)$ & $\mathcal{O}(N^2/M)$ & $\mathcal{O}(4(N^2/M))$ \\

\bottomrule
\end{tabular}
\begin{tablenotes}[para,flushleft]
      \item $\xi = N^2/M^2+D(N/M)$. 
    \end{tablenotes}
    }\label{tab:complexityTrainingADMM}
    \end{threeparttable}
\end{table*}
\subsection{Existing Centralized GP Training Methods}\label{ssec:centTrain}
To address the centralized factorized GP training problem (P2) an exact consensus ADMM (c-ADMM) and an inexact proximal consensus ADMM (px-ADMM) have been used in \citep{xu2019wireless,xie2019distributed}. Using the relaxation of the positivity constraint in Remark~\ref{rem:positivity_relax}, (P2) can be expressed as,
\begin{align}\nonumber
\textrm{(P3)} \quad \hat{\boldsymbol{\theta}} = \arg \min_{\boldsymbol{\theta}} \quad & 
\sum_{i=1}^M      \boldsymbol{y}_i^{\intercal}\boldsymbol{C}_{\theta,i}^{-1}\boldsymbol{y}_i + \log |\boldsymbol{C}_{\theta,i}| 
\\ \label{eq:P3_constrCons}
\textrm{s.to} \quad & \boldsymbol{\theta}_{i} = \boldsymbol{z}, \quad \forall i \in \mathcal{V},
\end{align}
where $\boldsymbol{\theta}_i = \{\theta_{1,i}, \hdots, \theta_{D+2,i} \}$ is the local vector of hyperparameters and $\boldsymbol{z} \in \mathbb{R}^{D+2}$ is an auxiliary variable. In other words, constraint \eqref{eq:P3_constrCons} implies that every agent $i$ is allowed to have its own opinion for the hyperparameters $\boldsymbol{\theta}_i$, yet at the end of the optimization all agents must agree on the global vector value $\boldsymbol{z}$. Recognize that (P3) has the same problem formulation with the c-ADMM problem. After formulating the augmented Lagrangian, the c-GP iterative scheme \citep{boyd2011admm} takes the form,
\begin{subequations}
\begin{alignat}{3}\label{eq:cadmm_z}
\boldsymbol{z}^{(s+1)} & = \frac{1}{M} \sum_{i=1}^M \left(\boldsymbol{\theta}_i^{(s)} + \frac{1}{\rho} \boldsymbol{\psi}_i^{(s)} \right),\\ \label{eq:cadmm_theta}
\boldsymbol{\theta}_i^{(s+1)} & = \arg \min_{\boldsymbol{\theta}_i} \left\{ \mathcal{L}_i(\boldsymbol{\theta}_i) +\left(\boldsymbol{\psi}_i^{(s)}\right)^{\intercal} \left(\boldsymbol{\theta}_i - \boldsymbol{z}^{(s+1)}\right)  +\frac{\rho}{2}\left\| \boldsymbol{\theta}_i - \boldsymbol{z}^{(s+1)} \right\|_2^2 \right\},\\ \label{eq:cadmm_psi}
\boldsymbol{\psi}_i^{(s+1)} & = \boldsymbol{\psi}_i^{(s)} + \rho \left(\boldsymbol{\theta}_i^{(s+1)}-\boldsymbol{z}^{(s+1)}\right),
\end{alignat}
\label{eq:cadmm}
\end{subequations}
where $\boldsymbol{\psi}_i \in \mathbb{R}^{D+2}$ is the vector of dual variables of the $i$-th node, $s\in \mathbb{Z}_{\geq 0}$ is the iteration number, and $\rho>0$ is the penalty constant term of the augmented Lagrangian. The steps of c-GP are the following: i) all agents transmit their $\boldsymbol{\theta}_i^{(s)}$ to the central node; ii) the central node updates the global hyperparameter vector $\boldsymbol{z}^{(s+1)}$ \eqref{eq:cadmm_z}; iii) the central node scatters the updated $\boldsymbol{z}^{(s+1)}$ vector; iv) every agent $i$ solves locally the nested optimization problem \eqref{eq:cadmm_theta} to find the local hyperparameter vector $\boldsymbol{\theta}_i^{(s+1)}$; and v) every agent $i$ updates the local dual vector $\boldsymbol{\psi}_i^{(s+1)}$ \eqref{eq:cadmm_psi}.

Let $s_{\textrm{nest}}^{\textrm{end}}$ be the number of iterations required from the nested optimization problem \eqref{eq:cadmm_theta} to converge. The computational complexity of c-GP is cubic in the number of local observations $\mathcal{O}(s_{\textrm{nest}}^{\textrm{end}}N_i^3) = \mathcal{O}(s_{\textrm{nest}}^{\textrm{end}}(N^3/M^3))$. More specifically, the nested optimization problem \eqref{eq:cadmm_theta} requires the evaluation of the local log-likelihood $\mathcal{L}_i(\boldsymbol{\theta}_i)$ at every internal iteration $s_{\textrm{nest}}$ which entails cubic computations to invert the local covariance matrix $\boldsymbol{C}_{\theta , i}^{-1}$ \eqref{eq:local_ll}. The communication complexity to transmit all local hyperparameter vectors yields $\mathcal{O}(s^{\textrm{end}}M(D+2))$. After convergence, every agent $i$ transmits the local inverted covariance matrix $\boldsymbol{C}_{\theta , i}^{-1}$ which yields $\mathcal{O}(MN_i^2)= \mathcal{O}(N^2/M)$ communications. Every agent $i$ occupies $\mathcal{O}(N_i^2+3(D+2)+D(N/M))) = \mathcal{O}(N^2/M^2+DN/M)$ memory to store the local inverted covariance matrix $\boldsymbol{C}_{\theta , i}^{-1}$, the three quantities of c-GP at the previous iteration $\boldsymbol{\theta}_i^{(s)}$, $\boldsymbol{z}^{(s)}$, $\boldsymbol{\psi}_i^{(s)}$, and the local dataset $\mathcal{D}_i$. 

The major disadvantage of c-GP is the time complexity of the nested optimization problem \eqref{eq:cadmm_theta}. To address this issue, the authors in \citep{xie2019distributed} employed the inexact px-ADMM \citep{hong2016convergence} and derived an analytical solution for the case of centralized factorized GP training to form the analytical px-ADMM-GP (apx-GP). Note that apx-GP employs a first-order approximation (linearization) on the local log-likelihood $\mathcal{L}_i$ around $\boldsymbol{z}^{(s+1)}$ which yields, 
\begin{equation}\label{eq:pxADMM_approx}
    \mathcal{L}_i(\boldsymbol{\theta}_i) \approx \nabla_{\boldsymbol{\theta}}^{\intercal}\mathcal{L}_i\left(\boldsymbol{z}^{(s+1)}\right)\left(\boldsymbol{\theta}_i - \boldsymbol{z}^{(s+1)}\right) + \frac{L_i}{2}\left\| \boldsymbol{\theta}_i - \boldsymbol{z}^{(s+1)} \right\|_2^2,
\end{equation}
where $L_i>0$ is a positive Lipschitz constant that satisfies Assumption~\ref{ass:lipschiz} of the local log-likelihood function $\mathcal{L}_i$ for all $i \in \mathcal{V}$. The apx-GP iteration steps are given by,
\begin{subequations}
\begin{alignat}{3}\label{eq:apxadmm_z}
\boldsymbol{z}^{(s+1)} & = \frac{1}{M} \sum_{i=1}^M \bigg(\boldsymbol{\theta}_i^{(s)} + \frac{1}{\rho} \boldsymbol{\psi}_i^{(s)} \bigg),\\\label{eq:apxadmm_theta}
\boldsymbol{\theta}_i^{(s+1)} & =  \boldsymbol{z}^{(s+1)} - \frac{1}{\rho+L_i} \bigg( \nabla_{\boldsymbol{\theta}}\mathcal{L}_i\left(\boldsymbol{z}^{(s+1)}\right) +\boldsymbol{\psi}_i^{(s)} \bigg) \\  \label{eq:apxadmm_psi}
\boldsymbol{\psi}_i^{(s+1)} & = \boldsymbol{\psi}_i^{(s)} + \rho \left(\boldsymbol{\theta}_i^{(s+1)}-\boldsymbol{z}^{(s+1)}\right),
\end{alignat}
\label{eq:apxadmm}
\end{subequations}
where the gradient of the local log-likelihood $\nabla_{\boldsymbol{\theta}}\mathcal{L}_i$ has similar structure to the the gradient of the log-likelihood \eqref{eq:grad_mle}. The only difference on the workflow of apx-GP and c-GP is that step iv) is computed analytically \eqref{eq:apxadmm_theta}, while the former incorporates a nested optimization problem \eqref{eq:cadmm_theta} at every ADMM-iteration. 

The space and communication complexity of apx-GP is identical to c-GP. The time complexity of apx-GP entails $\mathcal{O}(N_i^3) = \mathcal{O}(N^3/M^3)$ computations, significantly reduced from $\mathcal{O}(s_{\textrm{nest}}^{\textrm{end}}N^3/M^3)$ of c-GP. In other words, there is no nested optimization problem in apx-GP \eqref{eq:apxadmm} and thus requires just one inversion of the local covariance matrix $\boldsymbol{C}_{\theta,i}^{-1}$ per ADMM-iteration instead of $s_{\textrm{nest}}^{\textrm{end}}$ inversions per ADMM-iteration in c-GP. Both c-GP and apx-GP inherit the convergence properties of c-ADMM \citep{boyd2011admm} and px-ADMM \citep{hong2016convergence} which result in much faster convergence than gradient descent. 

A disadvantage of both centralized methods \eqref{eq:cadmm} and \eqref{eq:apxadmm} is that they are based on factorized GP training and thus they inherit poor approximation capabilities when the number of nodes increases. More specifically, for a bounded space of interest, Assumption~\ref{ass:independence} is violated as we increase the number of sub-models $\mathcal{M}_i$.
\subsection{Proposed Centralized GP Training}\label{ssec:centTrainNew}
The first method we propose is a centralized factorized GP training technique that extends apx-GP with a local augmented datatset $\mathcal{D}_{+i}$ for all $i\in \mathcal{V}$. The goal is to limit the approximation error of factorized GP training inherited by Assumption~\ref{ass:independence} at the cost of allowing partial local data exchange (Assumption~\ref{ass:parData}). A larger dataset entails more computations, thus we build on the computationally affordable apx-GP method. We term our methodology as generalized apx-GP (gapx-GP).

\newfloat{algorithm}{!t}{lop}
\begin{algorithm}
\caption{gapx-GP}\label{alg:gapx_ADMM_GP} 
{
\textbf{Input:} $\mathcal{D}_i (\boldsymbol{X}_i, \boldsymbol{y}_i)
$, $k(\cdot,\cdot)$, $\rho$, $L_i$, $\mathcal{N}_i$, $\mathcal{V}$, $\textrm{TOL}_{\textrm{ADMM}}$

\textbf{Output:} $\hat{\boldsymbol{\theta}}$, $\boldsymbol{C}_{\theta}^{-1}$, $\mathcal{D}_{+i}$ 
}
\begin{algorithmic}[1]

\ForEach {$i \in \mathcal V $}\Comment{Local Sample Dataset}
\State $\mathcal{D}_{\textrm{c},i} \gets \texttt{Sample}(\mathcal{D}_i)$
\State communicate $\mathcal{D}_{\textrm{c},i}$ to central node
\EndFor
\State scatter $\mathcal{D}_{\textrm{c}} = \{\mathcal{D}_{\textrm{c},i}\}_{i=1}^M$ from central node to every agent

\ForEach {$i \in \mathcal V $}\Comment{Local Augmented Dataset}
\State $\mathcal{D}_{+i} \gets \mathcal{D}_{i} \cup  \mathcal{D}_{\textrm{c}}$ 
\EndFor

\Repeat \Comment{ADMM Optimization}
\State communicate $\boldsymbol{\theta}_i^{(s)}$ to central node 
\State $\boldsymbol{z}^{(s+1)} \gets \texttt{primal-2}(\boldsymbol{\theta}_i^{(s)}, \boldsymbol{\psi}_i^{(s)},\mathrm{card}(\mathcal{V}))$ \eqref{eq:apxadmm_z}
\State scatter $\boldsymbol{z}^{(s+1)}$ from central node to every agent
\ForEach {$i \in \mathcal V $}
\State $\boldsymbol{\theta}_i^{(s+1)} \gets \texttt{primal-1}(\boldsymbol{\theta}_i^{(s)}, \boldsymbol{z}^{(s+1)}, \boldsymbol{\psi}_i^{(s)}, \rho, L_i, \mathcal{D}_{+i})$ \eqref{eq:apxadmm_theta}
\State $\boldsymbol{\psi}_i^{(s+1)} \gets \texttt{dual}(\boldsymbol{\theta}_i^{(s+1)}, \boldsymbol{z}^{(s+1)}, \boldsymbol{\psi}_i^{(s)}, \rho)$ \eqref{eq:apxadmm_psi}
\EndFor
\Until {$\|\boldsymbol{\theta}_i^{(s+1)} - \boldsymbol{z}^{(s+1)}  \|_2< \textrm{TOL}_{\textrm{ADMM}}$, for all $i\in \mathcal{V}$}

\ForEach {$i \in \mathcal V $}\Comment{Local Augmented Covariance Inversion}
\State $\hat{\boldsymbol{\theta}} \gets \boldsymbol{\theta}_i^{\textrm{end}}$
\State $\boldsymbol{C}_{\theta,+i}^{-1} \gets \texttt{invert}(k,\boldsymbol{X}_{+i}, \hat{\boldsymbol{\theta}})$
\State communicate $\boldsymbol{C}_{\theta,+i}^{-1}$ to central node 
\EndFor

\State $\boldsymbol{C}_{\theta}^{-1} \gets \texttt{diag}(\boldsymbol{C}_{\theta,+1}^{-1},\boldsymbol{C}_{\theta,+2}^{-1}, \ldots,\boldsymbol{C}_{\theta,+M}^{-1}  )$ \Comment{Block Diagonal}
\State Return $\hat{\boldsymbol{\theta}}$, $\boldsymbol{C}_{\theta}^{-1}$, $\mathcal{D}_{+i}$
\end{algorithmic}
\end{algorithm}

Let the communication dataset to be formed as discussed in Section~\ref{sssec:aggrPred}-grBCM. Then, every agent $i$ has access to a local augmented dataset which is the union of the corresponding local dataset and the communication dataset $\mathcal{D}_{+i} = \mathcal{D}_{i} \cup \mathcal{D}_{\textrm{c}} \in \mathbb{R}^{2N_i}$. Next, we implement the apx-GP \eqref{eq:apxadmm}, but now every agent is equipped with the local augmented dataset $\mathcal{D}_{+i}$. The implementation details are provided in Algorithm~\ref{alg:gapx_ADMM_GP}.

The local time complexity of gapx-GP yields $\mathcal{O}((2N_i)^3) = \mathcal{O}(8(N^3/M^3))$ computations to invert the local augmented covariance matrix $\boldsymbol{C}_{\theta,+i} = \boldsymbol{K}_{+i} + \sigma_{\epsilon}^2I_{2N_i}\in \mathbb{R}^{2N_i \times 2N_i}$. The total communication overhead is the same with c-GP and apx-GP. After convergence, each agent $i$ communicates the local augmented covariance matrix $\boldsymbol{C}_{\theta,+i}^{-1}$ that entails $\mathcal{O}(M(2N_i)^2) = \mathcal{O}(4(N^2/M))$ communications. The space complexity of every agent $i$ yields $\mathcal{O}((2N_i)^2+3(D+2)+D(2N_i)) = \mathcal{O}(4(N^2/M^2) + 2D(N/M))$ to store the local augmented covariance matrix, the optimization variables at the previous iteration, and the local augmented dataset. 

In Table~\ref{tab:complexityTrainingADMM}, we list the time, space, and communication complexity for all centralized factorized GP training methods based on ADMM. The proposed method is more demanding in space than c-GP. In terms of time complexity, gapx-GP is more affordable than c-GP, because the nested optimization of the latter \eqref{eq:cadmm_theta} takes on average more than eight iterations to converge, i.e., $s_{\textrm{nest}}^{\textrm{end}}>8$, but more demanding than apx-GP. The proposed method supports Assumption~\ref{ass:independence}, and thus we expect to produce more accurate hyperparameters. 

\begin{proposition}
Let  Assumption~\ref{ass:independence} and~\ref{ass:strict_convexity} hold for the local sub-model $\mathcal{M}_i$, then the gapx-GP converges, i.e., $\lim_{s\rightarrow \infty}\|\boldsymbol{\theta}_i^{(s)} - \boldsymbol{z}^{(s)} \|_2 = 0$ for all $i\in \mathcal{V}$, to a stationary solution $(\boldsymbol{\theta}_i^{\star},\boldsymbol{z}^{\star},\boldsymbol{\psi}_i^{\star})$ of (P3).

\proof
The proof is direct consequence of \citep[Theorem 2.10]{hong2016convergence}.

\end{proposition}
\section{Proposed Decentralized GP Training}\label{sec:decentTrain}
In this section, we propose solutions for Problem~\ref{pr:training1} and~\ref{pr:training2} based on the \textit{edge formulation} of ADMM \citep{shi2014linear} that yields parallel updates and decentralizes the factorized GP training. Let Assumption~\ref{ass:repeat_connectivity} hold, then (P3) can be  expressed as,
\begin{align}\nonumber
\textrm{(P4)} \quad \hat{\boldsymbol{\theta}} = \arg \min_{\boldsymbol{\theta}} \quad & 
\sum_{i=1}^M \boldsymbol{y}_i^{\intercal}\boldsymbol{C}_{\theta,i}^{-1}\boldsymbol{y}_i + \log |\boldsymbol{C}_{\theta,i}| 
\\ \label{eq:P4_constrCons_1}
\textrm{s.to} \quad & \boldsymbol{\theta}_{i} = \boldsymbol{\tau}_{ij}, \quad \forall i \in \mathcal{V}, \, j\in \mathcal{N}_i,\\ \label{eq:P4_constrCons_2}
\quad & \boldsymbol{\theta}_{j} = \boldsymbol{\tau}_{ij}, \quad \forall i \in \mathcal{V}, \, j\in \mathcal{N}_i,
\end{align}
where $\boldsymbol{\tau}_{ij}$ are auxiliary variables. Constraints~\eqref{eq:P4_constrCons_1} and \eqref{eq:P4_constrCons_2} imply that every agent $i$ is allowed to have its own opinion for the hyperparameters $\boldsymbol{\theta}_i$, yet at the end of the optimization all agents in the neighborhood $\mathcal{N}_i$ must agree on the neighborhood values $\boldsymbol{\tau}_{ij}$. The edge formulation requires each node $i$ to store and update variables for all of its neighbors $\mathcal{N}_i$. Conversely, one can employ the \textit{node formulation} that relaxes the storage capacity, as each agent $i$ is required to store and update variables of itself \citep{makhdoumi2017convergence}. In addition, the group ADMM \citep{elgabli2020gadmm} offers a decentralized optimization method, yet for a specific graph topology. Thus, we find the edge formulation more suitable for decentralized GP training.

Let us introduce an additional Assumption to study the convergence properties of the proposed methods.
\begin{assumption}\label{ass:strict_convexity}
A function $f:\mathbb{R^{N}}\rightarrow \mathbb{R}$ is strongly convex with positive parameter $m>0$ if it satisfies,
\begin{equation}
\left(\nabla f(\boldsymbol{x}) - \nabla f(\boldsymbol{y}) \right)^{\intercal}(\boldsymbol{x} - \boldsymbol{y} )\geq m \left\| \boldsymbol{x} - \boldsymbol{y} \right\|_2^2, \quad \forall \boldsymbol{x}, \boldsymbol{y}.
\end{equation}
\end{assumption}
\begin{table*}[!t]
\centering
\begin{threeparttable}
\caption{Time, Space, and Communication Complexity of Decentralized Factorized GP Training with ADMM-based Methods}\label{tab:complexityDecTrainingADMM}
{
\begin{tabular}{  c c c c c }
\toprule

 & &  {DEC-c-GP}   & {DEC-apx-GP}  &  {DEC-gapx-GP} \\
 
 \midrule 
  
 \multirow{3}{*}{Local}  &   Time & $\mathcal{O}(s_{\textrm{nest}}^{\textrm{end}}(N^3/M^3))$ & $\mathcal{O}(N^3/M^3)$ & $\mathcal{O}(8(N^3/M^3))$ \\
 
    &    Space & $\mathcal{O}(\xi)$ & $\mathcal{O}(\xi)$ & $\mathcal{O}(2\xi+2(N^2/M^2))$ \\
    
        &    Comm & $\mathcal{O}(s_{\textrm{DEC-c-GP}}^{\textrm{end}}(D+2))$ & $\mathcal{O}(s_{\textrm{DEC-apx-GP}}^{\textrm{end}}(D+2))$ & $\mathcal{O}(s_{\textrm{DEC-gapx-GP}}^{\textrm{end}}(D+2))$ \\

\bottomrule
\end{tabular}
\begin{tablenotes}[para,flushleft]
      \item $\xi = N^2/M^2+D(N/M)$. 
    \end{tablenotes}
    }
    \end{threeparttable}
\end{table*}

\subsection{DEC-c-GP}
This method is based on the decentralized consensus ADMM \citep{mateos2010distributed}. After rendering the augmented Lagrangian for (P4) we obtain the decentralized consensus ADMM iterative scheme,
\begin{subequations}
\begin{alignat}{3}\label{eq:dec_c_admm_p}
\boldsymbol{p}_i^{(s+1)} & = \boldsymbol{p}_i^{(s)} + \rho  \sum_{j\in \mathcal{N}_i} \bigg(\boldsymbol{\theta}_i^{(s)} - \boldsymbol{\theta}_j^{(s)} \bigg),\\  \label{eq:dec_c_admm_theta}
\boldsymbol{\theta}_i^{(s+1)} & = \arg \min_{\boldsymbol{\theta}_i} \left\{ \mathcal{L}_i(\boldsymbol{\theta}_i) + \boldsymbol{\theta}_i^{\intercal}\boldsymbol{p}_i^{(s+1)} + \rho \sum_{j\in \mathcal{N}_i} \left\|\boldsymbol{\theta}_i - \frac{\boldsymbol{\theta}_i^{(s)}+\boldsymbol{\theta}_j^{(s)}}{2} \right\|_2^2 \right\},
\end{alignat}
\label{eq:dec_c_admm_GP}
\end{subequations}
where $\rho>0$ is the penalty term of the augmented Lagrangian and $\boldsymbol{p}_i^{(s)} = \sum_{j\in \mathcal{N}_i}(\boldsymbol{u}_{ij}^{(s)} +\boldsymbol{v}_{ij}^{(s)})$ is the sum of the dual variables $\boldsymbol{u}_{ij}^{(s)}$ and $\boldsymbol{v}_{ij}^{(s)}$ corresponding to constraints \eqref{eq:P4_constrCons_1} and \eqref{eq:P4_constrCons_2}. Note that \eqref{eq:dec_c_admm_p} imposes initial values $\boldsymbol{p}_i^{(0)} = \boldsymbol{0}$.

\newfloat{algorithm}{!t}{lop}
\begin{algorithm}
\caption{{DEC-c-GP}}\label{alg:dec_c_ADMM_GP} 
{
\textbf{Input:} $\mathcal{D}_i (\boldsymbol{X}_i, \boldsymbol{y}_i)
$, $k(\cdot,\cdot)$, $\rho$, $\mathcal{N}_i$, $\alpha$, $s_{\textrm{DEC-c-GP}}^{\textrm{end}}$

\textbf{Output:} $\hat{\boldsymbol{\theta}}$, $\boldsymbol{C}_{\theta,i}^{-1}$
}
\begin{algorithmic}[1]

\State initialize $\boldsymbol{p}_{i}^{(0)} = \boldsymbol{0}$ 

\For  {$s =1$ to $s_{\textrm{DEC-c-GP}}^{\textrm{end}}$}  \Comment{ADMM Optimization}
\ForEach {$i \in \mathcal V $}
\State communicate $\boldsymbol{\theta}_i^{(s)}$ to neighbors $\mathcal{N}_i$ 
\State $\boldsymbol{p}_i^{(s+1)} \gets \texttt{duals}(\boldsymbol{p}_i^{(s)}, \boldsymbol{\theta}_i^{(s)}, \{\boldsymbol{\theta}_j^{(s)}\}_{j\in \mathcal{N}_i}, \rho)$ \eqref{eq:dec_c_admm_p}
\State $\boldsymbol{\theta}_i^{(s+1)} \gets \texttt{primal}(\boldsymbol{p}_i^{(s+1)}, \boldsymbol{\theta}_i^{(s)}, \{\boldsymbol{\theta}_j^{(s)}\}_{j\in \mathcal{N}_i}, \rho, \alpha,\mathcal{D}_i)$ \eqref{eq:dec_c_admm_theta}
\EndFor
\EndFor

\ForEach {$i \in \mathcal V $}\Comment{Local Covariance Inversion}
\State $\hat{\boldsymbol{\theta}} \gets \boldsymbol{\theta}_i^{\textrm{end}}$
\State $\boldsymbol{C}_{\theta,i}^{-1} \gets \texttt{invert}(k,\boldsymbol{X}_{i}, \hat{\boldsymbol{\theta}})$
\EndFor

\State Return $\hat{\boldsymbol{\theta}}$, $\boldsymbol{C}_{\theta,i}^{-1}$
\end{algorithmic}
\end{algorithm}

The workflow is as follows. Every agent $i$ communicates to its neighbors $j\in \mathcal{N}_i$ the current estimate of the hyperparameters $\boldsymbol{\theta}_i^{(s)}$. After each agent gathers all $\boldsymbol{\theta}_j^{(s)}$ vectors from its neighborhood, then the sum of the dual variables vector is updated \eqref{eq:dec_c_admm_p} to obtain $\boldsymbol{p}_i^{(s+1)}$. Next, every agent $i$ solves a nested optimization problem \eqref{eq:dec_c_admm_theta} to compute $\boldsymbol{\theta}_i^{(s+1)}$. The method iterates until it reaches a predefined maximum iteration number $s_{\textrm{DEC-c-GP}}^{\textrm{end}}$. The main routine of DEC-c-GP is provided in Algorithm~\ref{alg:dec_c_ADMM_GP}. The proposed method is decentralized (executed in parallel), requiring exclusively neighbor-wise communication as shown in Fig.~\ref{fig:dec_train_gp_concept}-(a). Note that the inter-agent communications do not involve any data exchange which satisfies Assumption \ref{ass:noData}. Provided that the graph topology is connected (Assumption \ref{ass:repeat_connectivity}), then DEC-c-GP \eqref{eq:dec_c_admm_GP} addresses Problem~\ref{pr:training1}. 

Let the total number of iterations for the nested optimization problem \eqref{eq:dec_c_admm_theta} be  $s_{\textrm{nest}}^{\textrm{end}}$. The time complexity of every agent $i$ is dominated by the inverse of the local covariance matrix $\boldsymbol{C}_{\theta,i}^{-1}$ for every iteration of the nested optimization problem \eqref{eq:dec_c_admm_theta}, which results in $\mathcal{O}(s_{\textrm{nest}}^{\textrm{end}} N_i^3) = \mathcal{O}(s_{\textrm{nest}}^{\textrm{end}} (N^3/M^3))$ computations. The gradient for the nested optimization is provided in Appendix~\ref{app:grad_dec_c_admm}. Moreover, every agent $i$ occupies $\mathcal{O}(N_i^2+DN_i+(D+2)+(\mathrm{card}(\mathcal{N}_i)+1)(D+2)) = \mathcal{O}(N^2/M^2+D(N/M)+(\mathrm{card}(\mathcal{N}_i)+2)(D+2))$ memory to store the local inverted covariance matrix $\boldsymbol{C}_{\theta,i}^{-1}$, the local dataset $\mathcal{D}_i$, the sum of dual variables vector at the previous iteration $\boldsymbol{p}_{i}^{(s)}$, the hyperparameter vector at the previous iteration $\boldsymbol{\theta}_{i}^{(s)}$, and the hyperparameter vectors of all neighbors at the previous iteration $\{\boldsymbol{\theta}_{j}^{(s)}\}_{j \in \mathcal{N}_i}$.  The total number of communications for each agent is $\mathcal{O}(s_{\textrm{DEC-c-GP}}^{\textrm{end}}(D+2))$ to transmit the hyperparameters to its neighbors. 

\begin{figure}[!t]
	\includegraphics[width=.65\columnwidth]{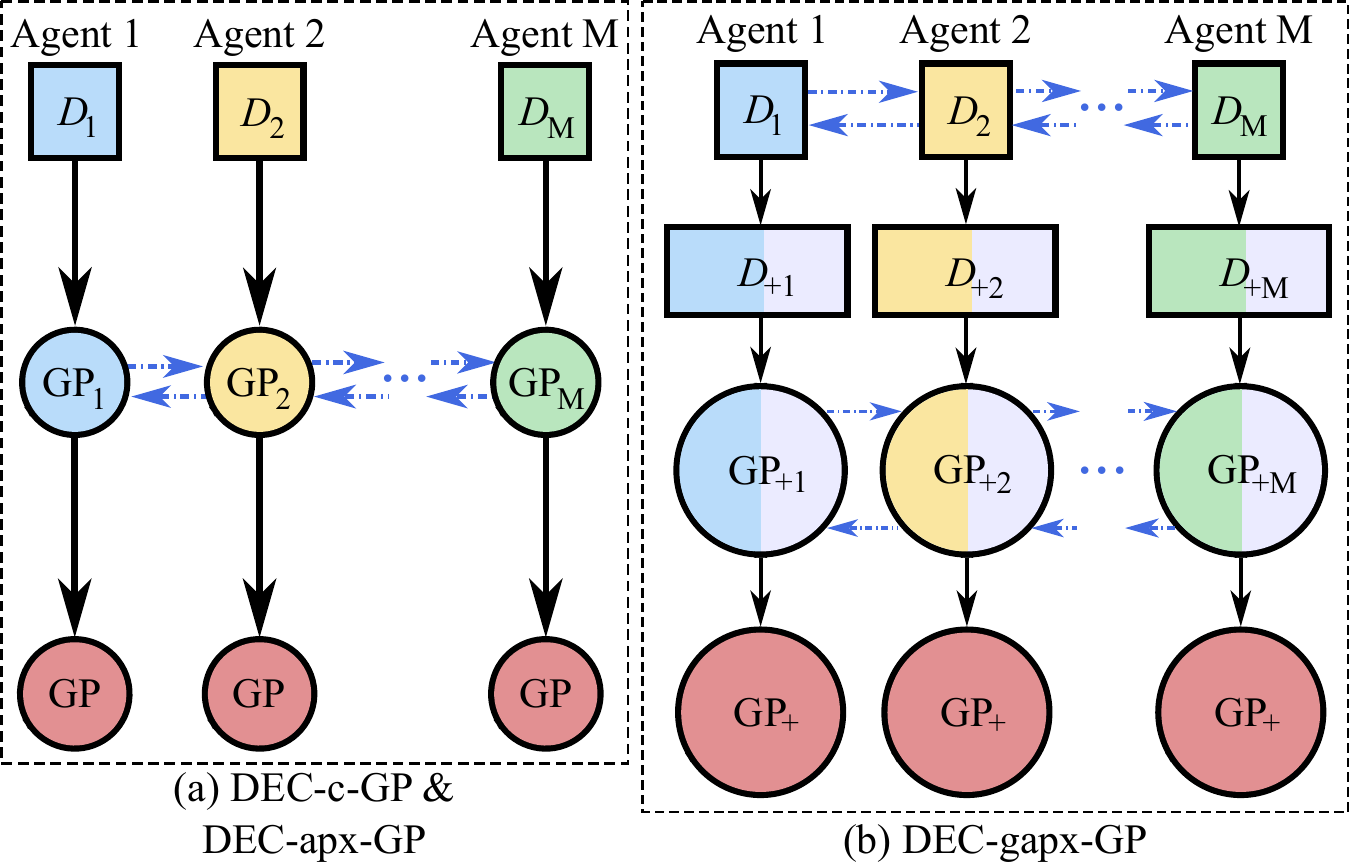}
	\centering
	\caption{The structure of the proposed decentralized factorized GP training methods. Blue dotted lines correspond to communication (strongly connected). a) Every agent $i$ has access to the local dataset $\mathcal{D}_i$. The agents are allowed to have their own opinion on the hyperparameter $\boldsymbol{\theta}_i$ using exclusively $\mathcal{D}_i$, but after communicating they all agree on the same hyperparameters $\boldsymbol{\theta}$. b) Every agent $i$ has access to $\mathcal{D}_i$. Next, they communicate to form the local augmented dataset $\mathcal{D}_{+i}$ which comprises of $\mathcal{D}_i$ (local color) and the global communication dataset $\mathcal{D}_{\textrm{c}}$ (gray color). The agents are allowed to have their own opinion on the hyperparameter $\boldsymbol{\theta}_i$ using exclusively $\mathcal{D}_{+i}$, but after communicating they all agree on the same hyperparameters $\boldsymbol{\theta}$.
	}
	\label{fig:dec_train_gp_concept}
\end{figure}

\begin{proposition}
Let Assumption~\ref{ass:repeat_connectivity},~\ref{ass:noData},~\ref{ass:independence}, and \ref{ass:strict_convexity} hold for the local sub-model $\mathcal{M}_i$, then the DEC-c-GP \eqref{eq:dec_c_admm_GP} converges to a stationary solution $\lim_{s\rightarrow \infty} \boldsymbol{\theta}_i^{(s)} = \boldsymbol{\theta}^{\star}$ of (P4) for all $i\in \mathcal{V}$.

\proof
The proof is direct application of \citep[Proposition 2]{mateos2010distributed}.

\end{proposition}

\begin{remark}\label{rem:con_c-ADMM}
The main disadvantage of the proposed DEC-c-GP method is the cubic computations on the number of local observations for every iteration of the nested optimization \eqref{eq:dec_c_admm_theta}, which results in a computationally demanding process.
\end{remark}
\subsection{DEC-apx-GP}
To address the computational complexity problem of DEC-c-GP (Remark~\ref{rem:con_c-ADMM}) we consider an inexact proximal step based on a first-order approximation on the local log-likelihood $\mathcal{L}_i$ around $\boldsymbol{\theta}^{(s)}$ which yields,
\begin{equation}\label{eq:dec_px_ADMM_approx}
    \mathcal{L}_i(\boldsymbol{\theta}_i) \approx \nabla_{\boldsymbol{\theta}}^{\intercal}\mathcal{L}_i\left(\boldsymbol{\theta}_i^{(s)}\right)\left(\boldsymbol{\theta}_i - \boldsymbol{\theta}_i^{(s)}\right) + \frac{\kappa_i}{2}\left\| \boldsymbol{\theta}_i - \boldsymbol{\theta}_i^{(s)} \right\|_2^2,
\end{equation}
where $\kappa_i>0$ is a positive constant for all $i \in \mathcal{V}$. To this end, we obtain the DEC-px-ADMM \citep{chang2014multi} iterative scheme,
\begin{align}\nonumber
\boldsymbol{p}_i^{(s+1)} & = \boldsymbol{p}_i^{(s)} + \rho  \sum_{j\in \mathcal{N}_i} \bigg(\boldsymbol{\theta}_i^{(s)} - \boldsymbol{\theta}_j^{(s)} \bigg),\\ \nonumber
\boldsymbol{\theta}_i^{(s+1)} & = \arg \min_{\boldsymbol{\theta}_i} \bigg\{ \nabla_{\boldsymbol{\theta}}^{\intercal}\mathcal{L}_i\left(\boldsymbol{\theta}_i^{(s)}\right)\left(\boldsymbol{\theta}_i - \boldsymbol{\theta}_i^{(s)}\right) + \frac{\kappa_i}{2}\left\| \boldsymbol{\theta}_i - \boldsymbol{\theta}_i^{(s)} \right\|_2^2  \\ \label{eq:dec-pxadmm_theta}
& \quad \quad \quad \quad \quad \ \ + \boldsymbol{\theta}_i^{\intercal}\boldsymbol{p}_i^{(s+1)} + \rho \sum_{j\in \mathcal{N}_i} \left\|\boldsymbol{\theta}_i - \frac{\boldsymbol{\theta}_i^{(s)}+\boldsymbol{\theta}_j^{(s)}}{2} \right\|_2^2 \bigg\}.
\end{align}

Essentially, the linearization \eqref{eq:dec_px_ADMM_approx} allows the evaluation of the local log-likelihood function $\mathcal{L}_i$ \eqref{eq:local_ll} at a fixed point $\boldsymbol{\theta}_i^{(s)}$ and not at the optimizing variable $\boldsymbol{\theta}_i$. To this end, the nested optimization of \eqref{eq:dec-pxadmm_theta} entails significantly less computations than \eqref{eq:dec_c_admm_theta}. For the special case of factorized GP training problem (P4), an analytical solution of \eqref{eq:dec-pxadmm_theta} can be derived.

\newfloat{algorithm}{!t}{lop}
\begin{algorithm}
\caption{{DEC-apx-GP}}\label{alg:dec_apx_ADMM_GP} 
{
\textbf{Input:} $\mathcal{D}_i (\boldsymbol{X}_i, \boldsymbol{y}_i)
$, $k(\cdot,\cdot)$, $\rho$, $\mathcal{N}_i$, $\kappa_i$, $s_{\textrm{DEC-apx-GP}}^{\textrm{end}}$

\textbf{Output:} $\hat{\boldsymbol{\theta}}$, $\boldsymbol{C}_{\theta,i}^{-1}$
}
\begin{algorithmic}[1]

\State Identical to Algorithm~\ref{alg:dec_c_ADMM_GP} with \eqref{eq:dec_c_admm_theta} replaced by \eqref{eq:dec_apx_admm_theta}
\end{algorithmic}
\end{algorithm}

\begin{theorem}\label{thrm:dec_apx_admm_gp}
Let Assumption~\ref{ass:repeat_connectivity},~\ref{ass:noData},~\ref{ass:independence},~\ref{ass:lipschiz}, and \ref{ass:strict_convexity} hold for the local sub-model $\mathcal{M}_i$. Suppose that the penalty term of the first-order approximation $\kappa_i$ is sufficiently large,
\begin{equation}\label{eq:thrm_dec_apx_convergence_cond}
\kappa_i>\frac{L_i^2}{m_i^2} -\rho\underline{\lambda}(\boldsymbol{D}+\boldsymbol{A})>0, \quad \forall i \in \mathcal{V}.
\end{equation}
Then, the hyperparameter update \eqref{eq:dec-pxadmm_theta} admits a closed-form solution, resulting in the iterative scheme of~DEC-apx-GP,
\begin{subequations}
\begin{alignat}{3}\label{eq:dec-apxadmm_p}
\boldsymbol{p}_i^{(s+1)} & = \boldsymbol{p}_i^{(s)} + \rho  \sum_{j\in \mathcal{N}_i} \bigg(\boldsymbol{\theta}_i^{(s)} - \boldsymbol{\theta}_j^{(s)} \bigg),\\  \label{eq:dec_apx_admm_theta}
\boldsymbol{\theta}_i^{(s+1)} &=  \frac{1}{\kappa_i+2\mathrm{card}(\mathcal{N}_i)\rho } \left( \rho \sum_{j \in \mathcal{N}_i}\boldsymbol{\theta}_j^{(s)} -\nabla_{\boldsymbol{\theta}}\mathcal{L}_i\left(\boldsymbol{\theta}_i^{(s)}\right)   + (\kappa_i+\mathrm{card}(\mathcal{N}_i)\rho )\boldsymbol{\theta}_i^{(s)} - \boldsymbol{p}_i^{(s+1)}\right),
\end{alignat}
\label{eq:dec_apx_admm_gp}
\end{subequations}
that converges to a stationary solution $(\boldsymbol{\theta}_i^{\star},\boldsymbol{p}^{\star})$ of (P4) for all local entities $i\in \mathcal{V}$.

\proof
The proof is provided in Appendix~\ref{app:proof:thrm:dec_apx_admm_gp}.

\end{theorem}

The condition to select the penalty parameter $\kappa_i$ \eqref{eq:thrm_dec_apx_convergence_cond} depends on the graph topology. This implies that the stronger the network the faster the convergence. 

The workflow of DEC-apx-GP is identical to DEC-c-GP, yet the hyperparameter update step \eqref{eq:dec_apx_admm_theta} is performed analytically without requiring a nested optimization update as in \eqref{eq:dec_c_admm_theta} or \eqref{eq:dec-pxadmm_theta}. Implementation details are given in Algorithm~\ref{alg:dec_apx_ADMM_GP} and the structure is illustrated in Fig.~\ref{fig:dec_train_gp_concept}-(a). The gradient of the local log-likelihood $\nabla_{\boldsymbol{\theta}}\mathcal{L}_i$ can be computed similarly to~\eqref{eq:grad_mle}. The proposed iterative method \eqref{eq:dec_apx_admm_gp} tackles Problem \ref{pr:training1}.

The local time complexity of DEC-apx-GP is reduced to $\mathcal{O}(N_i^3) = \mathcal{O}(N^3/M^3)$ for the inversion of the local covariance matrix $\boldsymbol{C}_{\theta,i}^{-1}$ just once at every ADMM iteration. The space complexity is identical to DEC-c-GP and the total communications entail $\mathcal{O}(s_{\textrm{DEC-apx-GP}}^{\textrm{end}}(D+2))$ messages.

\begin{remark}\label{rem:con_apx_ADMM_GP}
A disadvantage of both decentralized methods DEC-c-GP and DEC-apx-GP is the poor approximation capabilities when the number of agents increases, similarly to Section~\ref{ssec:centTrain}. In particular, Assumption~\ref{ass:independence} is violated as we increase the number of sub-models $\mathcal{M}_i$.
\end{remark}

\newfloat{algorithm}{!t}{lop}
\begin{algorithm}
\caption{{DEC-gapx-GP}}\label{alg:dec_gapx_ADMM_GP} 
{
\textbf{Input:} $\mathcal{D}_i (\boldsymbol{X}_i, \boldsymbol{y}_i)
$, $k(\cdot,\cdot)$, $\rho$, $\mathcal{N}_i$, $\kappa_i$, $s_{\textrm{DEC-gapx-GP}}^{\textrm{end}}$

\textbf{Output:} $\hat{\boldsymbol{\theta}}$, $\boldsymbol{C}_{\theta,+i}^{-1}$, $\mathcal{D}_{+i}$ 
}
\begin{algorithmic}[1]

\ForEach {$i \in \mathcal V $}
\State $\mathcal{D}_{\textrm{c},i} \gets \texttt{Sample}(\mathcal{D}_i)$
\State $\mathcal{D}_{\textrm{c}} \gets \texttt{flooding}(\mathcal{D}_{\textrm{c},i})$
\State $\mathcal{D}_{+i} = \mathcal{D}_{i} \cup \mathcal{D}_{\textrm{c}}$
\EndFor

\State $\boldsymbol{C}_{\theta,+i}^{-1} \gets \texttt{DEC-apx-GP}(\mathcal{D}_{+i}, k, \rho, \mathcal{N}_i, \kappa_i, s_{\textrm{DEC-gapx-GP}}^{\textrm{end}})$
\State Return $\hat{\boldsymbol{\theta}}$, $\boldsymbol{C}_{\theta,+i}^{-1}$, $\mathcal{D}_{+i}$
\end{algorithmic}
\end{algorithm}
\subsection{DEC-gapx-GP} 
We propose to extend the computationally efficient DEC-apx-GP method with a local augmented dataset $\mathcal{D}_{+i}$ for all $i\in \mathcal{V}$ to address the poor approximation capabilities of \eqref{eq:dec_c_admm_GP} and \eqref{eq:dec_apx_admm_gp} when the network has large number of nodes (Remark~\ref{rem:con_apx_ADMM_GP}). The idea is similar to the centralized gapx-GP method as presented in Section~\ref{ssec:centTrainNew}. In order to reduce the approximation error, we relax Assumption~\ref{ass:noData} by allowing exchange of local subsets of data Assumption~\ref{ass:parData}. We termed the proposed method as generalized DEC-apx-GP (DEC-gapx-GP).

Since the network has a decentralized topology, flooding \citep{topkis1985concurrent} is employed to broadcast the local sample datasets $\mathcal{D}_{\textrm{c},i}$ and form the communication dataset $\mathcal{D}_{\textrm{c}}$. The rest is a direct application of DEC-apx-GP with the local augmented datatset $\mathcal{D}_{+i}$ for all $i\in\mathcal{V}$. Algorithm \ref{alg:dec_gapx_ADMM_GP} presents the implementation details of DEC-gapx-GP. In Fig.~\ref{fig:dec_train_gp_concept}-(b) the structure of the proposed method is illustrated. Larger circular objects indicate that the augmented covariance matrices $\boldsymbol{C}_{\theta,+i}$ have double dimension, i.e., $2N_i \times 2N_i$ for all $i\in \mathcal{V}$. In addition, the larger rectangular blocks represent the double size local augmented datasets $\mathcal{D}_{+i}\in \mathbb{R}^{2N_i}$. The proposed method addresses Problem~\ref{pr:training2}.

The local time complexity of gapx-GP entails $\mathcal{O}((2N_i)^3) = \mathcal{O}(8(N^3/M^3))$ computations to invert the local augmented covariance matrix $\boldsymbol{C}_{\theta,+i} = \boldsymbol{K}_{+i} + \sigma_{\epsilon}^2I_{2N_i}\in \mathbb{R}^{2N_i \times 2N_i}$. The proposed method requires $\mathcal{O}((2N_i)^2+D(2N_i)+(\mathrm{card}(\mathcal{N}_i)+2)(D+2)) = \mathcal{O}(4(N^2/M^2)+2D(N/M)+(\mathrm{card}(\mathcal{N}_i)+2)(D+2))$ space to store the local augmented covariance matrix $\boldsymbol{C}_{\theta,+i}^{-1}$, the local augmented dataset $\mathcal{D}_{+i}$, the sum of dual variables vector at the previous iteration $\boldsymbol{p}_{i}^{(s)}$, the hyperparameter vector at the previous iteration $\boldsymbol{\theta}_{i}^{(s)}$, and the hyperparameter vectors of all neighbors at the previous iteration $\{\boldsymbol{\theta}_{j}^{(s)}\}_{j \in \mathcal{N}_i}$. The total communication overhead is $\mathcal{O}(s_{\textrm{DEC-gapx-GP}}^{\textrm{end}}(D+2))$.

In Table~\ref{tab:complexityDecTrainingADMM}, we list the time, space, and communication complexity for the proposed decentralized factorized GP training methods. The DEC-c-GP is the most computationally expensive method, but it requires less communications than the other methods to converge. Therefore, the DEC-c-GP method favors applications with significant computational resources on the local nodes. Note that this method can also be extended with local augmented dataset $\mathcal{D}_{+i}$ for all $i\in \mathcal{V}$. Next, the DEC-apx-GP is the computationally most  affordable method. The DEC-gapx-GP stands between the two former methods on time complexity, but requires more space. However, the latter can produce more accurate estimates of the hyperparameters.  

\begin{remark}\label{rem:multi_modal}
Assumption~\ref{ass:strict_convexity} requires the local log-likelihood function $\mathcal{L}_i$ to be strongly convex. Similarly to the global log-likelihood $\mathcal{L}$, this is not guaranteed for the local log-likelihoods $\mathcal{L}_i$ for all $i\in\mathcal{V}$. Usually $\mathcal{L}_i$ is nonconvex with respect to the hyperparameters $\boldsymbol{\theta}_i$  \citep{mackay1998introduction,rasmussen2006gaussian,perez2013gaussian}. This is a well known issue of GP hyperparameter training with MLE. A common trick to address the nonconvexity problem is to use multiple starting points to the optimization problem \citep{rasmussen2006gaussian,chen2018priors,basak2021numerical}. Consequently, in this work we follow a similar approach. Note that as we increase the observations the local log-likelihoods tend to be unimodal distributions around the hyperparameters, and thus Assumption~\ref{ass:strict_convexity} is satisfied \citep{mackay1998introduction}. 
\end{remark}

\begin{remark}\label{rem:con_dec_xx_ADMM}
There is no condition to evaluate the termination of the decentralized algorithms, c-GP, apx-GP, and gapx-GP. To this end, we resort to predetermined number of iterations $s^{\textrm{end}}$ which imposes additional computations, storage, and neighbor-wise communications from each agent.
\end{remark}
\section{Proposed Decentralized GP Prediction}\label{sec:decentPred}
In this section, we discuss the discrete-time average consensus (DAC) method \citep{olfati2007consensus}, the Jacobi over-relaxation method (JOR) \citep[Ch. 2.4]{bertsekas2003parallel}, and a distributed algorithm to solve systems of linear equations (DALE) \citep{wang2016improvement,liu2017asynchronous}. In addition, we introduce a technique to identify statistically correlated agents for the location of interest. We combine these tools to approximate the aggregation of GP experts methods in a decentralized fashion. 
\subsection{Decentralized Aggregation Methods}
\textbf{\textit{DAC}}: The DAC is an iterative and parallel method to compute the average of a vector $\boldsymbol{w}\in\mathbb{R^M}$ within a network. More specifically, every agent $i$ has access to one element $w_i \in \mathbb{R}$ and the goal is to compute the average $\bar{\boldsymbol{w}} = (1/M)\sum_{i=1}^Mw_i$. The DAC update law yields,
\begin{equation}\label{eq:dac}
    w_i^{(s+1)} = w_i^{(s)} + \epsilon \sum_{j\in \mathcal{N}_i(t)} a_{ij}(t) \left( w_j^{(s)}- w_i^{(s)}\right),
\end{equation}
where $\epsilon$ is the parameter of the Perron matrix and $a_{ij}(t)$ is the ($i,j$)-th entry of the adjacency matrix. Use of consensus protocols implicitly requires that each node can distributively determine convergence in the network. In other words, just because an agent converged, that does not imply that the network has reached consensus. We employ a maximin stopping criterion \citep{yadav2007distributed} to locally detect convergence in the network. An additional assumption is required to implement the DAC.

\begin{assumption}\label{ass:fleet}
The total number of agents $M$ is known. 
\end{assumption}

\begin{lemma}\label{lem:dac_convergence}
\citep[Theorem 2]{olfati2007consensus}, \citep[Corollary 5.2]{olshevsky2009convergence} Let Assumption \ref{ass:repeat_connectivity} hold. If~$\epsilon\in (0,1/\Delta)$, then the DAC \eqref{eq:dac} converges to the average $\bar{\boldsymbol{w}}$ for any initialization $w_i^{(0)}$ with convergence time $T_M(\epsilon) = \mathcal{O}(M^3 \log(M/\epsilon))$.
\end{lemma}

\textbf{\textit{JOR}}: The JOR is an iterative and parallel method to solve a system of linear algebraic equations in the form of $\boldsymbol{H} \boldsymbol{q} = \boldsymbol{b}$, where $\boldsymbol{H} = [h_{ij}]\in \mathbb{R}^{M \times M}$ is a known non-singular matrix with non-zero diagonal entries $h_{ii}\neq 0$, $\boldsymbol{b}\in \mathbb{R}^{M}$ is a known vector, and $\boldsymbol{q}\in \mathbb{R}^{M}$ is an unknown vector. More specifically, the $i$-th node knows: i) the $i$-th row of the known matrix $\mathrm{row}_i\{\boldsymbol{H}\}\in \mathbb{R}^{1\times M}$; and ii) the $i$-th element of the known vector $b_i\in \mathbb{R}$. The objective is to find $q_i\in \mathbb{R}$, the $i$-th element of the unknown vector $\boldsymbol{q}$. The JOR iterative scheme yields,
\begin{align}\label{eq:jor}
    q_i^{(s+1)} = (1-\omega)q_i^{(s)} + \frac{\omega}{h_{ii}} \left(b_i - \sum_{j\neq i } h_{ij}q_j^{(s)}   \right),
\end{align}
where $\omega \in (0,1)$ the relaxation parameter. 

\begin{figure}[!t]
	\includegraphics[width=.75\columnwidth]{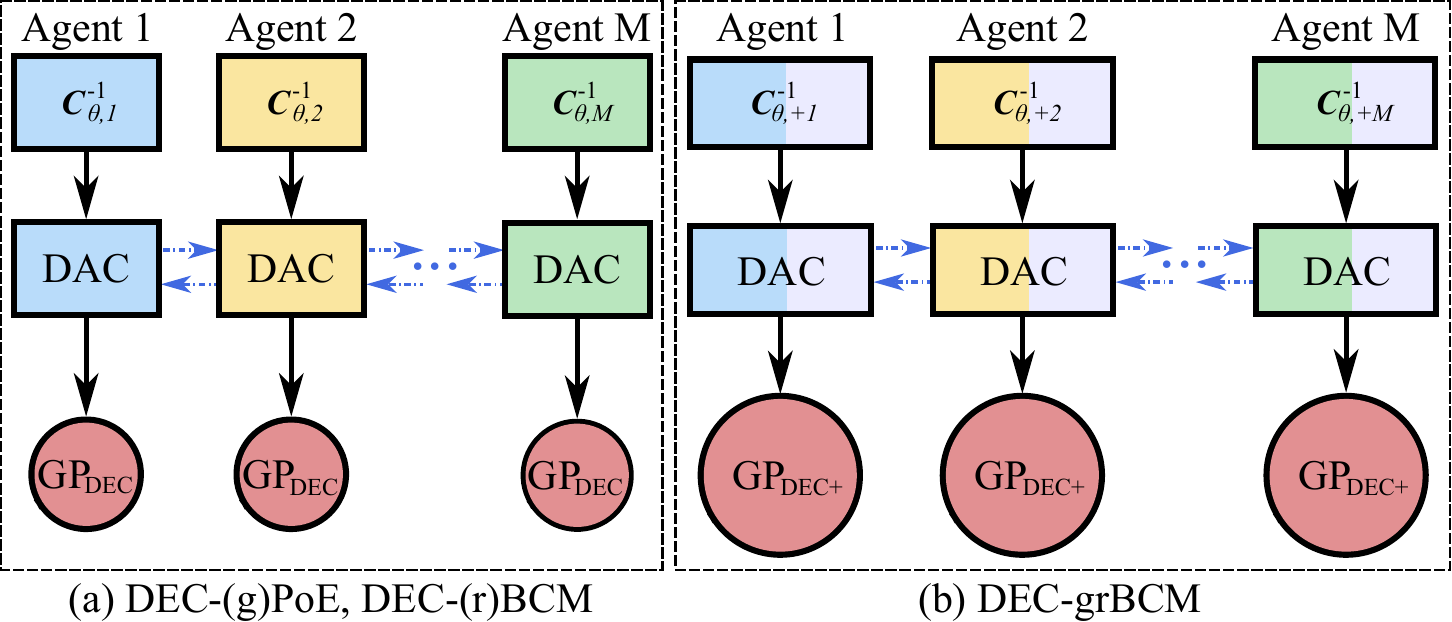}
	\centering
	\caption{The structure of the proposed DEC-PoE and DEC-BCM families. Blue dotted lines correspond to communication (strongly connected). Every agent implements discrete-time average consensus (DAC) methods.
	}
	\label{fig:dec_gp}
\end{figure}

\begin{remark}\label{rem:jor_neighbors}
The limit of the summation in \eqref{eq:jor} requires communication with all agents, as it is computed over $j$ other than $i$. This means that each agent must know the update value \eqref{eq:jor} of every other agent  $\{\boldsymbol{q}_j^{(s)}\}_{j\neq i}$, i.e. $j\neq i \implies j\in \mathcal{V}\backslash i$. That is a major restriction, as it imposes a strongly complete graph topology (Figure~\ref{fig:graph_mixed}). 
Although in \citep{cortes2009distributed,choi2015distributed,choudhary2017distributed} JOR is used for distributed networks, it is unrealistic for many network applications due to limited communication. However, we evaluate the use of JOR, as in some applications with small fleet size, strongly complete networks are feasible. For not strongly complete network topologies, distributed flooding is required at every iteration to obtain $\{\boldsymbol{q}_j^{(s)}\}_{j\neq i}$ and implement~\eqref{eq:jor}. 
The number of inter-agent communications for distributed flooding is the diameter of the graph $\mathrm{diam}(\mathcal{G})$. Thus, the total number of iterations yields $s_{\mathrm{JOR}} = \mathrm{diam}(\mathcal{G})s_{\mathrm{JOR}}^{\mathrm{end}}$. 
\end{remark} 

\begin{lemma}\label{lem:jor_convergence}
\citep[Theorem 2]{udwadia1992some} Let the graph $\mathcal{G}$ be time-invariant and strongly complete. If $\boldsymbol{H}$ is symmetric and PD, and $\omega<2/M$, then the JOR converges to the solution for any initialization $q_i^{(0)}$.
\end{lemma}

\begin{lemma}\label{lem:jor_convergence_optimal}
\citep[Theorem 4]{udwadia1992some} Let the graph $\mathcal{G}$ be time-invariant and strongly complete. If $\boldsymbol{H}$ is symmetric and PD, and~$\omega^{\star} = 2/(\overline{\lambda}(\boldsymbol{R})+\underline{\lambda}(\boldsymbol{R}))$ where $\boldsymbol{R} = \mathrm{diag}(\boldsymbol{H})^{-1}\boldsymbol{H}$, then the JOR converges to the solution for any initialization $q_i^{(0)}$ with the optimal rate.
\end{lemma}

\begin{remark}\label{rem:jor_jor_star}
The difference between Lemma~\ref{lem:jor_convergence} and \ref{lem:jor_convergence_optimal} is that the latter employs the optimal relaxation factor $\omega^{\star}$ which is characterized by the eigenvalues of $\boldsymbol{R}$. In principle, the smaller the relaxation factor $\omega$ the slower the convergence speed \citep{guo2015jor}. Since $\omega^*>2/M$, the optimal relaxation leads to faster convergence of JOR to the solution. To compute $\omega^{\star}$ in a network of agents, additional communication is required to distributively estimate the maximum and minimum eigenvalues of $\boldsymbol{R}$. However, the sufficient condition for $\omega$ of Lemma \ref{lem:jor_convergence} can be locally computed with no communication. Let the distributed method for the computation of $\omega^{\star}$ entail $s^{\textrm{end}}_{\omega^{\star}}$ iterations, JOR with $\omega$ from Lemma \ref{lem:jor_convergence} converge after $s^{\textrm{end}}_{\textrm{JOR}}$ iterations, and JOR with $\omega^{\star}$ from from Lemma \ref{lem:jor_convergence_optimal} converge after $s^{\textrm{end}}_{\textrm{JOR}^{\star}}$ iterations. Then, $\omega^{\star}$ is communication-wise more efficient in decentralized networks when $s^{\textrm{end}}_{\omega^{\star}}+s^{\textrm{end}}_{\textrm{JOR}^{\star}}<s^{\textrm{end}}_{\textrm{JOR}}$. 
\end{remark}

\begin{table}[!t]
\caption{Communication Complexity of Decentralized GP Aggregations}
\centering
\begin{threeparttable}
{\begin{tabular}{ l c c }
\toprule

 Method & Graph & Communication Complexity\\
 
 \midrule 
  
  DEC-PoE  &  SC & $\mathcal{O}(2\chi)$\\ 
 
  DEC-gPoE  & SC & $\mathcal{O}(2\chi)$\\
 
  DEC-BCM  &   SC & $\mathcal{O}(2\chi)$\\
  
  DEC-rBCM & SC & $\mathcal{O}(3\chi)$\\
  
  DEC-grBCM  & SC &  $\mathcal{O}(3\chi)$\\
 
  DEC-NPAE   &  SCC &  $\mathcal{O}(2Ms^{\textrm{end}}_{\textrm{JOR}}+2\chi+M\xi)$\\
  
  DEC-NPAE$^{\star}$ & SCC &  $\mathcal{O}(2M(s^{\textrm{end}}_{\textrm{JOR}^{\star}}+s^{\textrm{end}}_{\textrm{PM}})+2\chi +M\xi +M^2)$\\

\bottomrule
\end{tabular}
\begin{tablenotes}[para,flushleft]
      \item 
      $\chi = s^{\textrm{end}}_{\textrm{DAC}}\mathrm{card}(\mathcal{N}_i)$, $\xi = N^2/M^2+D(N/M)$, SC: strongly connected, SCC: strongly complete connected.
    \end{tablenotes}
    }\label{tab:comms_dec_aggrations}
    \end{threeparttable}
\end{table}

\textbf{\textit{PM}}: The optimal relaxation factor $\omega^{\star}$ involves the maximum eigenvalue $\overline{\lambda}(\boldsymbol{R})$ and minimum eigenvalue $\underline{\lambda}(\boldsymbol{R})$ (Lemma~\ref{lem:jor_convergence_optimal}). We employ the power method (PM) to compute $\overline{\lambda}(\boldsymbol{R})$ and the inverse power method (IPM) to compute $\underline{\lambda}(\boldsymbol{R})$. The PM is a two step iterative algorithm that follows,
\begin{subequations}
\begin{alignat}{3}\label{eq:pm_step1}
    \boldsymbol{g}^{(s+1)} &  = \boldsymbol{R} \boldsymbol{e}^{(s)}\\ \label{eq:pm_step2}
    \boldsymbol{e}^{(s+1)} & = \frac{1}{\|\boldsymbol{g}^{(s+1)}\|_{\infty}} \boldsymbol{g}^{(s+1)},
\end{alignat}
\label{eq:pm}
\end{subequations}
where $\|\cdot\|_{\infty}$ denotes the infinity norm. As the PM algorithm converges $\|\boldsymbol{e}^{(s)} -\boldsymbol{e}^{(s-1)} \|_2 \rightarrow 0$, the infinity norm approximates the dominant eigenvalue $\|\boldsymbol{g}^{(s)}\|_{\infty} \approx \overline{\lambda}(\boldsymbol{R})$. After obtaining $\overline{\lambda}(\boldsymbol{R})$, we formulate the spectral shift of $\boldsymbol{R}$, that is $\boldsymbol{B} = \boldsymbol{R}-\overline{\lambda}(\boldsymbol{R})I_M$. The IPM is the application of PM \eqref{eq:pm} on $\boldsymbol{B}$. Then, we derive the minimum eigenvalue as $\underline{\lambda}(\boldsymbol{R}) = |\overline{\lambda}(\boldsymbol{B})-\overline{\lambda}(\boldsymbol{R})|$. In order to obtain both $\overline{\lambda}(\boldsymbol{R})$ and $\underline{\lambda}(\boldsymbol{R})$, we need to execute the PM algorithm \eqref{eq:pm} two sequential times. Let the first PM algorithm to converge after $s^{\textrm{end}}_{\textrm{PM}}$ iterations and the second after $s^{\textrm{end}}_{\textrm{IPM}}$ iterations. The use of the optimal relaxation is communication-wise more efficient if $s^{\textrm{end}}_{\textrm{PM}}+s^{\textrm{end}}_{\textrm{IPM}}+s^{\textrm{end}}_{\textrm{JOR}^{\star}}<s^{\textrm{end}}_{\textrm{JOR}}$ (Remark~\ref{rem:jor_jor_star}). Note that if $\boldsymbol{H}$ is symmetric, then $\boldsymbol{R} = \mathrm{diag}(\boldsymbol{H})^{-1}\boldsymbol{H}$ is also symmetric, as the only changes occur in the diagonal elements, with $\mathrm{diag}(\boldsymbol{H})^{-1} = \{\boldsymbol{H}_{ii}^{-1}\}_{i=1}^M$.

\begin{lemma}\label{lem:pm_convergencel}
\citep[Chapter 8]{golub2013matrix} Let the graph $\mathcal{G}$ be time-invariant and strongly complete. If $\boldsymbol{H}$ is symmetric, then the PM converges to the dominant real eigenvalue $\overline{\lambda}(\boldsymbol{R})$ with convergence rate $\mathcal{O}((\lambda_2/\overline{\lambda})^{s^{\textrm{end}}_{\textrm{PM}}})$, where $\lambda_2$ is the second largest eigenvalue.
\end{lemma}

\newfloat{algorithm}{!t}{lop}
\begin{algorithm}
\caption{DEC-PoE}\label{alg:dec_poe} 
{
\textbf{Input:} $\mathcal{D}_i(\boldsymbol{X}_i,\boldsymbol{y}_i)$, $\hat{\boldsymbol{\theta}}$, $\boldsymbol{C}_{\theta,i}^{-1}$,  $\mathcal{N}_i$,  $k$, $M$, $\boldsymbol{x}_*$, $\Delta$

\textbf{Output:} ${\mu}_{\textrm{DEC-PoE}}$, ${\sigma}_{\textrm{DEC-PoE}}^{-2}$
}
\begin{algorithmic}[1]

\State $\epsilon = 1/\Delta$

\ForEach {$i \in \mathcal V $}

\State $\mu_i \gets \texttt{localMean}(\boldsymbol{x}_*,k, \hat{\boldsymbol{\theta}},\mathcal{D}_i, \boldsymbol{C}_{\theta,i}^{-1})$ \eqref{eq:local_mean} 
\State $\sigma_i^{-2} \gets \texttt{localVariance}(\boldsymbol{x}_*,k, \hat{\boldsymbol{\theta}},\mathcal{D}_i, \boldsymbol{C}_{\theta,i}^{-1}, )$ \eqref{eq:local_variance} 

\State initialize $w_{\mu,i}^{(0)} =\beta_i\sigma_i^{-2}\mu_i$, $w_{\sigma^{-2},i}^{(0)} =\beta_i\sigma_i^{-2}$, $\beta_i=1$

\Repeat
\State communicate ${w}_{\mu,i}^{(s)}$, ${w}_{\sigma^{-2},i}^{(s)}$ to agents in $\mathcal{N}_i$
\State $w_{\mu,i}^{(s+1)} \gets \texttt{DAC}(\epsilon,w_{\mu,i}^{(s)},\{\boldsymbol{w}_{\mu,j}^{(s)}\}_{j\in \mathcal{N}_i},\mathcal{N}_i  )$ \eqref{eq:dac} \Comment{DAC1}
\State $w_{\sigma^{-2},i}^{(s+1)} \gets \texttt{DAC}(\epsilon,w_{\sigma^{-2},i}^{(s)},\{\boldsymbol{w}_{\sigma^{-2},j}^{(s)}\}_{j\in \mathcal{N}_i},\mathcal{N}_i )$ \eqref{eq:dac} \Comment{DAC2}
\Until {{maximin stopping criterion}}
\State ${\sigma}_{\textrm{DEC-PoE}}^{-2} =Mw_{\sigma^{-2},i}^{(\textrm{end})}$ \eqref{eq:variance_poe}
\State ${\mu}_{\textrm{DEC-PoE}}  = {\sigma}_{\textrm{DEC-PoE}}^{2}Mw_{\mu,i}^{(\textrm{end})}$ \eqref{eq:mean_poe}

\EndFor
\end{algorithmic}
\end{algorithm}

\newfloat{algorithm}{!t}{lop}
\begin{algorithm}
\caption{DEC-gPoE}\label{alg:dec_gpoe} 
{
\textbf{Input:} $\mathcal{D}_i(\boldsymbol{X}_i,\boldsymbol{y}_i)$, $\hat{\boldsymbol{\theta}}$, $\boldsymbol{C}_{\theta,i}^{-1}$,  $\mathcal{N}_i$,  $k$, $M$, $\boldsymbol{x}_*$, $\Delta$

\textbf{Output:} ${\mu}_{\textrm{DEC-gPoE}}$, ${\sigma}_{\textrm{DEC-gPoE}}^{-2}$
}
\begin{algorithmic}[1]

\State Identical to Algorithm~\ref{alg:dec_poe} with $\beta_i = 1/M$ instead of $\beta_i = 1$ (line 5)
\end{algorithmic}
\end{algorithm}
\subsubsection{DEC-PoE Family} The decentralized PoE (DEC-PoE) method makes use of two DAC algorithms (FIgure~\ref{fig:dec_gp}-(a)). The first DAC computes the average $(1/M)\sum_{i=1}^M\beta_i\sigma_i^{-2}$ and the second 
$(1/M)\sum_{i=1}^M\beta_i\sigma_i^{-2}\mu_i$, where $\beta_i=1$. At every iteration of DAC each agent communicates both computed values $w_{\mu,i}^{(s)} $, $w_{\sigma^{-2},i}^{(s)} $ to its neighbors $\mathcal{N}_i$. After convergence, each DAC average is multiplied by the number of nodes $M$ and follow \eqref{eq:mean_poe}, \eqref{eq:variance_poe} to recover the DEC-PoE prediction mean and precision. The implementation details are given in Algorithm~\ref{alg:dec_poe}. The time and space complexity are identical to the local time and space complexity of the PoE family as listed in Table~\ref{tab:complexityPrediction}. Let  $s^{\textrm{end}}_{\textrm{DAC}} $ be the maximum number of iterations of the two DAC to converge. The total communications are $\mathcal{O}(2s^{\textrm{end}}_{\textrm{DAC}}\mathrm{card}(\mathcal{N}_i))$ for all~$i\in \mathcal{V}$ (Table~\ref{tab:comms_dec_aggrations}). 

Next, we form the decentralized gPoE (DEC-gPoE) (Figure~\ref{fig:dec_gp}-(a)). The DEC-gPoE is identical to the DEC-PoE, but $\beta_i=1/M$ instead of $\beta_i=1$ (Algorithm~\ref{alg:dec_gpoe}). The time, space, and communication complexity are identical to the DEC-PoE. Both DEC-PoE and DEC-gPoE methods address Problem~\ref{pr:prediction1}.

\newfloat{algorithm}{!t}{lop}
\begin{algorithm}
\caption{DEC-BCM}\label{alg:dec_bcm} 
{
\textbf{Input:} $\mathcal{D}_i(\boldsymbol{X}_i,\boldsymbol{y}_i)$, $\hat{\boldsymbol{\theta}}$, $\boldsymbol{C}_{\theta,i}^{-1}$,  $\mathcal{N}_i$,  $k$, $M$, $\boldsymbol{x}_*$, $\Delta$

\textbf{Output:} ${\mu}_{\textrm{DEC-BCM}}$, ${\sigma}_{\textrm{DEC-BCM}}^{-2}$
}
\begin{algorithmic}[1]

\State $\epsilon = 1/\Delta$

\ForEach {$i \in \mathcal V $}

\State $\mu_i \gets \texttt{localMean}(\boldsymbol{x}_*,k, \hat{\boldsymbol{\theta}},\mathcal{D}_i, \boldsymbol{C}_{\theta,i}^{-1})$ \eqref{eq:local_mean} 
\State $\sigma_i^{-2} \gets \texttt{localVariance}(\boldsymbol{x}_*,k, \hat{\boldsymbol{\theta}},\mathcal{D}_i, \boldsymbol{C}_{\theta,i}^{-1}, )$ \eqref{eq:local_variance}
\State $\sigma_{**}^2 = k(\boldsymbol{x}_*, \boldsymbol{x}_*)$

\State initialize $w_{\mu,i}^{(0)} =\beta_i\sigma_i^{-2}\mu_i$, $w_{\sigma^{-2},i}^{(0)} =\beta_i\sigma_i^{-2}$, $\beta_i=1$

\Repeat
\State communicate ${w}_{\mu,i}^{(s)}$, ${w}_{\sigma^{-2},i}^{(s)}$ to agents in $\mathcal{N}_i$
\State $w_{\mu,i}^{(s+1)} \gets \texttt{DAC}(\epsilon,w_{\mu,i}^{(s)},\{\boldsymbol{w}_{\mu,j}^{(s)}\}_{j\in \mathcal{N}_i},\mathcal{N}_i  )$ \eqref{eq:dac} \Comment{DAC1}
\State $w_{\sigma^{-2},i}^{(s+1)} \gets \texttt{DAC}(\epsilon,w_{\sigma^{-2},i}^{(s)},\{\boldsymbol{w}_{\sigma^{-2},j}^{(s)}\}_{j\in \mathcal{N}_i},\mathcal{N}_i )$ \eqref{eq:dac} \Comment{DAC2}
\Until {{maximin stopping criterion}}
\State ${\sigma}_{\textrm{DEC-BCM}}^{-2} =Mw_{\sigma^{-2},i}^{(\textrm{end})} + (1-\sum_{i=1}^M\beta_i)\sigma_{**}^{-2}$ \eqref{eq:variance_bcm}
\State ${\mu}_{\textrm{DEC-BCM}}  = {\sigma}_{\textrm{DEC-BCM}}^{2}Mw_{\mu,i}^{(\textrm{end})}$ \eqref{eq:mean_bcm}

\EndFor
\end{algorithmic}
\end{algorithm}

\newfloat{algorithm}{!t}{lop}
\begin{algorithm}
\caption{DEC-rBCM}\label{alg:dec_rbcm} 
{
\textbf{Input:} $\mathcal{D}_i(\boldsymbol{X}_i,\boldsymbol{y}_i)$, $\hat{\boldsymbol{\theta}}$, $\boldsymbol{C}_{\theta,i}^{-1}$,  $\mathcal{N}_i$,  $k$, $M$, $\boldsymbol{x}_*$, $\Delta$

\textbf{Output:} ${\mu}_{\textrm{DEC-rBCM}}$, ${\sigma}_{\textrm{DEC-rBCM}}^{-2}$
}
\begin{algorithmic}[1]

\State $\epsilon = 1/\Delta$

\ForEach {$i \in \mathcal V $}

\State $\mu_i \gets \texttt{localMean}(\boldsymbol{x}_*,k, \hat{\boldsymbol{\theta}},\mathcal{D}_i, \boldsymbol{C}_{\theta,i}^{-1})$ \eqref{eq:local_mean} 
\State $\sigma_i^{-2} \gets \texttt{localVariance}(\boldsymbol{x}_*,k, \hat{\boldsymbol{\theta}},\mathcal{D}_i, \boldsymbol{C}_{\theta,i}^{-1} )$ \eqref{eq:local_variance}
\State $\sigma_{**}^2 = k(\boldsymbol{x}_*, \boldsymbol{x}_*)$

\State initialize $w_{\mu,i}^{(0)} =\beta_i\sigma_i^{-2}\mu_i$, $w_{\sigma^{-2},i}^{(0)} =\beta_i\sigma_i^{-2}$, $w_{\beta_i}^{(0)} =\beta_i$, \quad \quad $\beta_i = 0.5[\log \sigma_{**}^2 - \log \sigma_i^2]$

\Repeat
\State communicate ${w}_{\mu,i}^{(s)}$, ${w}_{\sigma^{-2},i}^{(s)}$, ${w}_{\beta_i}^{(s)}$ to agents in $\mathcal{N}_i$
\State $w_{\mu,i}^{(s+1)} \gets \texttt{DAC}(\epsilon,w_{\mu,i}^{(s)},\{\boldsymbol{w}_{\mu,j}^{(s)}\}_{j\in \mathcal{N}_i},\mathcal{N}_i  )$ \eqref{eq:dac} \Comment{DAC1}
\State $w_{\sigma^{-2},i}^{(s+1)} \gets \texttt{DAC}(\epsilon,w_{\sigma^{-2},i}^{(s)},\{\boldsymbol{w}_{\sigma^{-2},j}^{(s)}\}_{j\in \mathcal{N}_i},\mathcal{N}_i )$ \eqref{eq:dac} \Comment{DAC2}
\State $w_{\beta_i}^{(s+1)} \gets \texttt{DAC}(\epsilon,w_{\beta_i}^{(s)},\{\boldsymbol{w}_{\beta_j}^{(s)}\}_{j\in \mathcal{N}_i},\mathcal{N}_i )$ \eqref{eq:dac} \Comment{DAC3}
\Until {{maximin stopping criterion}}
\State ${\sigma}_{\textrm{DEC-rBCM}}^{-2} =Mw_{\sigma^{-2},i}^{(\textrm{end})} + (1-Mw_{\beta_i}^{(\textrm{end})})\sigma_{**}^{-2}$ \eqref{eq:variance_bcm}
\State ${\mu}_{\textrm{DEC-rBCM}}  = {\sigma}_{\textrm{DEC-rBCM}}^{2}Mw_{\mu,i}^{(\textrm{end})}$ \eqref{eq:mean_bcm}

\EndFor
\end{algorithmic}
\end{algorithm}
\subsubsection{DEC-BCM Family} The decentralized BCM (DEC-BCM) method employs two DAC algorithms (Figure~\ref{fig:dec_gp}-(a)). The first DAC computes the average $(1/M)\sum_{i=1}^M\beta_i\sigma_i^{-2}$ and the second $(1/M)\sum_{i=1}^M\beta_i\sigma_i^{-2}\mu_i$, where $\beta_i=1$. At every iteration of DAC each agent communicates both computed values $w_{\mu,i}^{(s)} $, $w_{\sigma^{-2},i}^{(s)} $ to its neighbors $\mathcal{N}_i$. After convergence, each DAC average is multiplied by the number of nodes $M$ and follow \eqref{eq:mean_bcm}, \eqref{eq:variance_bcm} to recover the DEC-BCM mean and precision. The implementation details are provided in Algorithm~\ref{alg:dec_bcm}. The time, space, and communication complexity are identical to the DEC-PoE family. The DEC-BCM addresses Problem~\ref{pr:prediction1}.

We introduce the decentralized rBCM (DEC-rBCM) technique that employs three DAC algorithms to compute the averages $(1/M)\sum_{i=1}^M\beta_i\sigma_i^{-2}$, $(1/M)\sum_{i=1}^M\beta_i\sigma_i^{-2}\mu_i$, and $(1/M)\sum_{i=1}^M\beta_i$, where $\beta_i = 0.5[\log \sigma_{**}^2 - \log \sigma_i^2]$. At every iteration of DAC each agent communicates ${w}_{\mu,i}^{(s)}$, ${w}_{\sigma^{-2},i}^{(s)}$, ${w}_{\beta_i}^{(s)}$ to its neighbors $\mathcal{N}_i$. After convergence, each DAC average is multiplied by the number of nodes $M$ and follow \eqref{eq:mean_bcm}, \eqref{eq:variance_bcm} to recover the DEC-rBCM prediction mean and precision. Implementation details are given in Algorithm~\ref{alg:dec_rbcm}. The local time and space complexity are identical to the rBCM (Table~\ref{tab:complexityPrediction}). Let  $s^{\textrm{end}}_{\textrm{DAC}} $ be the maximum number of iterations of the three DAC to converge. The total communications are $\mathcal{O}(3s^{\textrm{end}}_{\textrm{DAC}}\mathrm{card}(\mathcal{N}_i))$ for all $i\in \mathcal{V}$ (Table~\ref{tab:comms_dec_aggrations}). The DEC-rBCM addresses Problem~\ref{pr:prediction1}.

\newfloat{algorithm}{!t}{lop}
\begin{algorithm}
\caption{DEC-grBCM}\label{alg:dec_grbcm} 
{
\textbf{Input:} $\mathcal{D}_{+i}(\boldsymbol{X}_{+i},\boldsymbol{y}_{+i})$, $\hat{\boldsymbol{\theta}}$, $\boldsymbol{C}_{\theta,+i}^{-1}$,  $\mathcal{N}_i$,  $k$, $M$, $\boldsymbol{x}_*$, $\Delta$

\textbf{Output:} ${\mu}_{\textrm{DEC-grBCM}}$, ${\sigma}_{\textrm{DEC-grBCM}}^{-2}$
}
\begin{algorithmic}[1]

\State $\epsilon = 1/\Delta$

\ForEach {$i \in \mathcal V $}

\State $\mu_{+i} \gets \texttt{localMean}(\boldsymbol{x}_*,k, \hat{\boldsymbol{\theta}},\mathcal{D}_{+i}, \boldsymbol{C}_{\theta,+i}^{-1})$ \eqref{eq:local_mean} 
\State $\sigma_{+i}^{-2} \gets \texttt{localVariance}(\boldsymbol{x}_*,k, \hat{\boldsymbol{\theta}},\mathcal{D}_{+i}, \boldsymbol{C}_{\theta,+i}^{-1} )$ \eqref{eq:local_variance}
\State $\sigma_{\textrm{c}}^2 = k(\boldsymbol{X}_{\textrm{c}}, \boldsymbol{X}_{\textrm{c}})$

\State initialize $w_{\mu,i}^{(0)} =\beta_i\sigma_{+i}^{-2}\mu_{+i}$, $w_{\sigma^{-2},i}^{(0)} =\beta_i\sigma_{+i}^{-2}$, $w_{\beta_i}^{(0)} =\beta_i$, \quad \quad $\beta_i = 0.5[\log \sigma_{\textrm{c}}^2 - \log \sigma_{+i}^2]$

\Repeat
\State communicate ${w}_{\mu,i}^{(s)}$, ${w}_{\sigma^{-2},i}^{(s)}$, ${w}_{\beta_i}^{(s)}$ to agents in $\mathcal{N}_i$
\State $w_{\mu,i}^{(s+1)} \gets \texttt{DAC}(\epsilon,w_{\mu,i}^{(s)},\{\boldsymbol{w}_{\mu,j}^{(s)}\}_{j\in \mathcal{N}_i},\mathcal{N}_i  )$ \eqref{eq:dac} \Comment{DAC1}
\State $w_{\sigma^{-2},i}^{(s+1)} \gets \texttt{DAC}(\epsilon,w_{\sigma^{-2},i}^{(s)},\{\boldsymbol{w}_{\sigma^{-2},j}^{(s)}\}_{j\in \mathcal{N}_i},\mathcal{N}_i )$ \eqref{eq:dac} \Comment{DAC2}
\State $w_{\beta_i}^{(s+1)} \gets \texttt{DAC}(\epsilon,w_{\beta_i}^{(s)},\{\boldsymbol{w}_{\beta_j}^{(s)}\}_{j\in \mathcal{N}_i},\mathcal{N}_i )$ \eqref{eq:dac} \Comment{DAC3}
\Until {{maximin stopping criterion}}
\State ${\sigma}_{\textrm{DEC-grBCM}}^{-2} =Mw_{\sigma^{-2},i}^{(\textrm{end})} + (1-Mw_{\beta_i}^{(\textrm{end})})\sigma_{\textrm{c}}^{-2}$ \eqref{eq:variance_grbcm}
\State ${\mu}_{\textrm{DEC-grBCM}}  = {\sigma}_{\textrm{DEC-grBCM}}^{2}(Mw_{\mu,i}^{(\textrm{end})} - (Mw_{\beta_i}^{(\textrm{end})} -1)\sigma_{\textrm{c}}^{-2}\mu_{\textrm{c}})$ \eqref{eq:mean_grbcm}

\EndFor
\end{algorithmic}
\end{algorithm}

We propose the decentralized grBCM (DEC-grBCM) method which employs three DAC algorithms (Figure~\ref{fig:dec_gp}-(b)) to compute the averages $(1/M)\sum_{i=1}^M\beta_i\sigma_{+i}^{-2}$, $(1/M)\sum_{i=1}^M\beta_i\sigma_{+i}^{-2}\mu_i$, and $(1/M)\sum_{i=1}^M\beta_i$, where $\beta_i = 0.5[\log \sigma_{\textrm{c}}^2 - \log \sigma_{+i}^2]$. At every iteration of DAC each agent communicates ${w}_{\mu,i}^{(s)}$, ${w}_{\sigma^{-2},i}^{(s)}$, ${w}_{\beta_i}^{(s)}$ to its neighbors $\mathcal{N}_i$. After convergence, each DAC average is multiplied by the number of nodes $M$ and follow \eqref{eq:mean_grbcm}, \eqref{eq:variance_grbcm} to recover the prediction mean and precision. Implementation details are given in Algorithm~\ref{alg:dec_grbcm}. The local time and space complexity are identical to the grBCM (Table~\ref{tab:complexityPrediction}). Let  $s^{\textrm{end}}_{\textrm{DAC}} $ be the maximum number of iterations of the three DAC to converge. The total communications are $\mathcal{O}(3s^{\textrm{end}}_{\textrm{DAC}}\mathrm{card}(\mathcal{N}_i))$ for all $i\in \mathcal{V}$ (Table~\ref{tab:comms_dec_aggrations}). The DEC-grBCM addresses Problem~\ref{pr:prediction2}.

\begin{proposition}
Let the Assumption \ref{ass:repeat_connectivity},~\ref{ass:parData},~\ref{ass:independence},~\ref{ass:cond_independence},~\ref{ass:fleet} hold throughout the approximation. If $\omega<2/M$ then the DEC-grBCM is consistent for any initialization. 

\proof
The proof is a direct consequence of \Cref{prop:grbcmCons} and \Cref{lem:dac_convergence}.
\end{proposition}
\subsubsection{DEC-NPAE Family}
An additional assumption is required to implement the DEC-NPAE family methods.

\begin{assumption}\label{ass:strongly_complete}
The graph topology is strongly complete, i.e. every agent $i$ can communicate with every other node $j \neq i$. 
\end{assumption}

Assumption~\ref{ass:strongly_complete} is conservative, but mandatory for the implementation of the PM and JOR algorithms. In order to use the DEC-NPAE family with strongly connected graph topologies, flooding is required (Remark~\ref{rem:jor_neighbors}).

\newfloat{algorithm}{!t}{lop}
\begin{algorithm}
\caption{DEC-NPAE}\label{alg:dec_npae} 
{
\textbf{Input:} $\mathcal{D}_i(\boldsymbol{X}_i,\boldsymbol{y}_i)$, $\boldsymbol{X}$, $\hat{\boldsymbol{\theta}}$, $\boldsymbol{C}_{\theta,i}^{-1}$,  $\mathcal{N}_i$,  $k$, $M$, $\boldsymbol{x}_*$, $\Delta$

\textbf{Output:} ${\mu}_{\textrm{DEC-NPAE}}$, ${\sigma}_{\textrm{DEC-NPAE}}^2$
}
\begin{algorithmic}[1]

\State initialize $\omega = 2/M$; $\epsilon = 1/\Delta$

\ForEach {$i \in \mathcal V $}

\State communicate $\boldsymbol{C}_{\theta,i}^{-1}$, $\boldsymbol{X}_{i}$ to agents in $\mathcal{V} \backslash i$ \label{alg:dec_npae_star:comm_neighbor}

\State $\mu_i \gets \texttt{localMean}(\boldsymbol{x}_*,k, \hat{\boldsymbol{\theta}},\mathcal{D}_i, \boldsymbol{C}_{\theta,i}^{-1})$ \eqref{eq:local_mean} 
\State $[\boldsymbol{k}_A ]_i \gets \texttt{crossCov}(\boldsymbol{x}_*,k,  \hat{\boldsymbol{\theta}},\boldsymbol{X}_i, \boldsymbol{C}_{\theta,i}^{-1})$ \eqref{eq:local_covariance}

\State $\mathrm{row}_i\{\boldsymbol{C}_{\theta,A}\} \gets \texttt{localCov}(\boldsymbol{x}_*,k,  \hat{\boldsymbol{\theta}}, \boldsymbol{X}, \boldsymbol{C}_{\theta,i}^{-1}, \{ \boldsymbol{C}_{\theta,j}^{-1}\}_{j\neq i})$~\eqref{eq:cross_covariance}

\State $[\boldsymbol{H}]_i = \mathrm{row}_i\{\boldsymbol{C}_{\theta,A}\}$; $b_{\mu,i} = \mu_i$;  $b_{\sigma^2,i} = [\boldsymbol{k}_A ]_i$
\State initialize $q_{\mu,i}^{(0)} = b_{\mu,i}/[\boldsymbol{H}]_{ii}$, $q_{\sigma^2,i}^{(0)} = b_{\sigma^2,i}/[\boldsymbol{H}]_{ii}$ 

\Repeat \Comment{2$\times$JOR}

\State communicate ${q}_{\mu,i}^{(s)}$, ${q}_{\sigma^2,i}^{(s)}$ to agents in $\mathcal{V} \backslash i$ \label{alg:dec_npae_comm_neighbor}

\State $q_{\mu,i}^{(s+1)} \gets \texttt{JOR}(\omega, [\boldsymbol{H}]_i,b_{\mu,i},q_{\mu,i}^{(s)},\{\boldsymbol{q}_{\mu,j}^{(s)}\}_{j\neq i}  )$ \eqref{eq:jor}  \label{alg:dec_npae_jor1}
\State $q_{\sigma^2,i}^{(s+1)} \gets \texttt{JOR}(\omega, [\boldsymbol{H}]_i,b_{\sigma^2,i},q_{\sigma^2,i}^{(s)},\{\boldsymbol{q}_{\sigma^2,j}^{(s)}\}_{j\neq i} )$ \eqref{eq:jor} \label{alg:dec_npae_jor2}

\Until {{maximin stopping criterion}}

\State initialize $w_{\mu,i}^{(0)} = [\boldsymbol{k}_A ]_i q_{\mu,i}^{(\textrm{end})}$, $w_{\sigma^2,i}^{(0)} = [\boldsymbol{k}_A ]_i q_{\sigma^2,i}^{(\textrm{end})}$

\Repeat
\State communicate ${w}_{\mu,i}^{(s)}$, ${w}_{\sigma^2,i}^{(s)}$ to agents in $\mathcal{N}_i$
\State $w_{\mu,i}^{(s+1)} \gets \texttt{DAC}(\epsilon,w_{\mu,i}^{(s)},\{\boldsymbol{w}_{\mu,j}^{(s)}\}_{j\in \mathcal{N}_i},\mathcal{N}_i  )$ \eqref{eq:dac} \Comment{DAC1}
\State $w_{\sigma^2,i}^{(s+1)} \gets \texttt{DAC}(\epsilon,w_{\sigma^2,i}^{(s)},\{\boldsymbol{w}_{\sigma^2,j}^{(s)}\}_{j\in \mathcal{N}_i},\mathcal{N}_i )$ \eqref{eq:dac}\Comment{DAC2}
\Until {{maximin stopping criterion}}
\State ${\mu}_{\textrm{DEC-NPAE}}  = Mw_{\mu,i}^{(\textrm{end})}$ 
\State ${\sigma}_{\textrm{DEC-NPAE}}^2 ={\sigma_f^2}(k_{**}- Mw_{\sigma^2,i}^{(\textrm{end})})$ 
\EndFor
\end{algorithmic}
\end{algorithm}

We present \textsc{dec-NPAE} which combines JOR and DAC to decentralize the computations \eqref{eq:mean_npae}, \eqref{eq:variance_npae} of~\textsc{NPAE} (Figure~\ref{fig:dec_npae}-(a)). We execute two parallel JOR algorithms with known matrix $\boldsymbol{H} = \boldsymbol{C}_{\theta,A}$ and known vectors: i) $\boldsymbol{b} = \boldsymbol{\mu}$; and ii) $\boldsymbol{b} = \boldsymbol{k}_{A}$. The first JOR is associated with the prediction mean \eqref{eq:mean_npae} and the second with the variance \eqref{eq:variance_npae}. Note that $\boldsymbol{C}_{\theta,A}$ is a symmetric and PD covariance matrix. Implementation details are provided in Algorithm~\ref{alg:dec_npae}. We split up the computation in two parts. First, each entity computes three quantities: i) the local mean $\mu_i$ \eqref{eq:local_mean}; ii) the local cross covariance $[\boldsymbol{k}_A]_i$ \eqref{eq:local_covariance}; and iii) the local row covariance $\mathrm{row}_i\{\boldsymbol{C}_{\theta,A}\}$ \eqref{eq:cross_covariance}. For the local computation of \eqref{eq:cross_covariance} the agents must know the inputs $\{\boldsymbol{X}_j\}_{j \neq i}$ of all other agents, to find $\boldsymbol{C}_{\theta,ij}$ and $\boldsymbol{k}_{j,*}$. The inputs $\{\boldsymbol{X}_j\}_{j \neq i}$ are communicated between agents. The local inverted covariance matrices of all other agents $\{\boldsymbol{C}_{\theta,j}^{-1}\}_{j\neq i}$ can be locally computed, but it is computationally very expensive to invert $M-1$ matrices, i.e. $\mathcal{O}(MN_i^3)= \mathcal{O}(N^3/M^2)$. Since every agent $i$ has already stored its local covariance matrix from the training step (Section~\ref{sec:decentTrain}), we select to exchange $\{\boldsymbol{C}_{\theta,j}^{-1}\}_{j\neq i}$ between agents (Algorithm~\ref{alg:dec_npae}-[Line~\ref{alg:dec_npae_star:comm_neighbor}]). After every JOR iteration, each agent $i$ communicates the computed values $q_{\mu,i}^{(s)} $, $q_{\sigma^2,i}^{(s)} $ to its neighbors $\mathcal{N}_i$ (Algorithm~\ref{alg:dec_npae}-[line \ref{alg:dec_npae_comm_neighbor}]). Next, we compute an element of the unknown vectors $q_{\mu,i} = [\boldsymbol{C}_{\theta,A}^{-1}\boldsymbol{\mu}]_i $, $q_{\sigma^2,i} = [\boldsymbol{C}_{\theta,A}^{-1}\boldsymbol{k}_{A}]_i $ (Algorithm~\ref{alg:dec_npae}-[lines \ref{alg:dec_npae_jor1}, \ref{alg:dec_npae_jor2}]) with the JOR method. When JOR converges, 
each agent computes locally the $i$-th element of the resulting summation from: i) the multiplication between the vectors $\boldsymbol{k}_{A}^{\intercal}$ and $\boldsymbol{C}_{\theta,A}^{-1}\boldsymbol{\mu}$ \eqref{eq:mean_npae}, that is $w_{\mu,i} = [\boldsymbol{k}_{A}]_iq_{\mu,i}^{(\textrm{end})}$; and ii) the multiplication between the vectors $\boldsymbol{k}_{A}^{\intercal}$ and $\boldsymbol{C}_{\theta,A}^{-1}\boldsymbol{k}_A$ \eqref{eq:variance_npae}, that is $w_{\sigma^2,i} = [\boldsymbol{k}_{A}]_i q_{\sigma^2,i}^{(\textrm{end})}$. Second, since all agents have stored a part of the summations $w_{\mu,i}$, $w_{\sigma^2,i}$, we use the DAC to compute the averages $(1/M)\sum_{i=1}^M [\boldsymbol{k}_{A}]_iq_{\mu,i}^{(\textrm{end})}$ and $(1/M)\sum_{i=1}^M [\boldsymbol{k}_{A}]_iq_{\sigma^2,i}^{(\textrm{end})}$. After every DAC iteration, each agent $i$ communicates the computed values $w_{\mu,i}^{(s)} $, $w_{\sigma^2,i}^{(s)} $ to its neighbors $\mathcal{N}_i$. When both DAC converge, each agent follows \eqref{eq:mean_npae}, \eqref{eq:variance_npae} to recover the DEC-NPAE mean and variance. The local time and space complexity are identical to the local NPAE as shown in Table~\ref{tab:complexityPrediction}. Let  $s^{\textrm{end}}_{\textrm{JOR}} $ and $s^{\textrm{end}}_{\textrm{DAC}} $ be the maximum number of iterations of the JOR and DAC to converge respectively. The total communications for a strongly complete topology yields $\mathcal{O}(2s^{\textrm{end}}_{\textrm{JOR}}M+2s^{\textrm{end}}_{\textrm{DAC}}\mathrm{card}(\mathcal{N}_i)+MN_i^2+MDN_i) = \mathcal{O}(2s^{\textrm{end}}_{\textrm{JOR}}M+2s^{\textrm{end}}_{\textrm{DAC}}\mathrm{card}(\mathcal{N}_i)+M(N^2/M^2+DN/M)$ for all $i\in \mathcal{V}$ as listed in Table~\ref{tab:comms_dec_aggrations}. 

\newfloat{algorithm}{!t}{lop}
\begin{algorithm}
\caption{\textsc{PowerMethod}}\label{alg:PM} 
{
\textbf{Input:} $\boldsymbol{R}$, $\mathcal{N}_i$, $M$, $\eta_{\textrm{PM}}$

\textbf{Output:} $\overline{\lambda}(\boldsymbol{R})$
}
\begin{algorithmic}[1]

\State initialize $\boldsymbol{e}^{(0)} = 1/M$

\Repeat

\State ${g}_i^{(s+1)} =  \mathrm{row}_i\{\boldsymbol{R}\} \boldsymbol{e}^{(s)}$ \eqref{eq:pm_step1} 

\State communicate $g_i^{(s+1)}$ to agents in $ \mathcal{V} \backslash i$ \label{alg:pm:communications}

\State $\|\boldsymbol{g}^{(s+1)}\|_{\infty} = \max\{ | \boldsymbol{g}^{s+1} | \}$

\State $\boldsymbol{e}^{(s+1)}  = \boldsymbol{g}^{(s+1)}/\|\boldsymbol{g}^{(s+1)}\|_{\infty}$ \eqref{eq:pm_step2}

\Until {$\| \boldsymbol{e}^{(s+1)} - \boldsymbol{e}^{(s)}\|_2< \eta_{\textrm{PM}} $}

\State $\overline{\lambda}(\boldsymbol{R}) = \|\boldsymbol{g}^{(\textrm{end})}\|_{\infty}$

\end{algorithmic}
\end{algorithm}

\newfloat{algorithm}{!t}{lop}
\begin{algorithm}
\caption{DEC-NPAE*}\label{alg:dec_npae_star} 
{
\textbf{Input:} $\mathcal{D}_i(\boldsymbol{X}_i,\boldsymbol{y}_i)$, $\boldsymbol{X}$, $\hat{\boldsymbol{\theta}}$, $\boldsymbol{C}_{\theta,i}^{-1}$,  $\mathcal{N}_i$,  $k$, $M$, $\boldsymbol{x}_*$, $\Delta$, $\eta_{\textrm{PM}}$

\textbf{Output:} ${\mu}_{\textrm{DEC-NPAE}^{\star}}$, ${\sigma}_{\textrm{DEC-NPAE}^{\star}}^2$
}
\begin{algorithmic}[1]
\ForEach {$i \in \mathcal V $}

\State communicate $\mathrm{row}_i\{\boldsymbol{C}_{\theta,A}\}$ to agents in $ \mathcal{V} \backslash i$ \label{alg:dec_npae_star:communications}

\State $\mathrm{diag}(\boldsymbol{C}_{\theta,A})^{-1} = \mathrm{diag}(\{\boldsymbol{C}_{\theta,A}\}_{ii}^{-1})$
\State $\boldsymbol{R} = \mathrm{diag}(\boldsymbol{C}_{\theta,A})^{-1} \boldsymbol{C}_{\theta,A}$

\State $\overline{\lambda}(\boldsymbol{R}) \gets \texttt{PowerMethod} (\boldsymbol{R}, \mathcal{N}_i, M, \eta_{\textrm{PM}})$ \Comment{PM1}

\State $\boldsymbol{B} = \boldsymbol{R} - \overline{\lambda}(\boldsymbol{R})I_M$ \label{alg:dec_npae_star:spectralShift}

\State $\overline{\lambda}(\boldsymbol{B}) \gets \texttt{PowerMethod} (\boldsymbol{B}, \mathcal{N}_i, M, \eta_{\textrm{PM}})$ \Comment{PM2}

\State $\underline{\lambda}(\boldsymbol{R}) = |\overline{\lambda}(\boldsymbol{B})-\overline{\lambda}(\boldsymbol{R})|$ \label{alg:dec_npae_star:min_eig}

\State $\omega^{\star} = 2/(\overline{\lambda}(\boldsymbol{R})+\underline{\lambda}(\boldsymbol{R}))$ 
\EndFor
\State $\texttt{DEC-NPAE}(\mathcal{D}_i, \boldsymbol{X}, \hat{\boldsymbol{\theta}},\boldsymbol{C}_{\theta,i}^{-1},  \mathcal{N}_i,  k, M, \boldsymbol{x}_*, \Delta , \omega^{\star})$
\end{algorithmic}
\end{algorithm}

The decentralized NPAE$^{\star}$ (DEC-NPAE$^{\star}$) method (Figure~\ref{fig:dec_npae}-(b)) is similar to the DEC-NPAE, but includes an additional routine (Algorithm~\ref{alg:PM}) to compute the optimal relaxation factor $\omega^{\star}$ (Lemma~\ref{lem:jor_convergence_optimal}). More specifically, we employ the PM iterative scheme \eqref{eq:pm} to estimate the largest $\overline{\lambda}$ and smallest $\underline{\lambda}$ eigenvalues of $\boldsymbol{R}$. The workflow is as follows. To compute the matrix of interest $\boldsymbol{R} = \mathrm{diag}(\boldsymbol{C}_{\theta,A})^{-1} \boldsymbol{C}_{\theta,A}$, each agent $i$ constructs $\boldsymbol{C}_{\theta,A}$ after exchanging $\{\mathrm{row}_j\{\boldsymbol{C}_{\theta,A}\}\}_{j\neq i}$ (Algorithm~\ref{alg:dec_npae_star}-[Line \ref{alg:dec_npae_star:communications}]). Next, each agent $i$ executes the PM (Algorithm~\ref{alg:PM}) to obtain the maximium eigenvalue $\overline{\lambda}(\boldsymbol{R})$. Then, the spectral shift matrix $\boldsymbol{B}$ is composed (Algorithm~\ref{alg:dec_npae_star}-[line \ref{alg:dec_npae_star:spectralShift}]). Using $\boldsymbol{B}$ as an input to the PM algorithm, its maximum eigenvalue is obtained $\overline{\lambda}(\boldsymbol{B})$. To this end, the minimum eigenvalue of $\boldsymbol{R}$ can be computed (Algorithm~\ref{alg:PM}-[Line~\ref{alg:dec_npae_star:min_eig}]). Subsequently, the optimal relaxation $\omega^{\star}$ is computed according to Lemma~\ref{lem:jor_convergence_optimal}. Provided $\omega^{\star}$, the DEC-NPAE (Algorithm~\ref{alg:dec_npae}) is executed. Let $s^{\textrm{end}}_{\textrm{PM}}$ be the iterations required for the PM to converge. Then, the total communications are $\mathcal{O}(2s^{\textrm{end}}_{\textrm{PM}}M +M^2)+\mathcal{O}(\textrm{DEC-NPAE})$ to exchange: i) the $g_i^{(s)}$ for two PM routines (Algorithm~\ref{alg:PM}-[Line~\ref{alg:pm:communications}]); ii) the $\mathrm{row}_i\{\boldsymbol{C}_{\theta,A}\}$ (Algorithm~\ref{alg:dec_npae_star}-[Line~\ref{alg:dec_npae_star:communications}]); and iii) the quantities of DEC-NPAE. A comparison of the communication complexity for all decentralized GP aggregation methods is presented in Table~\ref{tab:comms_dec_aggrations}. In Figure~\ref{fig:dec_npae} we illustrate the structure of the DEC-NPAE family. Both methods of the DEC-NPAE family address Problem~\ref{pr:prediction2}. 

\begin{figure}[!t]
	\includegraphics[width=.75\columnwidth]{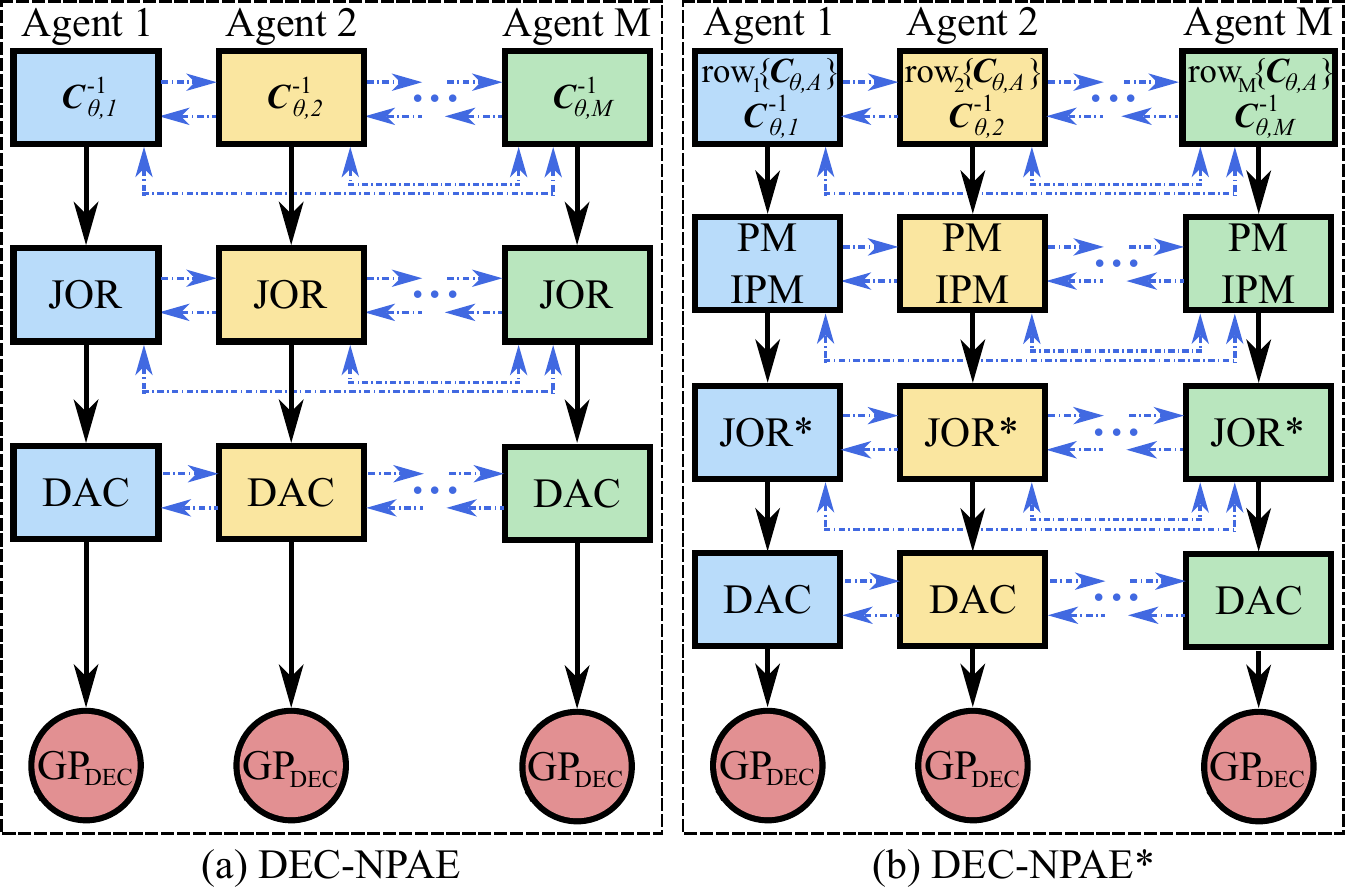}
	\centering
	\caption{The structure of the DEC-NPAE family. Blue dotted lines correspond to communication (strongly connected and strongly complete). (a) DEC-NPAE incorporates Jacobi over-relaxation (JOR) and discrete-time average consensus (DAC). (b) DEC-NPAE$^{\star}$ makes use of the power method (PM) to obtain the optimal relaxation factor and execute JOR$^{\star}$, and DAC.
	}
	\label{fig:dec_npae}
\end{figure}

\begin{proposition}\label{prop:dec_npae_family}
Let the graph $\mathcal{G}$ be strongly complete during the JOR and PM iterations (Assumption~\ref{ass:strongly_complete}), and strongly connected during the DAC iterations (Assumption \ref{ass:repeat_connectivity}). In addition, let Assumption~\ref{ass:parData},~\ref{ass:independence},~\ref{ass:fleet} hold throughout the approximation. If $\omega<2/M$, $\epsilon\in(0,1/\Delta)$, then the DEC-NPAE is consistent for any initialization. Provided that the conditions for JOR hold for the PM iterations and that $\omega^{\star} = 2/(\overline{\lambda}(\boldsymbol{R})+\underline{\lambda}(\boldsymbol{R}))$, then the DEC-NPAE$^{\star}$ is consistent for any initial conditions. 

\proof
The proof for DEC-NPAE is a direct consequence of \Cref{prop:consistency} and \Cref{lem:dac_convergence}, \ref{lem:jor_convergence}. Similarly for DEC-NPAE$^{\star}$, the proof follows from \Cref{prop:consistency} and \Cref{lem:dac_convergence}, \ref{lem:jor_convergence_optimal}.
\end{proposition}
\subsection{Nearest Neighbor Decentralized Aggregation Methods}
\textbf{\textit{DALE}}: An alternative method to solve a linear system of algebraic equations, but for strongly connected (Assumption~\ref{ass:repeat_connectivity}) and not strongly complete topology (Assumption~\ref{ass:strongly_complete}) is DALE. The latter is an iterative method with identical setup to JOR $\boldsymbol{H} \boldsymbol{q} = \boldsymbol{b}$, where  $\boldsymbol{H}$ is a known matrix, $\boldsymbol{b}$ a known vector, and $\boldsymbol{q}$ an unknown vector. The $i$-th node knows: i) $i$-th row of $\boldsymbol{H}_i = \mathrm{row}_i\{\boldsymbol{H}\}\in \mathbb{R}^{1\times M}$; and ii) $i$-th entry of $b_i\in \mathbb{R}$. In addition, DALE is formulated as a consensus problem, where the goal for all agents is to obtain the same solution $\boldsymbol{q}_i\in \mathbb{R}^M$ and not just an element of the unknown vector as in JOR. The DALE~follows,
\begin{equation}\label{eq:dale}
     \boldsymbol{q}_i^{(s+1)} = \boldsymbol{H}_i^{\intercal}(\boldsymbol{H}_i\boldsymbol{H}_i^{\intercal})^{-1}{b}_i+\frac{1}{\mathrm{card}( \mathcal{N}_i(t))} \boldsymbol{P}_i \sum_{j\in \mathcal{N}_i(t)}\boldsymbol{q}_j^{(s)} ,
\end{equation}
where $\boldsymbol{P}_i = I_M - \boldsymbol{H}_i^{\intercal}(\boldsymbol{H}_i\boldsymbol{H}_i^{\intercal})^{-1}\boldsymbol{H}_i \in \mathbb{R}^{M\times M}$ is the orthogonal projection onto the kernel of $\boldsymbol{H}_i$. In addition, DALE can be executed in a time-varying network under Assumption \ref{ass:repeat_connectivity}. 

\begin{assumption}\label{ass:row_rank}
Matrix $\boldsymbol{H}$ is full row rank.
\end{assumption}

\begin{lemma}\label{lem:dale_convergence}
\citep[Theorem 3]{liu2017asynchronous} Let Assumption \ref{ass:repeat_connectivity},~\ref{ass:row_rank} hold. There exists a constant $\phi \in (0,1)$ such that   all $\boldsymbol{q}_i^{(s)}$ converge to the solution for any initialization $\boldsymbol{q}_i^{(0)}$ with worst case convergence speed~$\phi^s$. 
\end{lemma}

\begin{remark}\label{rem:dale_convergence}
The convergence speed constant 
$\phi$ depends on the number of agents $M$ and the diameter of the graph $\mathrm{diam}(\mathcal{G})$. The larger the fleet size and the diameter the slower the convergence.
\end{remark}

\begin{remark}\label{rem:dale_nearest_neighbor}
The $i$-th node using DALE \eqref{eq:dale} exchanges information only with its neighbors $j \in \mathcal{N}_i$ and not with the whole network (see in contrast \Cref{rem:jor_neighbors} for JOR). In addition, DALE is concurrently a consensus algorithm and updates the whole vector $\boldsymbol{q}_i^{(s)} \in \mathbb{R}^{M}$, while JOR updates just the corresponding entry $[\boldsymbol{q}_i^{(s)}]_i\in \mathbb{R}$. Thus, DALE is equivalent to the operation of both JOR and~DAC. 
\end{remark}

\begin{figure}[!t]
	\includegraphics[width=.45\columnwidth]{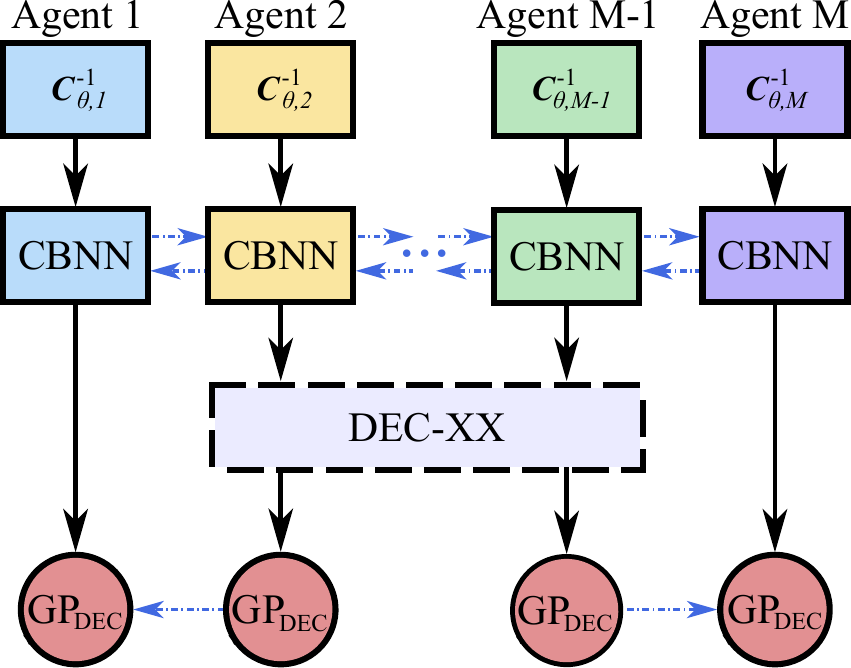}
	\centering
	\caption{The structure of the proposed nearest neighbor decentralized aggregation methods. Blue dotted lines correspond to communication (strongly connected). The covariance-based nearest neighbor (CBNN) method identifies statistically correlated agents---in this illustration the CBNN set is $\mathcal{V}_{\textrm{NN}}\in[2,M-1]$. Next, a decentralized aggregation method among the DEC-PoE and DEC-BCM families is executed within the $\mathcal{V}_{\textrm{NN}}$ nodes. After convergence, the predicted values are communicated to the rest agents of the network.
	}
	\label{fig:dec_nn_gp_1}
\end{figure}

\textbf{\textit{CBNN}}: To identify statistically correlated agents for a location of interest $\boldsymbol{x}_*$ we introduce the covariance-based nearest neighbor (CBNN) method. Let every agent $i$ to have its own opinion for the location of interest $\{ \mu_1, \hdots, \mu_M \}$, where $\mu_i = \mathrm{E}[y(\boldsymbol{x}_*) \mid \mathcal{D}_i, \boldsymbol{\theta}]$ computed as a GP local mean \eqref{eq:local_mean}. In other words, every agent makes a prediction $\mu_i$ for the location of interest $\boldsymbol{x}_*$ based on its local dataset $\mathcal{D}_i$. Then, we use the local mean values to form the \textit{mean dataset} $\mathcal{D}_{\mu} = (\{ \boldsymbol{X}_i\}_{i=1}^M,\{  \mu_i\}_{i=1}^M) = (\boldsymbol{X},\boldsymbol{\mu})$, where $\boldsymbol{X}_i \in \mathbb{R}^{D\times N_i}$, $\boldsymbol{X} \in \mathbb{R}^{D\times N}$, $\mu_i \in \mathbb{R}$, and $\boldsymbol{\mu} \in \mathbb{R}^M$. 

\begin{definition}\label{def:mean_random_process}
Let the vector of random variables $(\mu_1(\boldsymbol{x}_*), \hdots, \mu_M(\boldsymbol{x}_*), y(\boldsymbol{x}_*))^ \intercal \in \mathbb{R}^{M+1}$ to form a random process, where the first two moments exist with zero mean $\mu_{\mu}=0$ and a finite covariance $\boldsymbol{C}_{\theta, \mu}$.
\end{definition}

\begin{proposition}\label{prop:gp_mean_random_process}
\citep[Proposition 3]{bachoc2017some} The random process (Definition~\ref{def:mean_random_process}) approximates a GP, $(\mu_1(\boldsymbol{x}_*), \hdots, \mu_M(\boldsymbol{x}_*), y(\boldsymbol{x}_*))^{\intercal} \sim \mathcal{GP}(\mu_{\mu}, \boldsymbol{C}_{\theta, \mu})$ as $N\rightarrow \infty$.
\end{proposition}

\newfloat{algorithm}{!t}{lop}
\begin{algorithm}
\caption{DEC-NN-PoE}\label{alg:dec_nn_poe} 
{
\textbf{Input:} $\mathcal{D}_i(\boldsymbol{X}_i,\boldsymbol{y}_i)$, $\hat{\boldsymbol{\theta}}$, $\boldsymbol{C}_{\theta,i}^{-1}$,  $\mathcal{N}_i$,  $k$, $M$, $\boldsymbol{x}_*$, $\Delta$, $\eta_{\text{NN}}$

\textbf{Output:} ${\mu}_{\textrm{DEC-NN-PoE}}$, ${\sigma}_{\textrm{DEC-NN-PoE}}^{-2}$
}
\begin{algorithmic}[1]

\ForEach {$i \in \mathcal V $}

\State $[\boldsymbol{k}_{\mu,*}]_i \gets \texttt{CrossCovCBNN}(\boldsymbol{x}_*,k, \hat{\boldsymbol{\theta}},\boldsymbol{X}_i,\boldsymbol{C}_{\theta,i}^{-1})$ \eqref{eq:cbnn_cross_covar}

\ForEach {$j \in \mathcal{N}_i $}

\If {$[\boldsymbol{k}_{\mu,*}]_j <\eta_{\text{NN}}$} \label{alg:dec_nn_poe:eval}
\State $\mathcal{N}_{\textrm{NN},i} = \mathcal{N}_i\backslash j$
\State $j \gets \texttt{flooding}(\mathcal{V})$
\State $\mathcal{V}_{\textrm{NN}} = \mathcal{V}\backslash j$
\EndIf

\EndFor
\State $M_{\textrm{NN}} = \mathrm{card}(\mathcal{V}_{\textrm{NN}})$

\EndFor

\State $\texttt{DEC-PoE}(\mathcal{D}_i, \hat{\boldsymbol{\theta}}, \boldsymbol{C}_{\theta,i}^{-1},  \mathcal{N}_{\textrm{NN},i},  k, M_{\textrm{NN}}, \boldsymbol{x}_*, \Delta)$

\State communicate $\mu_{\textrm{DEC-NN-PoE}}$ and $\sigma^2_{\textrm{DEC-NN-PoE}}$ to agents in $\mathcal{V} \backslash \mathcal{{V}_{\textrm{NN}}}$

\end{algorithmic}
\end{algorithm}

\newfloat{algorithm}{!t}{lop}
\begin{algorithm}
\caption{DEC-NN-gPoE}\label{alg:dec_nn_gpoe} 
{
\textbf{Input:} $\mathcal{D}_i(\boldsymbol{X}_i,\boldsymbol{y}_i)$, $\hat{\boldsymbol{\theta}}$, $\boldsymbol{C}_{\theta,i}^{-1}$,  $\mathcal{N}_i$,  $k$, $M$, $\boldsymbol{x}_*$, $\Delta$, $\eta_{\text{NN}}$

\textbf{Output:} ${\mu}_{\textrm{DEC-gPoE}}$, ${\sigma}_{\textrm{DEC-gPoE}}^{-2}$
}
\begin{algorithmic}[1]

\State Identical to Algorithm~\ref{alg:dec_nn_poe} with routine $\texttt{DEC-PoE}$ replaced by $\texttt{DEC-gPoE}$
\end{algorithmic}
\end{algorithm}

The covariance of the new GP (Proposition~\ref{prop:gp_mean_random_process}) yields,
\begin{align*}
   \boldsymbol{C}_{\theta, \mu} = \mathrm{Cov}[\boldsymbol{\mu}(\boldsymbol{x}_*), y(\boldsymbol{x}_*)] =  \begin{bmatrix}
\boldsymbol{K}_{\mu} & \boldsymbol{k}_{\mu,*}^{\intercal} \\
 \boldsymbol{k}_{\mu,*} & k_{**}
 \end{bmatrix},
\end{align*}
where $\boldsymbol{k}_{\mu,*}^\intercal \in \mathbb{R}^{M}$ is the cross-covariance. Interestingly, the cross-covariance element of the $i$-th agent represents the correlation of a local dataset $\mathcal{D}_i$ to the location of interest $\boldsymbol{x}_*$ with a positive scalar number $[\boldsymbol{k}_{\mu,*}]_i\in \mathbb{R}_{\geq 0}$. Essentially, this means that when the corresponding entry tends to zero $[\boldsymbol{k}_{\mu,*}]_i \rightarrow 0$, then agent $i$ is statistically uncorrelated to the location of interest $\boldsymbol{x}_*$. Every agent $i$ can compute locally its cross-covariance element as,
\begin{align}\label{eq:cbnn_cross_covar}
    [\boldsymbol{k}_{\mu,*}]_i = \boldsymbol{k}_{i,*}^{\intercal}\boldsymbol{C}_{\theta, i}^{-1} \boldsymbol{k}_{i,*},
\end{align}
where $ \boldsymbol{k}_{i,*} = k(\boldsymbol{X}_i,\boldsymbol{x}_*)$. The workflow of CBNN is as follows. Every agent $i$ computes its cross-covariance $[\boldsymbol{k}_{\mu,*}]_i$ \eqref{eq:cbnn_cross_covar}. When the correlation of agent $i$ to the location of interest is below a threshold $[\boldsymbol{k}_{\mu,*}]_i< \eta_{\textrm{NN}}$, then the agent does not take place to the aggregation of GP experts. In other words, the agent is not allowed to have an opinion on $y(\boldsymbol{x}_*)$. After all agents compute their correlation, the nearest neighbor subset of nodes is derived $\mathcal{V}_{\textrm{NN}} \subseteq \mathcal{V}$ with $M_{\textrm{NN}} = \mathrm{card}(\mathcal{V}_{\textrm{NN}})\leq M$.

\begin{lemma}\label{lem:cbnn}
Let the agents to operate in a spatial environment with input space of dimension $D=2$. Let each agent $i$ to collect local data $\mathcal{D}_i$ from a disjoint partition in stripes along the $y$-axis (Figure~\ref{fig:sst_observations}-(b)) and the network of agents to form a path graph topology (Figure~\ref{fig:graph_mixed}-(a)). Then, the exclusion of agents from the aggregation using CBNN preserves network connectivity.

\proof
The proof is provided in Appendix~\ref{app:proof:cbnn}.
\end{lemma}

The advantages of using CBNN to identify statistically correlated agents are: i) the selection of nearest neighbors is justified through a covariance not just by using an arbitrary radius; ii) only the local dataset $\mathcal{D}_i$ is required to compute \eqref{eq:cbnn_cross_covar} with no data exchange, which satisfies Assumption~\ref{ass:noData}; iii) the total communications are reduced, as a subset of the agents takes part to the aggregation $\mathcal{V}_{\textrm{NN}}$; iv) the DAC converges faster (Lemma~\ref{lem:dac_convergence}); and v) the DALE can be employed as $\boldsymbol{H}$ is ensured to be full row rank. 
\subsubsection{DEC-NN-PoE Family}
The decentralized nearest neighbor PoE (DEC-NN-PoE) family is identical to the DEC-PoE family with a CBNN selection as shown in Figure~\ref{fig:dec_nn_gp_1}. The implementation details for DEC-NN-PoE are given in Algorithm~\ref{alg:dec_nn_poe} and for DEC-NN-gPoE in Algorithm~\ref{alg:dec_nn_gpoe}. The workflow is as follows. Every agent $i$ computes the local cross-covariance of CBNN $[\boldsymbol{k}_{\mu,*}]_i$ \eqref{eq:cbnn_cross_covar} and evaluates its involvement to the aggregation (Algorithm~\ref{alg:dec_nn_poe}-[Line~\ref{alg:dec_nn_poe:eval}]). After the CBNN terminates, the remaining agents $\mathcal{V}_{\textrm{NN}}$ run the DEC-PoE family routines (Algorithm~\ref{alg:dec_poe},~\ref{alg:dec_gpoe}). Finally, the predicted values are transmitted  to the agents that did not take part to the aggregation $\mathcal{V}\backslash \mathcal{V}_{\textrm{NN}}$. The time and space computational complexity is identical to the local PoE family (Table~\ref{tab:complexityPrediction}). The communication complexity for both methods is $\mathcal{O}(2s^{\textrm{end}}_{\textrm{DAC}}\mathrm{card}(\mathcal{N}_{\textrm{NN},i}))$. Both methods address Problem~\ref{pr:prediction1}.

\newfloat{algorithm}{!t}{lop}
\begin{algorithm}
\caption{DEC-NN-BCM}\label{alg:dec_nn_bcm} 
{
\textbf{Input:} $\mathcal{D}_i(\boldsymbol{X}_i,\boldsymbol{y}_i)$, $\hat{\boldsymbol{\theta}}$, $\boldsymbol{C}_{\theta,i}^{-1}$,  $\mathcal{N}_i$,  $k$, $M$, $\boldsymbol{x}_*$, $\Delta$, $\eta_{\text{NN}}$

\textbf{Output:} ${\mu}_{\textrm{DEC-BCM}}$, ${\sigma}_{\textrm{DEC-BCM}}^{-2}$
}
\begin{algorithmic}[1]

\State Identical to Algorithm~\ref{alg:dec_nn_poe} with routine $\texttt{DEC-PoE}$ replaced by $\texttt{DEC-BCM}$
\end{algorithmic}
\end{algorithm}

\newfloat{algorithm}{!t}{lop}
\begin{algorithm}
\caption{DEC-NN-rBCM}\label{alg:dec_nn_rbcm} 
{
\textbf{Input:} $\mathcal{D}_i(\boldsymbol{X}_i,\boldsymbol{y}_i)$, $\hat{\boldsymbol{\theta}}$, $\boldsymbol{C}_{\theta,i}^{-1}$,  $\mathcal{N}_i$,  $k$, $M$, $\boldsymbol{x}_*$, $\Delta$, $\eta_{\text{NN}}$

\textbf{Output:} ${\mu}_{\textrm{DEC-rBCM}}$, ${\sigma}_{\textrm{DEC-rBCM}}^{-2}$
}
\begin{algorithmic}[1]

\State Identical to Algorithm~\ref{alg:dec_nn_poe} with routine $\texttt{DEC-PoE}$ replaced by $\texttt{DEC-rBCM}$
\end{algorithmic}
\end{algorithm}

\newfloat{algorithm}{!t}{lop}
\begin{algorithm}
\caption{DEC-NN-grBCM}\label{alg:dec_nn_grbcm} 
{
\textbf{Input:} $\mathcal{D}_{+i}(\boldsymbol{X}_{+i},\boldsymbol{y}_{+i})$, $\hat{\boldsymbol{\theta}}$, $\boldsymbol{C}_{\theta,+i}^{-1}$,  $\mathcal{N}_i$,  $k$, $M$, $\boldsymbol{x}_*$, $\Delta$, $\eta_{\text{NN}}$

\textbf{Output:} ${\mu}_{\textrm{DEC-NN-grBCM}}$, ${\sigma}_{\textrm{DEC-NN-grBCM}}^{-2}$
}
\begin{algorithmic}[1]

\State Identical to Algorithm~\ref{alg:dec_nn_poe} with routine $\texttt{DEC-PoE}$ replaced by $\texttt{DEC-grBCM}$
\end{algorithmic}
\end{algorithm}
\subsubsection{DEC-NN-BCM Family}
The decentralized nearest neighbor BCM (DEC-NN-BCM) family is identical to the DEC-BCM family with a CBNN selection (Figure~\ref{fig:dec_nn_gp_1}). The implementation details for DEC-NN-BCM are given in Algorithm~\ref{alg:dec_nn_bcm}, for DEC-NN-rBCM in Algorithm~\ref{alg:dec_nn_rbcm}, and for DEC-NN-grBCM in Algorithm~\ref{alg:dec_nn_grbcm}. The time and space complexity is identical to the local complexity of the DEC-BCM family (Table~\ref{tab:complexityPrediction}). The communication complexity for DEC-NN-BCM and DEC-NN-rBCM is $\mathcal{O}(2s^{\textrm{end}}_{\textrm{DAC}}\mathrm{card}(\mathcal{N}_{\textrm{NN},i}))$, while for DEC-NN-grBCM is $\mathcal{O}(3s^{\textrm{end}}_{\textrm{DAC}}\mathrm{card}(\mathcal{N}_{\textrm{NN},i}))$. The DEC-NN-BCM and DEC-NN-rBCM methods address Problem~\ref{pr:prediction1}, while DEC-NN-grBCM addresses Problem~\ref{pr:prediction2}.

\begin{figure}[!t]
	\includegraphics[width=.45\columnwidth]{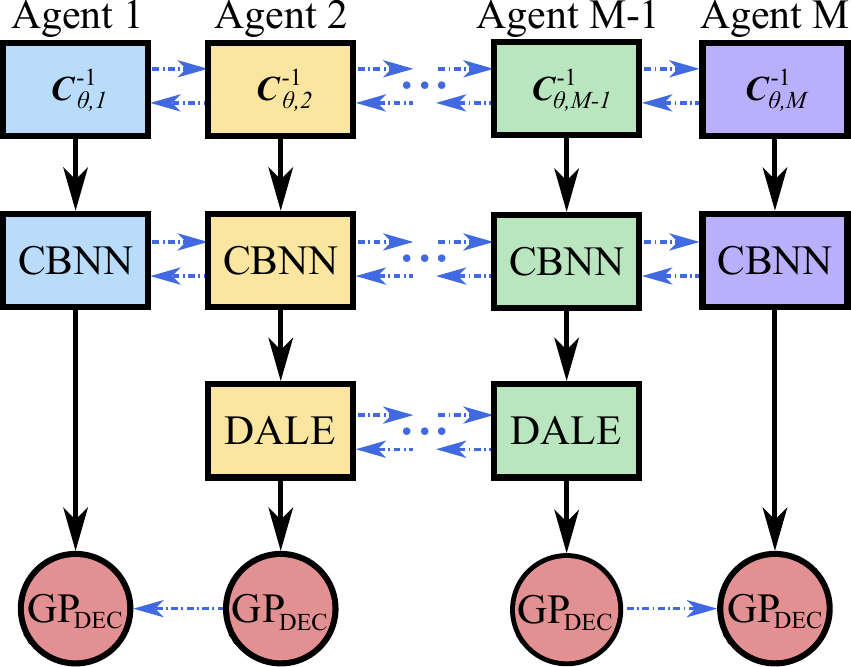}
	\centering
	\caption{The structure of the proposed nearest neighbor decentralized aggregation methods. Blue dotted lines correspond to communication (strongly connected). The covariance-based nearest neighbor (CBNN) method identifies statistically correlated agents---in this illustration the CBNN set is $\mathcal{V}_{\textrm{NN}}\in[2,M-1]$. Next, a distributed algorithm for solving a linear system of equations (DALE) is executed within the $\mathcal{V}_{\textrm{NN}}$ nodes. After convergence, the predicted values are communicated to the rest agents of the network.
	}
	\label{fig:dec_nn_gp_2}
\end{figure}

\begin{proposition}
Let the Assumption \ref{ass:repeat_connectivity},~\ref{ass:parData},~\ref{ass:independence},~\ref{ass:cond_independence},~\ref{ass:fleet} hold throughout the approximation. If $\omega<2/M$ then the DEC-NN-grBCM is consistent for any initialization. 

\proof
The proof is a direct consequence of \Cref{prop:grbcmCons} and \Cref{lem:dac_convergence},~\ref{lem:cbnn}.
\end{proposition}

\newfloat{algorithm}{!t}{lop}
\begin{algorithm}
\caption{\textsc{dec-NN-NPAE}}\label{alg:dec_nn_npae} 
{
\textbf{Input:} $\mathcal{D}_i(\boldsymbol{X}_i,\boldsymbol{y}_i)$, $\boldsymbol{X}$, $\hat{\boldsymbol{\theta}}$, $\boldsymbol{C}_{\theta,i}^{-1}$,  $\mathcal{N}_i$,  $k$, $M$, $\boldsymbol{x}_*$, $\Delta$, $\eta_{\text{NN}}$

\textbf{Output:} ${\mu}_{\textrm{DEC-NN-NPAE}}$, ${\sigma}_{\textrm{DEC-NN-NPAE}}^2$
}

\begin{algorithmic}[1]

\ForEach {$i \in \mathcal V $}

\State $[\boldsymbol{k}_A ]_i \gets \texttt{crossCov}(\boldsymbol{x}_*,k,  \hat{\boldsymbol{\theta}},\boldsymbol{X}_i, \boldsymbol{C}_{\theta,i}^{-1})$ \eqref{eq:local_covariance}

\State $[\boldsymbol{k}_{\mu,*}]_i \gets \texttt{CrossCovCBNN}(\boldsymbol{x}_*,k, \hat{\boldsymbol{\theta}},\boldsymbol{X}_i,\boldsymbol{C}_{\theta,i}^{-1})$ \eqref{eq:cbnn_cross_covar}

\ForEach {$j \in \mathcal{N}_i $}

\If {$[\boldsymbol{k}_{\mu,*}]_j <\eta_{\text{NN}}$} \label{alg:dec_nn_npae:eval}
\State $\mathcal{N}_{\textrm{NN},i} = \mathcal{N}_i\backslash j$; $\mathcal{V}_{\textrm{NN}} = \mathcal{V}\backslash j$

\State $j \gets \texttt{flooding}(\mathcal{V}_{\textrm{NN}})$

\Else 

\State $[\boldsymbol{k}_A ]_j \gets \texttt{flooding}(\mathcal{V}_{\textrm{NN}})$

\EndIf

\EndFor
\EndFor

\ForEach {$i \in \mathcal V_{\textrm{NN}} $}

\State $\mu_i \gets \texttt{localMean}(\boldsymbol{x}_*,k, \hat{\boldsymbol{\theta}},\mathcal{D}_i, \boldsymbol{C}_{\theta,i}^{-1})$ \eqref{eq:local_mean} 

\State  $\boldsymbol{C}_{\theta,i}^{-1}, \boldsymbol{X}_{i} \gets \texttt{flooding}(\mathcal{V}_{\textrm{NN}})$ \label{alg:dec_nn_npae:comm_neighbor}

\State $\boldsymbol{k}_{\textrm{NN},A} = [\boldsymbol{k}_A ]_i \cup \{[\boldsymbol{k}_A ]_j\}_{j\in\mathcal{V}_{\textrm{NN}}}$
\State $\mathrm{row}_{\textrm{NN},i}\{\boldsymbol{C}_{\theta,A}\} \gets \texttt{localCov}(\boldsymbol{x}_*,k, \boldsymbol{X}, \hat{\boldsymbol{\theta}},\mathcal{V}_{\textrm{NN}} )$ \eqref{eq:cross_covariance}

\State $\boldsymbol{H}_{i} = \mathrm{row}_{\textrm{NN},i}\{\boldsymbol{C}_{\theta,A}\}$; $M_{\textrm{NN}} = \mathrm{card}(\mathcal{V}_{\textrm{NN}})$

\State $b_{\mu,i} = \mu_{\textrm{NN},i}$; $\boldsymbol{b}_{\sigma^2} = \boldsymbol{k}_{\textrm{NN},A}$

\State $\boldsymbol{P}_{i} = I_{M_{\textrm{NN}}} - \boldsymbol{H}_{i}^{\intercal}(\boldsymbol{H}_{i}\boldsymbol{H}_{i}^{\intercal})^{-1}\boldsymbol{H}_{i}$

\State initialize $\boldsymbol{q}_{\mu,i}^{(0)} = b_{\mu,i}\oslash \boldsymbol{H}_{i}$;  $\boldsymbol{q}_{\sigma^2,i}^{(0)} = \boldsymbol{b}_{\sigma^2} \oslash \boldsymbol{H}_{i} $ 

\Repeat\Comment{2$\times$DALE}
\State communicate $\boldsymbol{q}_{\mu,i}^{(s)}$, $\boldsymbol{q}_{\sigma^2,i}^{(s)}$ to neighbors $\mathcal{N}_{\textrm{NN},i}$ \label{alg:dec_nn_npae:comm_neighbor}
\State $\boldsymbol{q}_{\mu,i}^{(s+1)} \gets \texttt{DALE}(\boldsymbol{P}_i,\boldsymbol{H}_i,b_{\mu,i},\{\boldsymbol{q}_{\mu,j}^{(s)}\}_{j\in \mathcal{N}_{\textrm{NN},i}},\mathcal{N}_{\textrm{NN},i}  )$ \eqref{eq:dale} \label{alg:dec_nn_npae_dale_mu}
\State $\boldsymbol{q}_{\sigma^2,i}^{(s+1)} \gets \texttt{DALE}(\boldsymbol{P}_i,\boldsymbol{H}_i,b_{\sigma^2,i},\{\boldsymbol{q}_{\sigma^2,j}^{(s)}\}_{j\in \mathcal{N}_{\textrm{NN},i}},\mathcal{N}_{\textrm{NN},i} )$ \eqref{eq:dale}
\label{alg:dec_nn_npae_dale_var}
\Until {{maximin stopping criterion}}

\State ${\mu}_{\textrm{DEC-NN-NPAE}}  = \boldsymbol{k}_{\textrm{NN},A}^{\intercal}\boldsymbol{q}_{\mu,i}^{\textrm{end}}$ 
\State ${\sigma}_{\textrm{DEC-NN-NPAE}}^2 ={\sigma_f^2}(k(\boldsymbol{x}_*,\boldsymbol{x}_* )- \boldsymbol{k}_{\textrm{NN},A}^{\intercal}\boldsymbol{q}_{\sigma^2,i}^{\textrm{end}})$ 

\EndFor

\end{algorithmic}
\end{algorithm}

\subsubsection{DEC-NN-NPAE}
We introduce the decentralized nearest neighbor NPAE (DEC-NN-NPAE) method to distribute the computations \eqref{eq:mean_npae}, \eqref{eq:variance_npae} of \textsc{NPAE} (Figure~\ref{fig:dec_nn_gp_2}). The DEC-NN-NPAE employs the CBNN and DALE \eqref{eq:dale} methods. By using CBNN, one can satisfy Assumption~\ref{ass:row_rank} and use the DALE. Thus, the DEC-NN-NPAE relaxes the strongly complete topology (Assumption~\ref{ass:strongly_complete}) to a strongly connected topology (Assumption~\ref{ass:repeat_connectivity}) as discussed in Remark~\ref{rem:dale_nearest_neighbor}. Implementation details are given in Algorithm~\ref{alg:dec_nn_npae}. The workflow is as follows. First, each entity computes: i) the local cross covariance $[\boldsymbol{k}_A]_i$ \eqref{eq:local_covariance}; and ii) the cross-covariance of CBNN $[\boldsymbol{k}_{\mu,*}]_i$ \eqref{eq:cbnn_cross_covar}. Next, we execute the CBNN routine to select the nearest neighbors. During the CBNN, if a criterion is met for an agent $j$ to stay in idle (Algorithm~\ref{alg:dec_nn_npae}-[Line~\ref{alg:dec_nn_npae:eval}]), it is removed from the list of agents $\mathcal{V}_{\textrm{NN}} = \mathcal{V}\backslash j$; and if not, the corresponding element of the local cross covariance $[\boldsymbol{k}_A ]_j$ is communicated to all other agents  $\mathcal{V}_{\textrm{NN}}\backslash i$. When the CBNN routine terminates, we execute the DALE method on the nearest neighbors $\mathcal{V}_{\textrm{NN}}$. Similarly to DEC-NPAE, the inputs $\{\boldsymbol{X}_j\}_{j \neq i}$ and the local inverted covariance matrices $\{\boldsymbol{C}_{\theta,j}^{-1}\}_{j\neq i}$ are communicated between CBNN agents. Next, we execute two parallel DALE algorithms with known matrix $\boldsymbol{H} = \boldsymbol{C}_{\theta,A}$ and known vectors: i) $\boldsymbol{b} = \boldsymbol{\mu}$; and ii) $\boldsymbol{b} = \boldsymbol{k}_{A}$. The first DALE is associated with the prediction mean $\mu_{\textrm{DEC-NN-NPAE}}$ (Algorithm~\ref{alg:dec_nn_npae}-[Line~\ref{alg:dec_nn_npae_dale_mu}]) and the second with the variance $\sigma_{\textrm{DEC-NN-NPAE}}^2$ (Algorithm~\ref{alg:dec_nn_npae}-[Line~\ref{alg:dec_nn_npae_dale_var}]).
After every DALE iteration, each agent $i$ communicates the computed vectors $\boldsymbol{q}_{\mu,i}^{(s)} $, $\boldsymbol{q}_{\sigma^2,i}^{(s)} $ to its neighbors $\mathcal{N}_{\textrm{NN},i}$ (Algorithm~\ref{alg:dec_nn_npae}-[Line \ref{alg:dec_nn_npae:comm_neighbor}]). Next, we update the vectors $\boldsymbol{q}_{\mu,i} $, $\boldsymbol{q}_{\sigma^2,i} $ (Algorithm~\ref{alg:dec_npae}-[Lines \ref{alg:dec_nn_npae_dale_mu}, \ref{alg:dec_nn_npae_dale_var}]) with the DALE method. When both DALE converge, each agent follows \eqref{eq:mean_npae}, \eqref{eq:variance_npae} to recover the DEC-NN-NPAE mean and variance. The local time and space complexity are identical to the local NPAE as shown in Table~\ref{tab:complexityPrediction}. Let  $s^{\textrm{end}}_{\textrm{DALE}} $ be the maximum number of iterations of DALE to converge. The total communications during the CBNN yields $\mathcal{O}(M_{\textrm{NN}})$ and during DALE $\mathcal{O}(2s^{\textrm{end}}_{\textrm{DALE}}\mathrm{card}(\mathcal{N}_{\textrm{NN},i}) + M_{\textrm{NN}}N_i^2+M_{\textrm{NN}}DN_i) = \mathcal{O}(2s^{\textrm{end}}_{\textrm{DALE}}\mathrm{card}(\mathcal{N}_{\textrm{NN},i}) +M_{\textrm{NN}}(N^2/M_{\textrm{NN}}^2+DN/M_{\textrm{NN}}))$ for all $i\in \mathcal{V}_{\textrm{NN}}$. DEC-NN-NPAE addresses Problem~\ref{pr:prediction2}.

\begin{proposition}\label{thrm:danGP}
Let Assumption \ref{ass:repeat_connectivity},~\ref{ass:parData},~\ref{ass:independence},~\ref{ass:fleet} hold throughout the approximation. Then, the DEC-NN-NPAE is consistent for any initialization of DALE.

\proof
The proof is a direct consequence of \Cref{prop:consistency} and \Cref{lem:dale_convergence},~\ref{lem:cbnn}.
\end{proposition}
\section{Numerical Experiments}\label{sec:numericalExperiments}
In this section, we perform numerical experiments to illustrate the efficiency of the proposed methods. Synthetic data with known hyper-parameters values are employed to evaluate the GP training methods in four aspects: i) hyper-parameter estimation accuracy; ii) computation time per agent; iii) communications per agent; and iv) comparison with centralized GP training techniques. A real-world dataset of sea surface temperature (SST) \citep{MUR4nasa,chin2017multi} 
is used to assess the GP prediction algorithms in four aspects: i) prediction accuracy; ii) uncertainty quantification; iii) communications per agent; and iv) comparison with aggregation of GP experts methods. All numerical experiments are conducted in MATLAB using the GPML package \citep{rasmussen2010gaussian} on an Intel Core i7-6700 CPU @3.40 GHz with 32.0 GB memory RAM. 

\begin{figure*}[!t]
	\includegraphics[width=\textwidth]{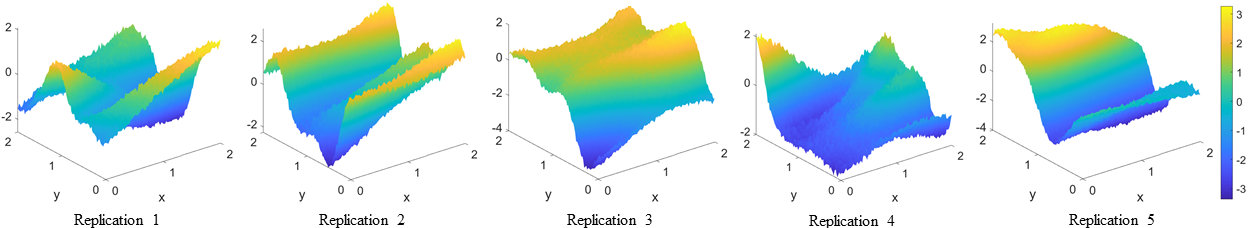}
	\centering
	\caption{Five replications of the synthetic GP with known hyper-parameter values $\boldsymbol{\theta} = (1.2, 0.3, 1.3, 0.1)^{\intercal}$  for $N = 8,100$ data.}
	\label{fig:replications_training}
\end{figure*}

\begin{figure*}[!t]
	\includegraphics[width=\textwidth]{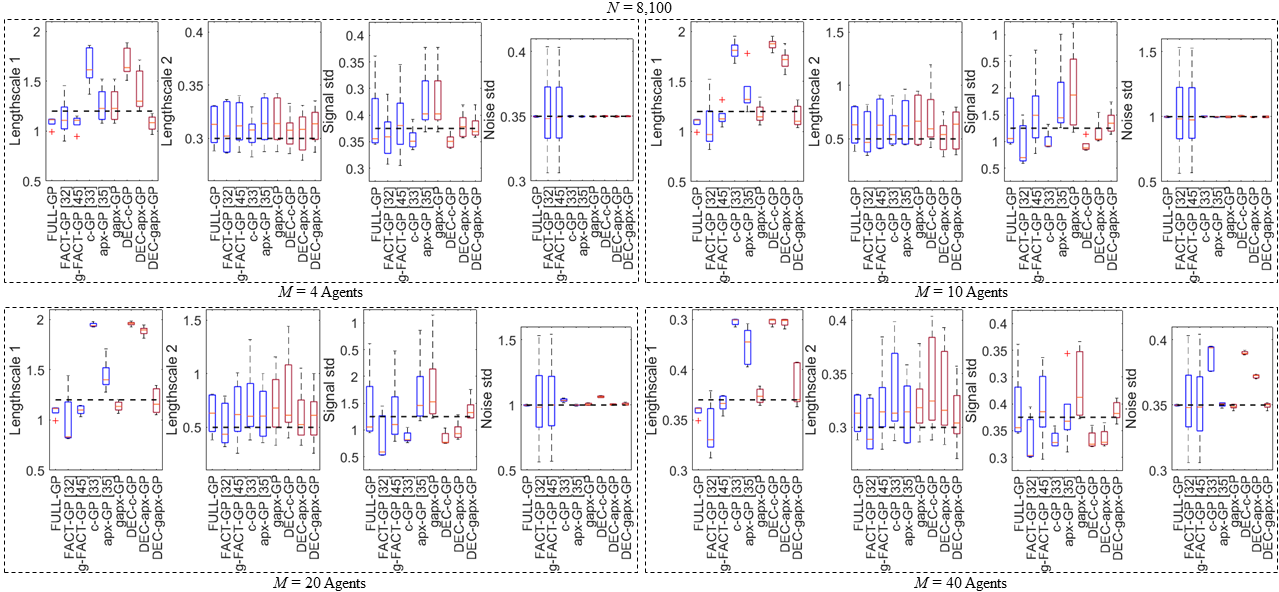}
	\centering
	\caption{Accuracy of GP hyper-parameter training using $N=8,100$ data for four fleet sizes and 10 replications. The true values are demonstrated with a black dotted line. The existing GP training methods are shown in blue boxes (FULLGP, FACT-GP \citep{deisenroth2015distributed}, g-FACT-GP \citep{liu2018generalized}, c-GP \citep{xu2019wireless}, apx-GP \citep{xie2019distributed}) and the proposed in maroon boxes (gapx-GP, DEC-c-GP, DEC-apx-GP, and DEC-gapx-GP). 
	}
	\label{fig:boxplots_training_8100}
\end{figure*}

\begin{table}[!t]
\caption{Time \& Communication Rounds of GP Training Methods}
\centering
\footnotesize{\begin{tabular}{ c l r r r r }
\toprule

 &  &  \multicolumn{2}{c}{$N=8,100$}  & \multicolumn{2}{c}{$N=32,400$} \\
 
 \cmidrule(lr){3-4} \cmidrule(lr){5-6} 
 
 \multirow{2}{*}{$M$} & \multirow{2}{*}{Method}  & \multirow{2}{*}{Time [s]} & Comms &  \multirow{2}{*}{Time [s]} & Comms \\ 
  &  &  & $s^{\textrm{end}}$ &   & $s^{\textrm{end}}$ \\ 
\midrule

 & FULL-GP & 2,114.2 & - & - & -   \\
  \midrule 
 \multirow{8}{*}{4} & FACT-GP \citep{deisenroth2015distributed} &  75.9  & 186.0 & 2,361.9 & 196.0   \\
 &  {g-FACT-GP} \citep{liu2018generalized} &  332.1 & 160.0 & >3,000 & -   \\
 &  {c-GP} \citep{xu2019wireless} &  404.1 & 141.4 & - & -  \\
 &  {apx-GP} \citep{xie2019distributed} & \textbf{26.8} & \phantom{0}43.6 & \textbf{817.6} &   45.2 \\
 &  {gapx-GP} & 67.3 & \phantom{0}\textbf{39.7} & 2,074.2 &  \textbf{42.1}  \\
 \cmidrule(lr){2-6}  
 &  {DEC-c-GP} &  414.1 & 100 & - & -  \\
 &  {DEC-apx-GP} &  \textbf{61.9} & 100 & \textbf{1,821.3} & 100   \\
 &  {DEC-gapx-GP} & 328.1 & 100 & >3,000 & -   \\

  \midrule 
 \multirow{8}{*}{10} &  FACT-GP \citep{deisenroth2015distributed} & 9.8 & 179.6 & 228.2 &  194.2\\
  &  {g-FACT-GP} \citep{liu2018generalized} & 31.8 & 131.8 & 1,035.6 & 155.2 \\
 & {c-GP} \citep{xu2019wireless} & 92.1 & 193.8 & - & -   \\
 & {apx-GP} \citep{xie2019distributed} & \textbf{3.8} & \phantom{0}47.8 & \textbf{88.8} & \phantom{0}46.8   \\
 & {gapx-GP} & 15.1 & \phantom{0}\textbf{42.2} & 522.2 & \phantom{0}\textbf{44.3} \\
 \cmidrule(lr){2-6}  
 & {DEC-c-GP} & 82.4 & 100 & - & -   \\
 & {DEC-apx-GP} & \textbf{8.4} & 100 & \textbf{188.8} & 100   \\
 & {DEC-gapx-GP} & 38.5 & 100 & 1,123.4 & 100  \\
 
  \midrule 
 \multirow{8}{*}{20} & FACT-GP \citep{deisenroth2015distributed} & 2.6 & 172.6 & 46.6 & 226.2 \\
  &  {g-FACT-GP} \citep{liu2018generalized} & 7.0 & 127.2 & 199.4 & 167.6   \\
 & {c-GP} \citep{xu2019wireless} & 31.4 & 127.8 & - & -  \\
 & {apx-GP} \citep{xie2019distributed} & \textbf{1.3} & \phantom{0}56.2 & \textbf{18.3} & \phantom{0}49.8  \\
 & {gapx-GP} & 4.1 & \phantom{0}\textbf{50.6} & 85.8 & \phantom{0}\textbf{45.6}  \\
 \cmidrule(lr){2-6}  
 & {DEC-c-GP} & 30.4 & 100 & - & -  \\
 & {DEC-apx-GP} & \textbf{2.2} & 100 & \textbf{36.9} & 100  \\
 & {DEC-gapx-GP} & 8.1 & 100 & 185.8 & 100  \\
 
  \midrule 
 \multirow{8}{*}{40} & FACT-GP \citep{deisenroth2015distributed} & 0.5 & 139.6  & \phantom{0}9.1 & 160.0 \\
  &  {g-FACT-GP} \citep{liu2018generalized} & 1.8 & 112.2 & 30.9 & 128.6 \\
 & {c-GP} \citep{xu2019wireless} & 8.9 & \phantom{0}66.6 & - & - \\
 & {apx-GP} \citep{xie2019distributed} & \textbf{0.3} & \phantom{0}56.4  & \phantom{0}\textbf{4.6} & \phantom{0}54.4 \\
 & {gapx-GP} & 1.2 & \phantom{0}\textbf{51.2} & 17.9 & \phantom{0}\textbf{49.2} \\
 \cmidrule(lr){2-6}  
 & {DEC-c-GP} & 9.1 & 100 & - & -  \\
 & {DEC-apx-GP} & \textbf{0.5} & 100 & \phantom{0}\textbf{8.2} & 100 \\
 & {DEC-gapx-GP} & 2.5 & 100 & 36.4 & 100 \\

\bottomrule
\end{tabular}
    }\label{tab:gp_training_sim}
\end{table}
\subsection{Decentralized GP Training}
We generate two sets of data with total size $N = 8,100$ and $N=32,400$ using the observation model \eqref{eq:model} and the separable squared exponential covariance function \eqref{eq:seKernel} with hyper-parameter values $\boldsymbol{\theta} = ( l_1, l_2, \sigma_f, \sigma_{\epsilon})^{\intercal} = (1.2, 0.3, 1.3, 0.1)^{\intercal}$. For every set of random functions we perform 10 replications to avoid random assignment of data. An example of five replications for $N = 8,100$ data is presented in Figure~\ref{fig:replications_training}. Note that the smaller the length-scale $l$, the more wiggly is the random function. Since $l_2<l_1$, the profile of the produced random functions is more uneven along the $y$-axis rather than the $x$-axis. Next, we equally partition the space of interest $\mathbb{S} = [0,2]^2$ along the $x$-axis according to fleet sizes $M = \{4,10,20,40\}$, and assign local datasets that lie in the corresponding local space, e.g., for $M=10$ agents see Figure~\ref{fig:sst_observations}-(b). We compare the centralized GP training methods FACT-GP \citep{deisenroth2015distributed}, g-FACT-GP \citep{liu2018generalized}, c-GP \citep{xu2019wireless}, and apx-GP \citep{xie2019distributed} to the proposed gapx-GP. In addition, we include in the comparison the proposed decentralized GP training methods DEC-c-GP, DEC-apx-GP, and DEC-gapx-GP. All decentralized GP training methods follow a path graph topology as depicted in~Figure~\ref{fig:graph_mixed}. Thus, the maximum degree of the graph is $\Delta = 2$ and its diameter $\mathrm{diam}(\mathcal{G})=M-1$. All methods start from the same initial vector value $(l_1^{(0)},l_2^{(0)}, \sigma_f^{(0)}, \sigma_{\epsilon}^{(0)})^{\intercal} = (2,0.5,1,1)^{\intercal}$. The penalty parameter of the augmented Lagrangian is set to $\rho=500$, the decentralized ADMM tolerance for convergence  $\textrm{TOL}_{\textrm{ADMM}}=10^{-3}$, the positive Lipschitz constant of the approximation \eqref{eq:pxADMM_approx} $L_i=5,000$, and the regulation positive constant of the approximation \eqref{eq:dec_px_ADMM_approx} $\kappa_i = 5,000$ for all $i\in \mathcal{V}$. For the nested optimization problem of c-GP \eqref{eq:cadmm_theta} and  DEC-c-GP \eqref{eq:dec_c_admm_theta} we use gradient descent with step size $\alpha = 10^{-5}$. All decentralized GP training methods terminate after $s^{\textrm{end}}=100$ predetermined communication rounds (Remark~\ref{rem:con_dec_xx_ADMM}), yielding identical communication complexity (Table~\ref{tab:complexityDecTrainingADMM}). Any algorithm that takes over 3,000 s to be executed is terminated.  

\begin{figure*}[!t]
	\includegraphics[width=\textwidth]{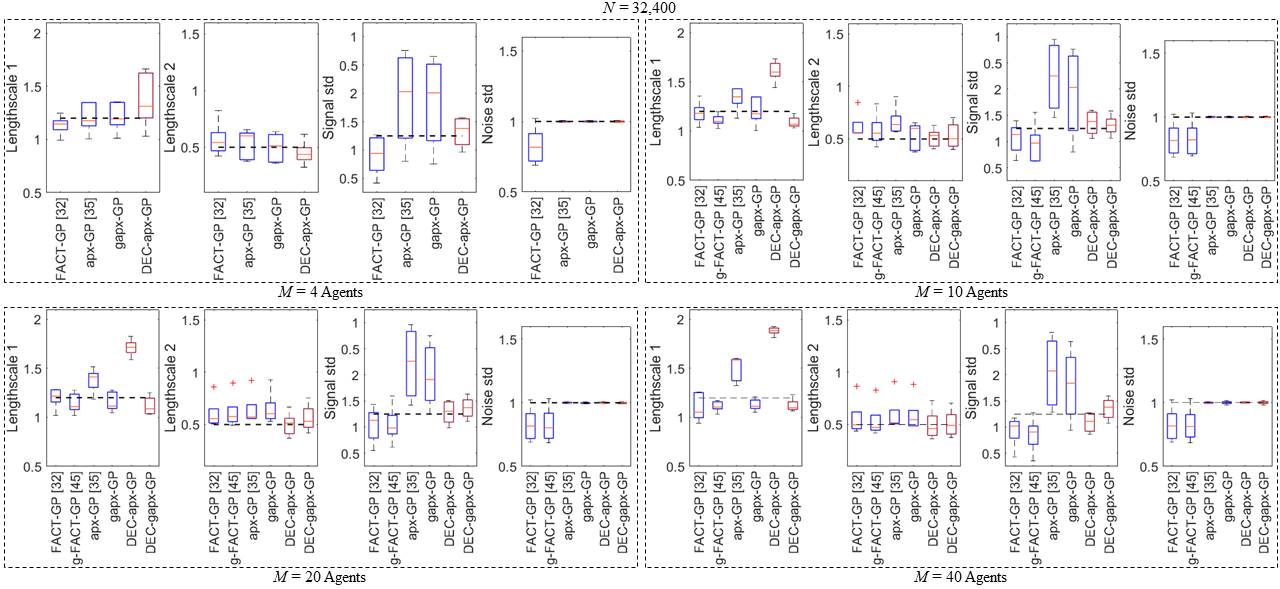}
	\centering
	\caption{Accuracy of GP hyper-parameter training using $N=32,400$ data for four fleet sizes and 10 replications. The true values are demonstrated with a black dotted line. The existing GP training methods are shown in blue boxes (FACT-GP \citep{deisenroth2015distributed}, g-FACT-GP \citep{liu2018generalized}, apx-GP \citep{xie2019distributed}) and the proposed in maroon boxes (DEC-apx-GP, and DEC-gapx-GP).
	}
	\label{fig:boxplots_training_32400}
\end{figure*}

In Figure~\ref{fig:boxplots_training_8100}, we show the boxplots of the estimated hyper-parameters for all GP training methods and all fleet sizes using $N=8,100$ data. Blue boxes illustrate existing GP training methods and maroon boxes represent the proposed GP training methods. The corresponding average computation time per agent and the communication rounds are shown in Table~\ref{tab:gp_training_sim}. Provided the communication rounds $s^{\textrm{end}}$, the  communication complexity can be computed according to Table~\ref{tab:complexityTraining},~\ref{tab:complexityTrainingADMM}. For the case of $M=4$ agents, all centralized methods provide accurate hyper-parameters estimates except of the c-GP on $l_1$. In terms of computation time, c-GP is the more demanding method, while FACT-GP, apx-GP, and gapx-GP convergence very fast, outperforming FULL-GP two orders of magnitude for similar or even better level of accuracy. The least communication rounds are achieved by the proposed methodology~gapx-GP which results in the lowest communication complexity. Regarding the decentralized methods, both DEC-apx-GP and DEC-gapx-GP produce accurate hyper-parameter estimates, while DEC-c-GP is inaccurate on $l_1$. DEC-apx-GP requires less computation time per agent than the other two decentralized methods. As we increase the number of agents ($M=10$ and $M=20$ agents), the hyper-parameter estimation accuracy deteriorates for all centralized methods except of the proposed gapx-GP. In addition, gapx-GP results in the lowest communication complexity and in competitive computation time per agents, outperformed only by apx-GP. Regarding the decentralized GP training methods, the hyper-parameter estimation of DEC-gapx-GP is the most accurate. Both DEC-apx-GP and DEC-c-GP provide reasonable estimates for all hyper-parameters other than $l_1$ which is inaccurate. The lowest computation per entity is measured for DEC-apx-GP, while the most accurate method DEC-gapx-GP requires four times more computations than DEC-apx-GP. For $M=40$ agents, the proposed gapx-GP produces the most accurate hyper-parameter estimates with only g-FACT-GP competing with reasonable accuracy. However, g-FACT-GP requires more computation time per agent and exchanges double the amount of messages to converge than the proposed gapx-GP. From the proposed decentralized methods, only DEC-gapx-GP is accurate  (Remark~\ref{rem:con_apx_ADMM_GP}) and requires reasonable computation per agent. 

We present the boxplots of the estimated hyper-parameters using $N=32,400$ data and for all fleet sizes in Figure~\ref{fig:boxplots_training_32400}, while in Table~\ref{tab:gp_training_sim} we list the corresponding computation time per agent as well as the communication rounds. The FULL-GP, c-GP, and DEC-c-GP methods are not implemented for $N=32,400$ data, as we expect significantly high computation time (Remark~\ref{rem:con_c-ADMM}). For $M=4$ agents, both g-FACT-GP and DEC-gapx-GP exceeded the time limit (3,000 s) for convergence. Among the feasible centralized methods for $N=32,400$ data, apx-GP and gapx-GP are more accurate than FACT-GP. All methods are computationally expensive as each agent $i$ is assigned with $N_i = 32,400/4=8,100 $ data, yet apx-GP is the fastest. Regarding the decentralized methods, DEC-apx-GP is the only feasible method and produces accurate hyper-parameter estimates. As we increase the number of agents ($M=10$ and $M=20$ agents), the number of data is distributed to local agents, and thus g-FACT-GP and DEC-gapx-GP can be implemented. Since the number of data is high, all centralized methods produce accurate hyper-parameters estimates. Yet, apx-GP is computationally more efficient. Although the proposed gapx-GP requires more time to converge, the communication overhead is the least. Among the decentralized methods, DEC-gapx-GP is more accurate, but computationally more demanding than DEC-apx-GP. For the case of $M=40$ agents, the most accurate centralized hyper-parameter estimator is the gapx-GP with the lowest information exchange requirements. The fastest centralized method is the apx-GP, yet its accuracy is moderate. Regarding the decentralized methods, DEC-gapx-GP remains accurate and requires reasonable computation time.

Overall, for $N=8,100$ the proposed gapx-GP is the most accurate centralized GP training method, especially as the fleet size increases. Moreover, gapx-GP requires reasonable computations and it is the most efficient method with respect to communication. Among the proposed decentralized GP training methods, DEC-gapx-GP is the most accurate method, yet DEC-apx-GP produces competitive hyper-parameter estimates for medium and small fleet size. DEC-apx-GP is the fastest decentralized GP training method, while DEC-gapx-GP is more demanding, yet requires reasonable computational resources. In principle, as we increase the number of agents, the computation is distributed and thus yields lower computation time per agent. Note that the hyper-parameter estimation accuracy improves as we obtain more data which leads to higher accuracy for $N=32,400$ data (Remark~\ref{rem:multi_modal}). Some techniques are not scalable for the larger dataset $N=32,400$, especially when the fleet size is small $M=4$. However, for larger fleet size the distribution of data facilitates the execution of most methods. Among the centralized methods, apx-GP is accurate and requires significantly less computational time for small fleet size, but as we increase the number of agents the proposed gapx-GP becomes computationally more efficient and remains accurate. Similarly, DEC-apx-GP is accurate and computationally less demanding for small fleet size, but DEC-gapx-GP becomes more computationally efficient as we distribute the data to more agents.

\begin{figure}[!t]
	\includegraphics[width=.8\columnwidth]{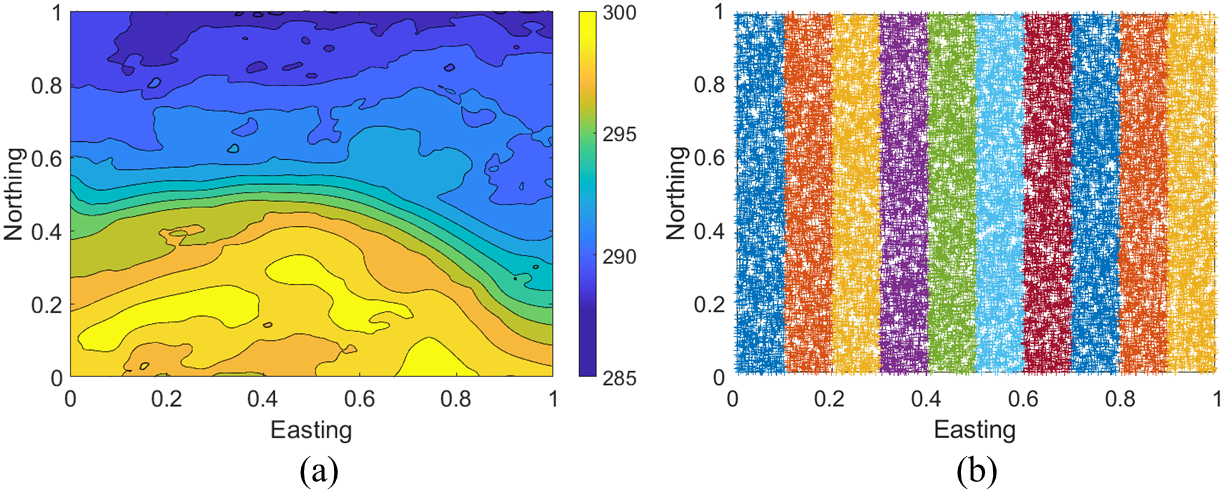}
	\centering
	\caption{(a) SST field \citep{MUR4nasa}; (b) Observations of each agent for $M=10$.
	}
	\label{fig:sst_observations}
\end{figure}

\subsection{Decentralized GP Prediction}
We use a real-world dataset of sea surface temperature (SST) \citep{MUR4nasa,chin2017multi}. We extract $122,500$ SST values from ($36.4^{\textrm{o}},-73.0^{\textrm{o}}$) to ($40.0^{\textrm{o}},-69.4^{\textrm{o}}$) measured in Kelvins. The area corresponds to $400\, $km $\times$ $400\, $km of the Atlantic ocean and for demonstration is normalized over $[0,1]^2$ (Figure \ref{fig:sst_observations}-(a)). Additionally, we add iid noise $\epsilon \sim \mathcal{N}(0,0.25)$ to the observations \eqref{eq:model}. We use $20,000$ observations, equally distributed for four fleet sizes ${M} = \{4,10,20,40\}$. An example of data distribution assignment for $M=10$ agents is shown in Figure \ref{fig:sst_observations}-(b). The GP training of the hyper-parameters is performed with the DEC-gapx-GP method. We employ 13 techniques over $N_{\textrm{t}} = 100$ prediction points: i) DEC-PoE with path graph; ii) DEC-NN-PoE with path graph; iii) DEC-gPoE with path graph; iv) DEC-NN-gPoE with path graph; v) DEC-BCM with path graph; vi) DEC-NN-BCM with path graph; vii) DEC-rBCM with path graph; viii) DEC-NN-rBCM with path graph; ix) DEC-grBCM with path graph; x) DEC-NN-grBCM with path graph; xi) DEC-NPAE with strongly complete graph; xii) DEC-NPAE$^{\star}$ with path graph and strongly complete graph; and xiii) DEC-NN-NPAE with path graph, where the graph types are shown in Figure \ref{fig:graph_mixed}. For every scenario we perform $15$ replications to avoid random assignment of data. 

The quality assessment is accomplished with two metrics. The root mean square error  $\textrm{RMSE} = [{1}/{N} \sum_{i = 1}^{N} ( \mu(\boldsymbol{x}_{*}) - {y}(\boldsymbol{x}_{*}))^2]^{1/2}$ assesses the prediction mean. The negative log predictive density $\textrm{NLPD} = -{1}/{N} \sum_{i = 1}^{N} \log p(\hat{\boldsymbol{y}}_*\mid \mathcal{D}, \boldsymbol{x}_*)$ characterizes the prediction mean and variance, where $p(\hat{\boldsymbol{y}}_*\mid \mathcal{D}, \boldsymbol{x}_*)$ is the predictive distribution \citep{quinonero2005evaluating}. 

In Figure~\ref{fig:RMSE_NLPD_PoE_family}, we show the average RMSE and NLPD values for four fleet sizes and 15 replications using the decentralized PoE-based methods. Since the proposed algorithms DEC-PoE, DEC-NN-PoE; DEC-gPoE, and DEC-NN-gPoE approximate the PoE; and gPoE respectively, the optimal RMSE and NLPD values are that of PoE~\citep{hinton2002training} and gPoE~\citep{cao2014generalized}. All PoE-based methods produce identical RMSE accuracy, illustrating that the proposed decentralized methods converge with almost zero approximation error. Indeed, PoE and gPoE have identical mean prediction values, validating Proposition~\ref{prop:poe_gpoe_mean}. In terms of uncertainty quantification, DEC-PoE and DEC-NN-PoE are characterised by the same NLPD values with the PoE. Moreover, they all fail to report NLPD values for larger fleet sizes ($M=20$ and $M=40$ agents), as the predictive variance of PoE \eqref{eq:variance_poe} is additive and leads to overconfident results which subsequently yield infinite values of NLPD. Similarly, DEC-gPoE and DEC-NN-gPoE report identical NLPD values with the gPoE. Both nearest neighbor methods (DEC-NN-PoE and DEC-NN-gPoE) produce results that are indistinguishable to  PoE and gPoE respectively, even though 42.5\% of agents were excluded on average from the aggregation (Table~\ref{tab:gp_prediction_sim}).

\begin{figure}[!t]
	\includegraphics[width=.75\columnwidth]{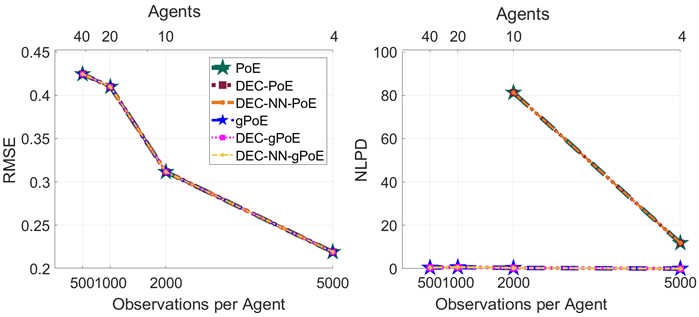}
	\centering
	\caption{Average RMSE and NLPD values for four fleet sizes and 15 replications with the PoE-based methods on a path graph topology.
	}
	\label{fig:RMSE_NLPD_PoE_family}
\end{figure}

\begin{figure}[!t]
	\includegraphics[width=.7\columnwidth]{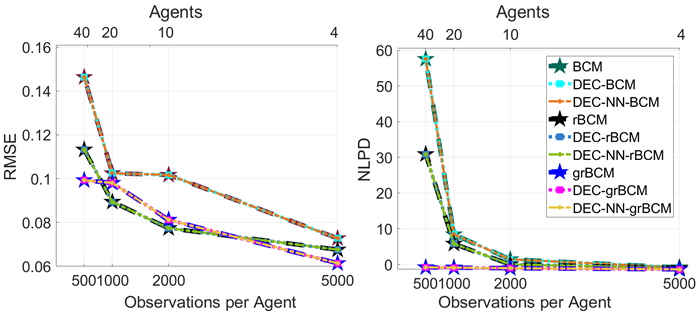}
	\centering
	\caption{Average RMSE and NLPD values for four fleet sizes and 15 replications with the BCM-based methods on a path graph topology.
	}
	\label{fig:RMSE_NLPD_BCM_family}
\end{figure}

In Figure~\ref{fig:RMSE_NLPD_BCM_family}, we present the average RMSE and NLPD values for four fleet sizes and 15 replications using the decentralized BCM-based methods. Since the proposed algorithms DEC-BCM, DEC-NN-BCM; DEC-rBCM, DEC-NN-rBCM; DEC-grBCM, and DEC-NN-grBCM approximate the BCM~\citep{tresp2000bayesian}; rBCM~\citep{deisenroth2015distributed}; and grBCM~\citep{liu2018generalized} respectively, the optimal RMSE and NLPD values are that of BCM, rBCM and grBCM. We observe that DEC-BCM and DEC-NN-BCM converge to BCM as they report identical RMSE and NLPD values. Similarly, for the rest decentralized methods, i.e., DEC-rBCM, DEC-NN-rBCM converge to rBCM, and DEC-grBCM, DEC-NN-grBCM converge to grBCM with almost zero approximation error. All nearest neighbor methods DEC-NN-BCM, DEC-NN-rBCM, and DEC-NN-grBCM make identical predictions to BCM, rBCM, and grBCM, although a subset of agents are selected to participate in the prediction.

\begin{figure}[!t]
	\includegraphics[width=.7\columnwidth]{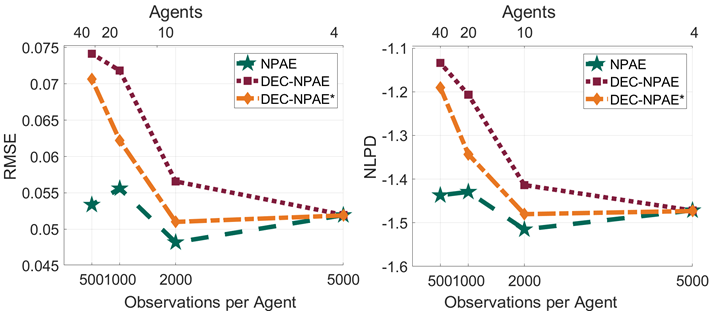}
	\centering
	\caption{Average RMSE and NLPD values for four fleet sizes and 15 replications with the NPAE-based methods on a strongly complete topology.
	}
	\label{fig:RMSE_NLPD_NPAE_family}
\end{figure}

The average RMSE and NLPD values for four fleet sizes and 15 replications using the decentralized NPAE-based methods are presented in Figure~\ref{fig:RMSE_NLPD_NPAE_family},~\ref{fig:RMSE_NLPD_dist_NPAE}. The difference between Figure~\ref{fig:RMSE_NLPD_NPAE_family} and Figure~\ref{fig:RMSE_NLPD_dist_NPAE} is that the latter demonstrates methods in a strongly connected network topology (path graph), while the former methods in a strongly complete network topology (see Figure~\ref{fig:graph_mixed} for differences). Since the proposed algorithms DEC-NPAE, DEC-NPAE$^{\star}$, and DEC-NN-NPAE approximate the NPAE, the optimal RMSE and NLPD values are that of NPAE~\citep{rulliere2018nested}. The main difference of DEC-NPAE and DEC-NPAE$^{\star}$ lies in the estimation of the optimal relaxation factor rather than selecting the relaxation factor based on the fleet size (Remark~\ref{rem:jor_jor_star}). All DEC-NPAE, DEC-NPAE$^{\star}$, and DEC-NN-NPAE produce an approximation error, and thus they do not converge to the optimal RMSE and NLPD values of NPAE. More specifically, DEC-NPAE$^{\star}$ has the smallest approximation error (Figure~\ref{fig:RMSE_NLPD_NPAE_family}), while DEC-NN-NPAE reports high approximation error (Figure~\ref{fig:RMSE_NLPD_dist_NPAE}). Since $s^{\textrm{end}}_{\omega^{\star}}+s^{\textrm{end}}_{\textrm{JOR}^{\star}}<s^{\textrm{end}}_{\textrm{JOR}}$, the DEC-NPAE$^{\star}$ converges faster than DEC-NPAE. This advocates that the proposed scheme to estimate the optimal relaxation factor before implementing the JOR, is more efficient both in speed and accuracy. Thus, the proposed method outperforms other approaches that employ the JOR~ \citep{cortes2009distributed,choi2015distributed,choudhary2017distributed}.

\begin{figure}[!t]
	\includegraphics[width=.7\columnwidth]{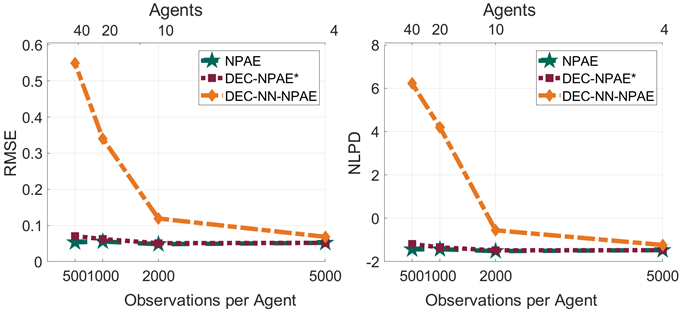}
	\centering
	\caption{Average RMSE and NLPD values for four fleet sizes and 15 replications with the NPAE-based methods on a path graph topology.
	}
	\label{fig:RMSE_NLPD_dist_NPAE}
\end{figure}

In Table~\ref{tab:gp_prediction_sim}, we compare the average computation time for each agent and the communication rounds of all nearest neighbor methods on a path graph network topology. In addition, we compute the average number of nearest neighbors from 15 replications that participate in the prediction $M_{\textrm{NN}}$ for all $N_{\textrm{t}} = 100$ prediction points of each fleet size. The results reveal a 42.5\% agent reduction with no approximation error for DEC-NN-PoE, DEC-NN-gPoE, DEC-NN-BCM, DEC-NN-rBCM, and DEC-NN-grBCM (Figure~\ref{fig:RMSE_NLPD_PoE_family},~\ref{fig:RMSE_NLPD_BCM_family}); and significant approximation error for DEC-NN-NPAE (Figure~\ref{fig:RMSE_NLPD_dist_NPAE}). The computation time per agent and the communication rounds are similar for all methods other than the DEC-NN-NPAE. Thus, DEC-NN-NPAE is insufficient in accuracy, computation time per agent, and communications, while all other decentralized nearest neighbor methods report optimal RMSE and NLPD values,  scalable computation time, and require little information exchange.

\begin{table}[!t]
\caption{Decentralized CBNN Aggregation Methods }
\centering
{\begin{tabular}{ c l r r r}
\toprule

 \multirow{2}{*}{$M$} & \multirow{2}{*}{Method}  & Nearest & Time per & Comms \\ 
 
  &  & Neighbors $M_{\textrm{NN}}$ & Agent [s] & $s^{\textrm{end}}$ \\ 
\midrule

 \multirow{6}{*}{4} & DEC-NN-PoE  &  \multirow{6}{*}{2.3$\pm$0.1}  & 0.0312 & 3.6  \\
 &  DEC-NN-gPoE  &  & 0.0339 & 3.6\\
 &  DEC-NN-BCM &  & 0.0323 & 3.6\\
 &  DEC-NN-rBCM &  & 0.0469  & 3.6\\
 &  DEC-NN-grBCM & & 0.0516 & 3.6\\
 &  DEC-NN-NPAE &   & 1.1683  & 69.4 \\

  \midrule 
 \multirow{6}{*}{10} & DEC-NN-PoE  & \multirow{6}{*}{5.7$\pm$0.2}  & 0.0122    & 7.2\\
 &  DEC-NN-gPoE  &   & 0.0121   & 7.2\\
 &  DEC-NN-BCM &  & 0.0126  & 7.2\\
 &  DEC-NN-rBCM &  & 0.0172  & 7.2\\
 &  DEC-NN-grBCM &  & 0.0167 & 7.2\\
 &  DEC-NN-NPAE &   & 0.4844  & 247.2\\
 
  \midrule 
 \multirow{6}{*}{20} & DEC-NN-PoE  &  \multirow{6}{*}{11.3$\pm$0.4}  &    0.0087 & 6.8 \\
 &  DEC-NN-gPoE  &   & 0.0083  & 6.8 \\
 &  DEC-NN-BCM &   &  0.0086 & 6.8 \\
 &  DEC-NN-rBCM &  & 0.0126  & 7.0 \\
 &  DEC-NN-grBCM & & 0.0124 & 7.0 \\
 &  DEC-NN-NPAE &   &  0.2698 & 625.6 \\
 
  \midrule 
 \multirow{6}{*}{40} & DEC-NN-PoE  &  \multirow{6}{*}{23.6$\pm$1.1}  &  0.0052   & 7.4 \\
 &  DEC-NN-gPoE  &   &  0.0049  & 7.4 \\
 &  DEC-NN-BCM &   &  0.0051 & 7.4 \\
 &  DEC-NN-rBCM &  &  0.0073 & 7.4 \\
 &  DEC-NN-grBCM &  & 0.0071 & 7.4\\
 &  DEC-NN-NPAE &   &  2.7009 & 1,824.1\\

\bottomrule
\end{tabular}
    }\label{tab:gp_prediction_sim}
\end{table}

\begin{figure*}[!t]
	\includegraphics[width=\textwidth]{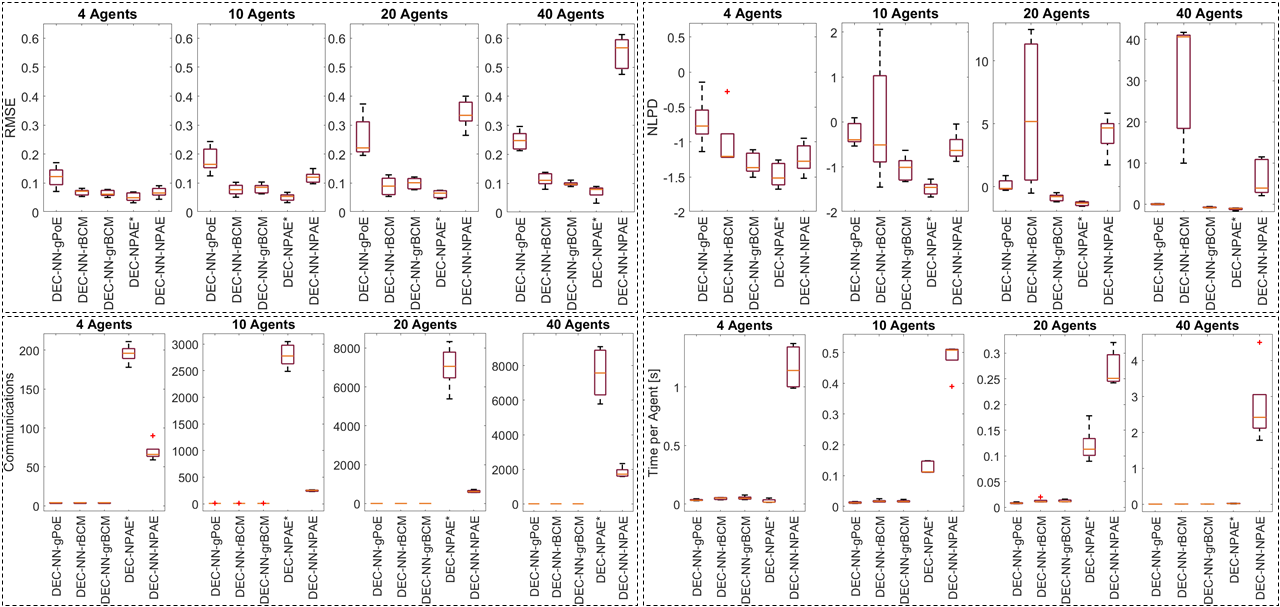}
	\centering
	\caption{Comparison of accuracy, uncertainty quantification, communication rounds $s^{\textrm{end}}$, and computation time per agent for four fleet sizes and 15 replications on decentralized GP predictions at $N_{\textrm{t}}=100$ unknown locations using $N=20,000$ observations in a path graph network topology. Lower RMSE and NLPD values indicate better accuracy and better uncertainty quantification respectively. The comparison includes the five best  decentralized GP prediction methods out of the 13 proposed methods.
	}
	\label{fig:boxplots_prediction_20000}
\end{figure*}

In Figure~\ref{fig:boxplots_prediction_20000}, we compare five out of the 13 proposed methods that produce accurate predictions and properly quantify the uncertainty. The comparison includes the DEC-NN-gPoE, DEC-NN-rBCM, DEC-NN-grBCM, DEC-NPAE$^{\star}$, and DEC-NN-NPAE for all fleet sizes in a path graph topology. Accuracy is evaluated with RMSE, uncertainty quantification with NLPD, communication complexity with communication rounds, and scalability with computation time per agent. In terms of accuracy all methods perform well by producing low RMSE values for small fleet sizes. However, as we increase the the munber of agents only DEC-NN-rBCM, DEC-NN-grBCM, and DEC-NPAE$^{\star}$ recover good accuracy with the later being the most accurate. Similarly for uncertainty quantification, all methods quantify satisfactorily the uncertainty by reporting low NLPD values for small fleet size. Yet, as the fleet sizes increases only DEC-NN-gPoE, DEC-NN-grBCM, and DEC-NPAE$^{\star}$ maintain good level of uncertainty quantification with DEC-NPAE$^{\star}$ being the best. Inter-agent communication favors DEC-NN-gPoE, DEC-NN-rBCM, and DEC-NN-grBCM for all fleet sizes. Notably the most accurate method both in terms of RMSE and NLPD (DEC-NPAE$^{\star}$) requires signification information exchange to converge. In terms of scalability, DEC-NN-gPoE, DEC-NN-rBCM, and DEC-NN-grBCM are executed very fast, while DEC-NPAE$^{\star}$ requires reasonable computations. A qualitative assessment of the comparison in Figure~\ref{fig:boxplots_prediction_20000} for all four aspects is presented on Table~\ref{tab:quality_assess}. The results reveal that DEC-NN-grBCM is overall the best decentralized GP prediction method; DEC-NPAE$^{\star}$ is an accurate method, quantifies the uncertianty well, and entails reasonable computations, but requires signification communication; and DEC-NN-rBCM is accurate, scalable, requires little information exchange, but quantifies the uncertainty poorly.

\begin{table}[!t]
\caption{Qualitative Assessment of Decentralized GP Methods}
\centering
{\begin{tabular}{ l c c c c}
\toprule

 \multirow{2}{*}{Method} & RMSE & NLPD & \multirow{2}{*}{Comms} & \multirow{2}{*}{Scalable} \\ 
 
  & Accuracy & UQ &  &  \\ 
\midrule

 DEC-NN-gPoE  & Moderate  & Moderate & Excellent & Excellent  \\
 DEC-NN-rBCM  & Excellent & Bad  & Excellent & Excellent \\
 DEC-NN-grBCM  & Excellent & Excellent & Excellent & Excellent \\
 DEC-NPAE$^{\star}$  & Excellent & Excellent & Bad & Moderate \\
 DEC-NN-NPAE & Bad & Moderate & Moderate & Bad \\

\bottomrule
\end{tabular}
    }\label{tab:quality_assess}
\end{table}
\section{Conclusion and Future Work}\label{sec:conclusion}
This paper proposes methods to implement GPs in decentralized networks that cover a broad spectrum of multi-agent learning applications. The proposed methods can be employed both for decentralized GP training and decentralized GP prediction on various fleet sizes with different computation and communication capabilities of local agents. We use distributed optimization methods of ADMM to perform accurate and scalable GP training in networks. More specifically, a closed-form solution of the decentralize ADMM is derived for the case of GP hyper-parameter training with maximum likelihood estimation. DEC-apx-GP is shown to achieve competitive accuracy in hyper-parameter estimates for small and medium fleet sizes, while DEC-gapx-GP produces accurate hyper-parameter estimates for all fleet sizes with reasonable computations of local entities. Additionally, we propose a centralized GP training method, the gapx-GP, that improves the accuracy of hyper-parameter estimates for medium and large fleet sizes, entails reasonable computations, and requires little information exchange. Next, we use iterative and consensus protocols to decentralize the implementation of various aggregation of GP experts methods. We propose 13 techniques that can be used in various applications depending on the fleet size, the computational resources of local agents, and the communication capabilities. Moreover, we introduce a nearest neighbor selection method, namely CBNN, that excludes agents with no statistical correlation from the GP prediction. Although the CBNN achieves on average 42.5\% agent reduction, it does not sacrifice prediction accuracy, and leads to significant computation and communication reduction. Most of the proposed decentralized GP prediction methods converge to the optimal values without reporting approximation error for all fleet sizes. The decentralized NPAE-based methods converge with approximation error, yet for DEC-NPAE$^{\star}$ the error is insignificant. DEC-NPAE$^{\star}$ and DEC-NN-grBCM are the most competitive methods for all fleet sizes both in terms of accuracy and uncertainty quantification, yet DEC-NN-grBCM is also scalable with low communication~overhead. 
\newpage
\appendix
\section{Gradients}\label{app:grad}
\subsection{Partial derivative of covariance matrix}\label{app:derivative_SE_covariance}
The partial derivative of the covariance matrix in \eqref{eq:grad_mle} is computed with respect to each hyperparameter as $\partial \boldsymbol{C}_{\theta}/ \partial \boldsymbol{\theta} = (\partial \boldsymbol{C}_{\theta}/ \partial {l_1},\partial \boldsymbol{C}_{\theta}/ \partial {l_2},\hdots,\partial \boldsymbol{C}_{\theta}/ \partial {l_D},\partial \boldsymbol{C}_{\theta}/ \partial {\sigma_f}, \partial \boldsymbol{C}_{\theta}/ \partial {\sigma}_{\epsilon} )^{\intercal}\in \mathbb{R}^{(D+2)N\times N} $. In particular, for each length-scale $l_d$ we obtain,
\begin{align*}
    \left[\frac{\partial \boldsymbol{C}_{\theta}}{\partial {l_d}} \right]_{ij} &=  \sigma_f^2\left[\exp\left\{-\sum_{d=1}^D\frac{({x}_{di} - {x}_{dj} )^2}{l_d^2} \right\}\frac{2(x_{di} -x_{dj} )^2}{l_d^3}\right]_{ij}\\
    &=\frac{2[\boldsymbol{K}]_{ij}\left[(x_{di} -x_{dj} )^2\right]_{ij}}{l_d^3},
\end{align*}
where $\partial \boldsymbol{C}_{\theta}/ \partial {l_d} \in \mathbb{R}^{N\times N}$. For the signal variance we get,
\begin{align*}
    \left[\frac{\partial \boldsymbol{C}_{\theta}}{\partial {\sigma_f}} \right]_{ij} &=  2\sigma_f \left[\exp\left\{-\sum_{d=1}^D\frac{({x}_{di} - {x}_{dj} )^2}{l_d^2} \right\} \right]_{ij}\\
    &=\frac{2[\boldsymbol{K}]_{ij}}{\sigma_f},
\end{align*}
where $\partial \boldsymbol{C}_{\theta}/ \partial {\sigma_f} \in \mathbb{R}^{N\times N}$. Note that we express the partial derivatives as functions of the correlation matrix $\boldsymbol{K}$, because it has already been computed to construct the covariance matrix, i.e., $\boldsymbol{C}_{\theta} = \boldsymbol{K} +\sigma_{\epsilon}^2I_N$. Lastly, for the 
noise variance ${\partial \boldsymbol{C}_{\theta}}/{\partial {\sigma_{\epsilon}}}  =  2\sigma_{\epsilon}I_N\in \mathbb{R}^{N\times N} $.
\subsection{Gradient for nested problem of DEC-c-GP}\label{app:grad_dec_c_admm}
Let the objective for the nested optimization problem \eqref{eq:dec_c_admm_theta} of the DEC-c-GP to be $\mathcal{K}=\mathcal{L}_i(\boldsymbol{\theta}_i) + \boldsymbol{\theta}_i^{\intercal}\boldsymbol{p}_i^{(s+1)} + \rho \sum_{j\in \mathcal{N}_i} \|\boldsymbol{\theta}_i - ({\boldsymbol{\theta}_i^{(s)}+\boldsymbol{\theta}_j^{(s)}})/{2} \|_2^2$, then its gradient yields,
\begin{align*}
    \frac{\partial \mathcal{K}}{\partial \boldsymbol{\theta}_i}  = \nabla_{\boldsymbol{\theta}_i}\mathcal{L}_i(\boldsymbol{\theta}_i) + \boldsymbol{p}_i^{(s+1)}+2\rho \sum_{j\in \mathcal{N}_i} \boldsymbol{\theta}_i - \frac{\boldsymbol{\theta}_i^{(s)}+\boldsymbol{\theta}_j^{(s)}}{2} .
\end{align*}
Note that $\nabla_{\boldsymbol{\theta}_i}\mathcal{L}_i$ can be computed as in Appendix~\ref{app:derivative_SE_covariance}.
\section{Proofs}
\subsection{Proof of Proposition~\ref{prop:poe_gpoe_mean}}\label{app:proof:poe_gpoe_mean}
The prediction mean value of any agent $i$ using PoE yields,
\begin{align}\nonumber
    \mu_{\textrm{PoE}}( \boldsymbol{x}_*) &= \sigma_{\textrm{PoE}}^2(\boldsymbol{x}_*)\sum_{i=1}^M \beta_i \sigma_i^{-2}(\boldsymbol{x}_*)\mu_i(\boldsymbol{x}_*)\\ \nonumber
    &= \left( \sum_{i=1}^M \beta_i \sigma_i^{-2}(\boldsymbol{x}_*) \right)^{-1}\sum_{i=1}^M \beta_i \sigma_i^{-2}(\boldsymbol{x}_*)\mu_i(\boldsymbol{x}_*)\\ \label{eq:mean_poe_olny}
    &= \sum_{i=1}^M  \sigma_i^{2}(\boldsymbol{x}_*) \sum_{i=1}^M  \sigma_i^{-2}(\boldsymbol{x}_*)\mu_i(\boldsymbol{x}_*).
\end{align}
The prediction mean of the $i$-th agent using gPoE yields,
\begin{align}\nonumber
    \mu_{\textrm{gPoE}}( \boldsymbol{x}_*) &= \sigma_{\textrm{gPoE}}^2(\boldsymbol{x}_*)\sum_{i=1}^M \beta_i \sigma_i^{-2}(\boldsymbol{x}_*)\mu_i(\boldsymbol{x}_*)\\ \nonumber
    &= \left( \sum_{i=1}^M \beta_i \sigma_i^{-2}(\boldsymbol{x}_*) \right)^{-1}\sum_{i=1}^M \beta_i \sigma_i^{-2}(\boldsymbol{x}_*)\mu_i(\boldsymbol{x}_*)\\ \nonumber
    &= \left( \sum_{i=1}^M  \frac{1}{M}\sigma_i^{-2}(\boldsymbol{x}_*)\right)^{-1} \sum_{i=1}^M \frac{1}{M} \sigma_i^{-2}(\boldsymbol{x}_*)\mu_i(\boldsymbol{x}_*)\\ \nonumber
    &=\left(\frac{1}{M}\right)^{-1}\frac{1}{M}\sum_{i=1}^M  \sigma_i^{2}(\boldsymbol{x}_*)\sum_{i=1}^M  \sigma_i^{-2}(\boldsymbol{x}_*)\mu_i(\boldsymbol{x}_*)\\ \label{eq:mean_gpoe_olny}
    &=\sum_{i=1}^M  \sigma_i^{2}(\boldsymbol{x}_*) \sum_{i=1}^M  \sigma_i^{-2}(\boldsymbol{x}_*)\mu_i(\boldsymbol{x}_*).
\end{align}
Hence, from \eqref{eq:mean_poe_olny}, \eqref{eq:mean_gpoe_olny} $\mu_{\textrm{PoE}}( \boldsymbol{x}_*) = \mu_{\textrm{gPoE}}( \boldsymbol{x}_*) $ for all $i\in \mathcal{V}$.
\subsection{Proof of Theorem~\ref{thrm:dec_apx_admm_gp}}\label{app:proof:thrm:dec_apx_admm_gp}
Let us employ the local objective of \eqref{eq:dec-pxadmm_theta} as,
\begin{align*}\nonumber
    \mathcal{Q}_i(\boldsymbol{\theta}_i) = & \ \nabla_{\boldsymbol{\theta}}^{\intercal}\mathcal{L}_i\left(\boldsymbol{\theta}_i^{(s)}\right)\left(\boldsymbol{\theta}_i - \boldsymbol{\theta}_i^{(s)}\right) + \frac{\kappa_i}{2}\left\| \boldsymbol{\theta}_i - \boldsymbol{\theta}_i^{(s)} \right\|_2^2  + \boldsymbol{\theta}_i^{\intercal}\boldsymbol{p}_i^{(s+1)} + \rho \sum_{j\in \mathcal{N}_i} \left\|\boldsymbol{\theta}_i - \frac{\boldsymbol{\theta}_i^{(s)}+\boldsymbol{\theta}_j^{(s)}}{2} \right\|_2^2,
\end{align*}
where $\mathcal{Q}_i:\mathbb{R}^{D+2}\rightarrow \mathbb{R}$. Next, factor out the optimizing parameter $\boldsymbol{\theta}_i$ to obtain,
\begin{align}\nonumber
    \mathcal{Q}_i(\boldsymbol{\theta}_i) = & \, \nabla_{\boldsymbol{\theta}}^{\intercal}\mathcal{L}_i\left(\boldsymbol{\theta}_i^{(s)}\right)\boldsymbol{\theta}_i- c_1 + \frac{\kappa_i}{2}\bigg( \boldsymbol{\theta}_i^{\intercal}\boldsymbol{\theta}_i -2  \boldsymbol{\theta}_i^{\intercal} \boldsymbol{\theta}_i^{(s)} + c_2 \bigg)  + \boldsymbol{\theta}_i^{\intercal}\boldsymbol{p}_i^{(s+1)} + \mathcal{T}_i \\ \label{eq:prfThrm_Q_i}
= &\,  \boldsymbol{\theta}_i^{\intercal} \left(\nabla_{\boldsymbol{\theta}}\mathcal{L}_i\left(\boldsymbol{\theta}_i^{(s)}\right) - \kappa_i\boldsymbol{\theta}_i^{(s)}+\boldsymbol{p}_i^{(s+1)} \right)  +\frac{\kappa_i}{2} \boldsymbol{\theta}_i^{\intercal}\boldsymbol{\theta}_i + \mathcal{T}_i ,
\end{align}
where $\mathcal{T}_i = \rho \sum_{j\in \mathcal{N}_i} \|\boldsymbol{\theta}_i - ({\boldsymbol{\theta}_i^{(s)}+\boldsymbol{\theta}_j^{(s)}})/{2} \|_2^2 $, $c_1 =-\nabla_{\boldsymbol{\theta}}^{\intercal}\mathcal{L}_i(\boldsymbol{\theta}_i^{(s)})\boldsymbol{\theta}_i^{(s)} $, and $c_2 =\boldsymbol{\theta}_i^{\intercal(s)} \boldsymbol{\theta}_i^{(s)} $. Note that $c_1$, $c_2$ are constants with respect to the optimizing parameter $\boldsymbol{\theta}_i$ and thus irrelevant to the problem. For any strongly connected graph topology, the term $\mathcal{T}_i$ can be expressed as,
\begin{align}\nonumber
    \mathcal{T}_i &=\rho \sum_{j\in \mathcal{N}_i} \left \|\boldsymbol{\theta}_i - \frac{\boldsymbol{\theta}_i^{(s)}+\boldsymbol{\theta}_j^{(s)}}{2} \right\|_2^2\\ \nonumber
    &= \rho \sum_{j\in \mathcal{N}_i} \boldsymbol{\theta}_i^{\intercal}\boldsymbol{\theta}_i -\boldsymbol{\theta}_i^{\intercal}\left(\boldsymbol{\theta}_i^{(s)}+\boldsymbol{\theta}_j^{(s)}\right) +c_3  \\ \nonumber
    &= \rho \mathrm{card}(\mathcal{N}_i) \boldsymbol{\theta}_i^{\intercal}\boldsymbol{\theta}_i -\rho \sum_{j\in \mathcal{N}_i} \boldsymbol{\theta}_i^{\intercal}\boldsymbol{\theta}_i^{(s)}+\boldsymbol{\theta}_i^{\intercal}\boldsymbol{\theta}_j^{(s)} \\ \nonumber
    & =  \rho \mathrm{card}(\mathcal{N}_i) \boldsymbol{\theta}_i^{\intercal}\boldsymbol{\theta}_i -\rho  \mathrm{card}(\mathcal{N}_i) \boldsymbol{\theta}_i^{\intercal}\boldsymbol{\theta}_i^{(s)} -\rho \boldsymbol{\theta}_i^{\intercal} \sum_{j\in \mathcal{N}_i}\boldsymbol{\theta}_j^{(s)}  \\ \label{eq:prfThrm_T_i}
    &=\mathrm{card}(\mathcal{N}_i)\rho\boldsymbol{\theta}_i^{\intercal}\boldsymbol{\theta}_i - \rho\boldsymbol{\theta}_i^{\intercal} \left( \mathrm{card}(\mathcal{N}_i)\boldsymbol{\theta}_i^{(s)}+\sum_{j\in \mathcal{N}_i}\boldsymbol{\theta}_j^{(s)}  \right),
\end{align}
where $c_3 = (1/4)(\boldsymbol{\theta}_i^{(s)}+\boldsymbol{\theta}_j^{(s)})^{\intercal}(\boldsymbol{\theta}_i^{(s)}+\boldsymbol{\theta}_j^{(s)})$ is a constant for the optimizing parameter $\boldsymbol{\theta}_i$ and thus ignored. 
The local objective $\mathcal{Q}_i$ by combining \eqref{eq:prfThrm_Q_i}, \eqref{eq:prfThrm_T_i} results in,
\begin{align}\nonumber
    \mathcal{Q}_i(\boldsymbol{\theta}_i) = &\,  \boldsymbol{\theta}_i^{\intercal} \left(\nabla_{\boldsymbol{\theta}}\mathcal{L}_i\left(\boldsymbol{\theta}_i^{(s)}\right) - \kappa_i\boldsymbol{\theta}_i^{(s)}+\boldsymbol{p}_i^{(s+1)} \right)  +\frac{\kappa_i}{2} \boldsymbol{\theta}_i^{\intercal}\boldsymbol{\theta}_i + \mathrm{card}(\mathcal{N}_i)\rho\boldsymbol{\theta}_i^{\intercal}\boldsymbol{\theta}_i \\ \nonumber
    & - \rho\boldsymbol{\theta}_i^{\intercal} \left( \mathrm{card}(\mathcal{N}_i)\boldsymbol{\theta}_i^{(s)}+\sum_{j\in \mathcal{N}_i}\boldsymbol{\theta}_j^{(s)}  \right) \\ \nonumber
    = &\,\boldsymbol{\theta}_i^{\intercal} \left( \nabla_{\boldsymbol{\theta}}\mathcal{L}_i\left(\boldsymbol{\theta}_i^{(s)}\right) - \bigg(\kappa_i+\mathrm{card}(\mathcal{N}_i)\rho\bigg)\boldsymbol{\theta}_i^{(s)} + \boldsymbol{p}_i^{(s+1)} -\rho\sum_{j\in \mathcal{N}_i}\boldsymbol{\theta}_j^{(s)}  \right) \\ \label{eq:prfThrm_Qi}
    &+ \bigg(\frac{\kappa_i}{2}+\mathrm{card}(\mathcal{N}_i)\rho\bigg) \boldsymbol{\theta}_i^{\intercal}\boldsymbol{\theta}_i.
\end{align}
Next, we show that the local objective $\mathcal{Q}_i$ \eqref{eq:prfThrm_Qi} is a convex function in a quadratic form \citep{boyd2004convex} by computing its Hessian,
\begin{align*}
    \mathcal{H}_{\mathcal{Q}_i} &= \frac{\partial ^2 \mathcal{Q}_i}{\partial\boldsymbol{\theta}_i^2} \\
    &= \left(\kappa_i +2\mathrm{card}(\mathcal{N}_i)\rho\right)I_{D+2} \succ 0.
\end{align*}
Since the local objective $\mathcal{Q}_i$ is convex and quadratic, we can obtain a closed-form solution by computing the first derivative,
\begin{align*}\nonumber
\frac{\partial \mathcal{Q}}{\partial \boldsymbol{\theta}_i} &=\nabla_{\boldsymbol{\theta}}\mathcal{L}_i\left(\boldsymbol{\theta}_i^{(s)}\right) - \bigg(\kappa_i+\mathrm{card}(\mathcal{N}_i)\rho\bigg)\boldsymbol{\theta}_i^{(s)} + \boldsymbol{p}_i^{(s+1)}   -\rho\sum_{j\in \mathcal{N}_i}\boldsymbol{\theta}_j^{(s)}+  2\bigg(\frac{\kappa_i}{2}+\mathrm{card}(\mathcal{N}_i)\rho\bigg)\boldsymbol{\theta}_i,
\end{align*}
and then set $\partial \mathcal{Q}/\partial \boldsymbol{\theta}_i = \boldsymbol{0}$, which yields,
\begin{align*}\nonumber
  \boldsymbol{\theta}_i &=  \frac{1}{\kappa_i+2\mathrm{card}(\mathcal{N}_i)\rho } \left( \rho \sum_{j \in \mathcal{N}_i}\boldsymbol{\theta}_j^{(s)} -\nabla_{\boldsymbol{\theta}}\mathcal{L}_i\left(\boldsymbol{\theta}_i^{(s)}\right) + \left(\kappa_i+\mathrm{card}(\mathcal{N}_i)\rho \right)\boldsymbol{\theta}_i^{(s)} - \boldsymbol{p}_i^{(s+1)}\right).  
\end{align*}
The rest proof is a direct consequence of \citep[Theorem 1]{chang2014multi}.
\subsection{Proof of Lemma~\ref{lem:cbnn}}\label{app:proof:cbnn}
First, we show that the separable squared exponential kernel \eqref{eq:seKernel} is a monotonically decreasing function. Let any two observations $\boldsymbol{x}_1$,  $\boldsymbol{x}_2$ and a location of interest $\boldsymbol{x}_*$ with $\|\boldsymbol{x}_*-\boldsymbol{x}_1\|_2^2>\|\boldsymbol{x}_*-\boldsymbol{x}_2\|_2^2 $. The covariance function of the first observations takes the form of,
\begin{align}\nonumber
    k(\boldsymbol{x}_1,\boldsymbol{x}_*) &= \sigma_f^2 \exp\left\{-\sum_{d=1}^D\frac{(x_{*,d} - x_{1,d})^2}{l_d^2}\right\}\\ \label{eq:prf:lm:cov_1_sep}
    & = \sigma_f^2 \exp \left\{ -\left(\frac{(x_{*,1} - x_{1,1})^2}{l_1^2}  + \frac{(x_{*,2} - x_{1,2})^2}{l_2^2}\right) \right\}.
\end{align}
Since the agents collect data in stripes along $y$-axis and the network topology is a path graph, the length-scales of the covariance function \eqref{eq:seKernel} are the same $l_1=l_2=l$. Note that this applies only for the CBNN selection subroutine. Thus, the covariance function of the first observation \eqref{eq:prf:lm:cov_1_sep} yields,
\begin{align*}\nonumber
    k(\boldsymbol{x}_1,\boldsymbol{x}_*) &= \sigma_f^2 \exp \left\{  - \left( \frac{(x_{*,1} - x_{1,1})^2}{l^2}  + \frac{(x_{*,2} - x_{1,2})^2}{l^2}\right) \right\}\\ 
    & = \sigma_f^2 \exp \left\{ -\frac{\|\boldsymbol{x}_{*} - \boldsymbol{x}_{1}\|^2_2}{l^2}  \right\}.
\end{align*}
Similarly, the covariance function of the second observation becomes,
\begin{align*}
    k(\boldsymbol{x}_2,\boldsymbol{x}_*) =  \sigma_f^2 \exp \left\{ -\frac{\|\boldsymbol{x}_{*} - \boldsymbol{x}_{2}\|^2_2}{l^2}  \right\}.
\end{align*}
Provided that if $\|\boldsymbol{x}_*-\boldsymbol{x}_1\|_2^2>\|\boldsymbol{x}_*-\boldsymbol{x}_2\|_2^2 $ then $k(\boldsymbol{x}_1,\boldsymbol{x}_*) \leq k(\boldsymbol{x}_2,\boldsymbol{x}_*)$ for all $\boldsymbol{x}_*, \boldsymbol{x}_1, \boldsymbol{x}_2 \in \mathbb{R}^2$, the covariance function $k(\cdot,\cdot)$ used in CBNN calculation \eqref{eq:cbnn_cross_covar} is monotonically decreasing with the rate of decrease depending on the signal variance $\sigma_{{f}}^2$ and the length-scale $l$.

The CBNN is described by a sub-graph $\mathcal{G}_{\textrm{NN}} = (\mathcal{V}_{\textrm{NN}}, \mathcal{{E}_{\textrm{NN}}}(t))$, where $\mathcal{V}_{\textrm{NN}} = \{i \in \mathcal{V} : \|\boldsymbol{x}_i - \boldsymbol{x}_* \|_2 \leq r_{\textrm{NN}} \}$ for all $\boldsymbol{x}_*$ with $r_{\textrm{NN}}$ the nearest neighbor radius and $\mathcal{{E}_{\textrm{NN}}}(t) \subseteq \mathcal{{V}_{\textrm{NN}}}\times \mathcal{{V}_{\textrm{NN}}}$. In the current application setup for the CBNN selection, the separable squared exponential kernel \eqref{eq:seKernel} behaves as a squared exponential kernel with $l_1 = l_2=l$. This radial decay creates a circle with radius $r_{\textrm{NN}}$ around the location of interest $\boldsymbol{x}_*$ for the selection of nearest neighbors. Note that the radius is a function of the cross-covariance and the user-defined threshold $r_{\textrm{NN}}([\boldsymbol{k}_{\mu,*}]_i, \eta_{\textrm{NN}})$, thus inherits a covariance-based dependence. Next, due to fact that the path graph topology (Figure~\ref{fig:graph_mixed}-(a)) includes only connected agents in the circular space, the connectivity in the nearest neighbor graph $\mathcal{G}_{\textrm{NN}}$ is preserved. Hence, the CBNN selection maintains strong connectivity.

\vskip 0.2in
\normalem

\bibliographystyle{unsrtnat}
\bibliography{references}  






\end{document}